\documentclass{article}

\usepackage{hyperref}

\usepackage{microtype}
\usepackage{subfigure}
\usepackage{booktabs} %

\usepackage{amsmath,amssymb,amsthm,enumitem,mathtools}
\usepackage{graphicx,subfigure}
\usepackage{nicefrac}
\usepackage[ruled,vlined]{algorithm2e}
\usepackage{colortbl}
\usepackage[dvipsnames]{xcolor}

\usepackage{todonotes}
\usepackage{graphicx}
\usepackage{multirow}

\def\D{\mathcal{D}}
\def\R{\mathbb{R}}

\def\P{\mathcal{P}}

\def\H{\mathcal{C}}
\def\X{\mathcal{X}}

\def\byzantine{faulty}

\newcommand{\cva}{\textsc{NNA}}

\newcommand{\newalgorithm}{\textsc{MoNNA}}

\newcommand{\reduction}{reduction}

\newcommand{\expect}[1]{\mathop{{}\mathbb{E}}\left[{#1}\right]}
\newcommand{\condexpect}[2]{\mathbb{E}_{#1}\left[{#2}\right]}

\newcommand{\card}[1]{\left\lvert{#1}\right\rvert}

\providecommand{\iprod}[2]{\ensuremath{\left\langle #1,\,#2  \right\rangle}}
\providecommand{\norm}[1]{\ensuremath{\left\lVert#1\right\rVert }}

\newcommand{\ceil}[1]{\left\lceil{#1}\right\rceil}

\newcommand{\loss}{Q}
\newcommand{\lossperdatapoint}{q}
\newcommand{\localloss}[1]{\indexvar{#1}{}{\loss}}

\newcommand{\avgloss}{\loss^{(\H)}}

\newcommand{\localgrad}[1]{\nabla{} \indexvar{#1}{}{\loss}}

\newcommand{\indexvar}[3]{{#3}^{\ifthenelse{\equal{#1}{}}{}{\left({#1}\right)}}_{#2}}

\newcommand{\gradient}[2]{\indexvar{#1}{#2}{g}}
\newcommand{\learningrate}[1]{\gamma_{#1}}

\newcommand{\lipschitz}{L}

\newcommand{\nodenumber}{n}
\newcommand{\byzantinebound}{f}

\newcommand{\gradientstddev}{\sigma}
\newcommand{\honestnode}{i}

\newcommand{\datapoint}{x}

\newcommand{\iteration}{t}

\newcommand{\mmtcoefficient}{\beta}
\newcommand{\numbercvarounds}{K}
\newcommand{\cvaround}{k}
\newcommand{\cvavector}[2]{\indexvar{#1}{#2}{x}}
\newcommand{\cvaavg}[1]{\bar{x}_{#1}}

\newcommand{\model}[2]{\indexvar{#1}{#2}{\theta}}
\newcommand{\receive}[2]{\indexvar{#1}{#2}{\mathcal{R}}}
\newcommand{\filter}[2]{\indexvar{#1}{#2}{\mathcal{S}}}
\newcommand{\avgmodel}[1]{\indexvar{}{#1}{\overline \theta}}
\newcommand{\returnedmodel}[1]{\indexvar{#1}{}{\hat{\theta}}}

\newcommand{\avggrad}[1]{\indexvar{}{#1}{\overline{g}}}

\newcommand{\mmt}[2]{\indexvar{#1}{#2}{m}}
\newcommand{\diam}[2]{\Gamma\left( {#1}_{#2}\right)}

\newcommand{\dev}[1]{\delta_{#1}}
\newcommand{\AvgMmt}[1]{\overline{m}_{#1}}
\newcommand{\AvgGrad}[1]{\overline{\nabla Q}_{#1}}
\newcommand{\AvgNoisyGrad}[1]{\overline{g}_{#1}}

\newcommand{\effectgrad}[1]{G_{#1}}

\newcommand{\heter}{\zeta}
\newcommand{\cvabyzbound}{\delta}
\newcommand{\coeffalpha}{E(\alpha)}

\renewcommand{\paragraph}[1]{\textbf{#1}~}

\newtheorem{definition}{Definition}
\newtheorem{proposition}{Proposition}
\newtheorem{theorem}{Theorem}
\newtheorem{lemma}{Lemma}
\newtheorem{corollary}{Corollary}
\newtheorem{remark}{Remark}
\newtheorem{assumption}{Assumption}
\newtheorem{theorem**}{\bfseries Theorem}
\newtheorem*{theorem*}{\bfseries Theorem}
\newtheorem*{lemma*}{\bfseries Lemma}

\renewcommand{\paragraph}[1]{\textbf{#1}~}

\usepackage{todonotes}

\makeatletter
\newtheorem*{rep@theorem}{\rep@title}
\newcommand{\newreptheorem}[2]{%
\newenvironment{rep#1}[1]{%
 \def\rep@title{#2 \ref{##1}}%
 \begin{rep@theorem}}%
 {\end{rep@theorem}}}
\makeatother
\newreptheorem{theorem}{Theorem}

\makeatletter
\newtheorem*{rep@assumption}{\rep@title}
\newcommand{\newrepassumption}[2]{%
\newenvironment{rep#1}[1]{%
 \def\rep@title{#2 \ref{##1}}%
 \begin{rep@assumption}}%
 {\end{rep@assumption}}}
\makeatother
\newreptheorem{assumption}{Assumption}

\makeatletter
\newtheorem*{rep@lemma}{\rep@title}
\newcommand{\newreplemma}[2]{%
\newenvironment{rep#1}[1]{%
 \def\rep@title{#2 \ref{##1}}%
 \begin{rep@lemma}}%
 {\end{rep@lemma}}}
\makeatother
\newreptheorem{lemma}{Lemma}

\makeatletter
\newtheorem*{rep@proposition}{\rep@title}
\newcommand{\newrepproposition}[2]{%
\newenvironment{rep#1}[1]{%
 \def\rep@title{#2 \ref{##1}}%
 \begin{rep@proposition}}%
 {\end{rep@proposition}}}
\makeatother
\newreptheorem{proposition}{Proposition}

\makeatletter
\newtheorem*{rep@definition}{\rep@title}
\newcommand{\newrepdefinition}[2]{%
\newenvironment{rep#1}[1]{%
 \def\rep@title{#2 \ref{##1}}%
 \begin{rep@definition}}%
 {\end{rep@definition}}}
\makeatother
\newreptheorem{definition}{Definition}

\usepackage[accepted]{icml2023}

\usepackage{amsmath}
\usepackage{amssymb}
\usepackage{mathtools}
\usepackage{amsthm}

\usepackage[capitalize,noabbrev]{cleveref}

\icmltitlerunning{Robust Collaborative Learning with Linear Gradient Overhead}

\begin{document}

\twocolumn[
\icmltitle{Robust Collaborative Learning with Linear Gradient Overhead}
\icmlsetsymbol{equal}{*}

\begin{icmlauthorlist}
\icmlauthor{Sadegh Farhadkhani}{comp}
\icmlauthor{Rachid Guerraoui}{comp}
\icmlauthor{Nirupam Gupta}{comp}
\\
\icmlauthor{Lê Nguyên Hoang}{sch,sss}
\icmlauthor{Rafael Pinot}{comp}
\icmlauthor{John Stephan}{comp}
\end{icmlauthorlist}

\icmlaffiliation{comp}{EPFL,}
\icmlaffiliation{sch}{Tournesol,}
\icmlaffiliation{sss}{Calicarpa}

\icmlcorrespondingauthor{Sadegh Farhadkhani}{sadegh.farhadkhani@epfl.ch}
\icmlkeywords{Robust collaborative learning, Byzantine learning}

\vskip 0.3in
]

\printAffiliationsAndNotice{}  %
\begin{abstract}
{\em Collaborative learning} algorithms, such as {\em distributed SGD} (or D-SGD), are prone to faulty machines that may deviate from their prescribed algorithm because of software or hardware bugs, poisoned data or malicious behaviors. While many solutions have been proposed to enhance the robustness of D-SGD to such machines, previous works
either resort to strong assumptions (\emph{trusted} server, \emph{homogeneous} data, specific noise model) or impose a gradient computational cost that is several orders of magnitude higher than that of D-SGD. 
We present~\newalgorithm{}, a new algorithm that (a) is provably robust under standard assumptions and (b) has a gradient computation overhead that is linear in the fraction of \byzantine{} machines, which is conjectured to be tight. Essentially, \newalgorithm{} uses {\em Polyak's momentum} of local gradients for {\em local updates} and {\em nearest-neighbor averaging (NNA)} for \emph{global mixing}, respectively.
While \newalgorithm{} is rather simple to implement, its analysis has been more challenging and relies on two key elements that may be of independent interest. Specifically, we introduce the mixing criterion of {\em $(\alpha, \lambda)$-reduction} to analyze the \emph{non-linear mixing} of non-faulty machines, and present a way to control the tension between the momentum and the model {\em drifts}.
We validate our theory by experiments on image classification and make our code  available at \url{https://github.com/LPD-EPFL/robust-collaborative-learning}.
\end{abstract}

\vspace{-0.5cm}
\section{Introduction}

Collaborative learning allows multiple machines (or \emph{nodes}), each with a local dataset, to learn local models that offer a high accuracy on the union of their local datasets~\cite{boyd2011distributed}. This paradigm facilitates the training of complex models over a large volume of data, while addressing concerns on data locality and ownership. The general task of collaborative learning can be formulated as follows.
Consider a {\em parameter space} $\R^d$, a {\em data space} $\X$ and a {\em loss function} $q: \R^d \times \X \to \R$. Given a parameter $\theta \in \R^d$, a data point $x \in \X$ incurs a loss of value $q(\theta, \, x)$. The system comprises $n$ nodes. Each node $i$ samples data from distribution $\D_i$, and thus has a {\em local loss function} $\loss^{(i)}\left( \model{}{} \right) \coloneqq \condexpect{x \sim \D_i}{q(\model{}{}, \, x)}$.
The goal for each node $i$ is to compute $\theta_*^{(i)}$ minimizing the {\em global average loss}, i.e., 
\begin{align}
    \theta_*^{(i)} \in \underset{\theta \in \R^d}{\arg\min} ~ \frac{1}{n} \sum_{j = 1}^n \loss^{(j)}\left( \model{}{} \right). \label{eqn:global_loss_all}
\end{align}
 {\bf Collaborative learning with D-SGD.} The most standard way of solving the optimization problem~\eqref{eqn:global_loss_all} is through the use of the celebrated distributed SGD (D-SGD) method~\cite{TangLYZL18, koloskova2020unified}. Each node maintains a local parameter, approximating a solution of the optimization problem~\eqref{eqn:global_loss_all}, which is updated iteratively in two phases. In the first phase, also called {\em the local phase}, each node updates its current parameter {\em partially} using a stochastic estimate of its local loss function's gradient.
 In the second phase, also called {\em the coordination phase}, the nodes exchange their partially updated parameters with each other over a network, and then each node replaces its current parameter by the {\em average} of all the partially updated parameters. While the former is essential for reducing the local loss functions, the latter yields reduction in the global average loss function. Alternately, as is the case in {\em federated learning}~\cite{KairouzOpenProblemsinFed2021}, the nodes may rely on a {\em trusted} coordinator (called the {\em server}) to execute the coordination phase involving the averaging operation.

 \setlength{\arrayrulewidth}{0.2mm}
\renewcommand{\arraystretch}{1.7}
\begin{table*}[!ht]
  \centering
  \caption{Comparison of \newalgorithm{} with other prominent schemes for robust collaborative learning including BRIDGE~\cite{yang2019bridge}, BTARD~\cite{Gorbunov21}, SCC~\cite{he22}, and LEARN~\cite{collaborativeElMhamdi21}. {\bf S} - Stochastic gradients, {\bf H} - Heterogeneous (a.k.a., non-iid) datasets, {\bf A} - Asynchronous communication, and $\mathbf{f/n}$ - tolerable fraction of \byzantine{} nodes.}
  \vspace{2mm}
  \begin{tabular}{c||ccccccc}
    {\bf Method} & {\bf Loss Function} & {\bf S} & {\bf H} & {\bf A} & {\bf Communication} & $\mathbf{f/n}$ & {\bf Gradient Complexity}\\ 
    \hline
    \hline
    BRIDGE & Locally strongly convex & $\times$ & $\times$ & $\times$ & Sparse &  $< \frac{1}{2}$ & $\times^{\textcolor{blue}{(*)}}$ \\
    \hline
    BTARD &  Non-convex & $\checkmark$ & $\times$ & $\times$ & Pairwise & $< \frac{1}{10}$ & $\mathcal{O}\left( \frac{1}{n\epsilon^2} \right)^{\textcolor{blue}{(**)}}$ \\
    \hline
    SCC & Non-convex & \checkmark & \checkmark & $\times$ & Sparse & $\leq \frac{1}{10240}$ & $\mathcal{O}\left( \left( \frac{1}{n} + \frac{f}{n} \right) \frac{1}{\epsilon^2} \right)$ \\
    \hline
    LEARN & Non-convex & \checkmark & \checkmark & \checkmark & Pairwise  & $< \frac{1}{3}$ \text{ or $< \frac{1}{6}$}  & $\mathcal{O}\left(\frac{1}{\epsilon^5}\right)$ \\
    \hline
    \rowcolor{green!20}
      &  &  &   &  &  & $\leq \frac{1}{11}$ & $\mathcal{O}\left( \left( \frac{1}{n} + \frac{f}{n} \right) \frac{1}{\epsilon^2} \right)$\\
        \rowcolor{green!20}
     \multirow{-2}{*}{\newalgorithm{}}& \multirow{-2}{*}{Non-convex}& \multirow{-2}{*}{\checkmark} & \multirow{-2}{*}{\checkmark} & \multirow{-2}{*}{\checkmark} &  \multirow{-2}{*}{Pairwise}  & $ < \frac{1}{5}$ & $\mathcal{O}\left(\frac{(1 + f)^2}{n \epsilon^2}\right)$\\
    \hline
    \multicolumn{7}{c}{Gradient complexity for non-convex losses with a trusted server~\cite{karimireddy2022byzantine}:}  &$\mathcal{O}\left( \left( \frac{1}{n} + \frac{f}{n} \right) \frac{1}{\epsilon^2} \right)$\\
    \hline
  \end{tabular}
  \vspace{0.2cm}

   \raggedright \small{ $^{\textcolor{blue}{(*)}}$ As of yet, no finite time convergence rate is known for BRIDGE.\\
   $^{\textcolor{blue}{(**)}}$ The leading term in the convergence rate of BTARD~\cite{Gorbunov21} is identical to that of D-SGD without faults and  does not introduce any overhead. The reason is that it considers a weaker adversarial model with public datasets, where each node can access the entire training data and validate the computations done by other nodes, and thereby check any faults.}
  \label{tab:1}
  \vspace{-3mm}
\end{table*}
 
{\bf Robustness issue.} D-SGD is not very robust: a handful of faulty nodes, deviating from their prescribed algorithm, may prevent the remaining non-faulty (or {\em correct}) nodes from computing a valid solution~\cite{su2016fault}.
Such behavior may result from software and hardware bugs, poisoned data, or malicious adversaries controlling part of the network. 
 We consider a setting where at most $f$ (out of $n$) nodes in the system are \byzantine{} and assume that these can behave arbitrarily\footnote{In distributed computing, such faulty nodes are also commonly referred to as Byzantine~\cite{lamport1982Byzantine}.} (either by accident or intent).
 In this case, the original optimization problem~\eqref{eqn:global_loss_all} is rendered vacuous. A more reasonable goal is 
to minimize the average loss function for the correct nodes~\cite{gupta2020fault}. However, there is a fundamental limit on achieving this goal,  because \byzantine{} nodes may behave as correct nodes with outlying local data distributions~\cite{liu2021approximate, karimireddy2022byzantine}. Thus, the ultimate goal of {\em robustness} reduces to designing an algorithm that enables all correct nodes to compute a {\em tight} approximation of a minimum
of the average correct loss~\cite{collaborativeElMhamdi21, he22}.

\subsection{Prior Work}

The problem of robustness in collaborative learning has received significant attention in recent years~\cite{yang2020adversary, liu2021survey, bouhata2022byzantine}. Most previous works focused on server-based coordination (i.e., the nodes have access to a server that is assumed fault-free)~\cite{brute_bulyan, kardam, chen2017distributed, yin2018byzantine, Karimireddy2021, karimireddy2022byzantine, farhadkhani2022byzantine}. This server constitutes a {\em single point of failure}, which greatly compromises the security of the learning procedure. It is therefore appealing to consider a scenario in which the nodes collaborate by communicating directly, without relying on a central server.

The absence of a central authority, 
combined with asynchronous communication~\cite{cachin2011introduction}
and  \byzantine{} nodes, lead to a non-trivial {\em drift} between the local parameters maintained by the correct nodes. Controlling this drift is key to learning an accurate model by the correct nodes, and constitutes a major challenge.
Prior attempts to address this issue, including~\cite{yang2019bridge,yang2019byrdie,collaborativeElMhamdi21,Guo21byzantine, he22, Gorbunov21}, rely on strong assumptions such as
{\em homogeneous} data~\cite{yang2019bridge, yang2019byrdie,Guo21byzantine, Gorbunov21}, {\em strong convexity}~\cite{gupta2021byzantine_thinh, yang2019byrdie}, a precise gradient noise modeling, and an extremely small fraction of \byzantine{} nodes \textcolor{black}{as in the parallel work of} ~\citet{he22}; or impose orders of magnitude larger gradient overhead compared to D-SGD~\cite{collaborativeElMhamdi21}. These shortcomings limit the practicality of the state-of-the-art methods.

\subsection{Contributions}
We take an important step towards making robust collaborative learning more realizable. Specifically, 
we present an adaptation of D-SGD, named~\newalgorithm{}, which to the best of our knowledge, is the first collaborative learning algorithm that is provably robust under assumptions that are standard in analyzing D-SGD~\cite{lian2017can,TangLYZL18}. Moreover, the gradient computational overhead imposed by~\newalgorithm{}, compared to D-SGD, only grows linearly in the fraction of \byzantine{} nodes, which is conjectured to be tight~\cite{Karimireddy2021}. We compare \newalgorithm{} with the most relevant related approaches in Table~\ref{tab:1}.

{\bf Overview of~\newalgorithm{}.} 
In the local phase, unlike D-SGD, each correct node uses the Polyak's momentum~\cite{POLYAK19641} of its local stochastic gradients to partially update its current local parameter. The use of local momentum amortizes the dependence on local variance in the error due to \byzantine{} nodes. In the coordination phase, instead of simply averaging the received partial updates, each correct node aggregates them using a {\em robust aggregation rule} we call nearest neighbor averaging (NNA). In NNA, as the name suggests, a node eliminates the $f$ parameters it receives that are the farthest from its own and then averages the rest. This filtering aims to reduce the drift between correct nodes' local parameters, by mitigating the influence of arbitrary parameters that may be sent by the \byzantine{} nodes. While~\newalgorithm{} has been rather simple to implement, its analysis has been more challenging, involving two elements that may be of independent interest to the distributed optimization community at large:
namely, (i) the mixing criterion of {\em $(\alpha, \, \lambda)$-\reduction{}}, and (ii) the control of local parameters' drift under {$(\alpha, \, \lambda)$-\reduction{}} mixing when incorporating Polyak's momentum. We discuss these elements below, after the summary of our theoretical results. 
{\bf Theoretical results.}
We assume at most $f$ out of $n$ nodes may be \byzantine{} and behave arbitrarily. We 
denote by 
$\H$ the set of correct nodes and $\loss^{(\H)} (\model{}{})$ their average loss, i.e., 
\begin{equation}
    \loss^{(\H)} (\model{}{}) \coloneqq\frac{1}{\card{\H}} \sum_{i \in \H} \loss^{(i)}\left( \model{}{} \right) . \label{eqn:global_loss}
\end{equation}
We consider the class of Lipschitz smooth non-convex loss functions, and assume local stochastic gradients (of correct nodes) to satisfy standard properties in the context of D-SGD~\cite{TangLYZL18}, i.e., bounded local variance of $\sigma^2$ and bounded global diversity of $\heter^2$. We show that if $n \geq 11 f$, then upon executing $T$ iterations of~\newalgorithm{}, each correct node $i$ returns a local parameter $\hat{\theta}^{(i)}$ such that $\expect{\norm{ \nabla \loss^{(\H)} \left(\hat{\theta}^{(i)} \right)}^2} \leq \epsilon$ where  
\begin{align}
\label{eq:errorintro}
    \epsilon \in \mathcal{O}\left(\sqrt{\frac{\sigma^2}{T}\left( \frac{1+f}{n} \right)} + \frac{f}{n} \heter^2\right).
\end{align}
Recall that the number of iterations $T$ equals the total number of gradients computed by each correct node. Hence, the gradient complexity of~\newalgorithm{} is $1+f$ times that of D-SGD, i.e., the gradient overhead is linear in the fraction of \byzantine{} nodes. 
Note that the non-vanishing error of $(f/n) \heter^2$ is a fundamental lower bound in the presence of \byzantine{} nodes due to diversity in local distributions~\cite{karimireddy2022byzantine}. We also show that, by increasing gradient complexity by a factor $f$, \newalgorithm{} is robust to $n/5$ \byzantine{} nodes. 
{\bf $(\alpha, \, \lambda)$-Reduction mixing.}
In the presence of \byzantine{} nodes, it is impossible to ensure {\em linear} mixing of correct updates in the coordination phase. We can no longer rely on the linear mixing criterion of double stochasticity with a bounded spectral gap, usually assumed in the case of D-SGD~\cite{tsitsiklis1986distributed, xiao2004fast, TangLYZL18, koloskova2020unified}. To circumvent this limitation, we introduce a new mixing criterion of {\em $(\alpha, \, \lambda)$-reduction} that extends the classical mixing criterion to analyze {\em robust mixing} schemes, such as NNA, that may be non-linear and even non-continuous. Parameters $\alpha$ and $ \lambda$ are positive real values quantifying the levels of contraction and centering, respectively, over the set of correct updates. 
We prove that the use of $(\alpha, \, \lambda)$-reduction, with $\alpha < 1$ and $\lambda < \infty$, in the coordination phase of D-SGD 
enables each correct node $i$ to return $\hat{\theta}^{(i)}$ such that $\expect{\norm{ \nabla \loss^{(\H)} \left(\hat{\theta}^{(i)} \right)}^2} \leq \epsilon$, where
\begin{align}
\epsilon \in \mathcal{O}\left(\sqrt{\frac{\sigma^2}{nT}} + \frac{\lambda}{(1-\alpha)^2} \, ( \heter^2 + \sigma^2 ) \right). \label{eqn:conv_reduction}
\end{align}

Assuming asynchronous pairwise communication between nodes and 
$f \leq \frac{n}{11}$, we prove that NNA satisfies $(\alpha, \, \lambda)$-reduction with $\alpha \leq 0.988 < 1$ and $\lambda \in \Theta\left( \nicefrac{f}{n}\right)$. The key to proving this result is that, unlike standard aggregation rules~\cite{farhadkhani2022byzantine}, NNA makes each correct node pivot the aggregation around their own local parameter. Substituting these values of $\alpha$ and $\lambda$ in~\eqref{eqn:conv_reduction} yields the error in~\eqref{eq:errorintro}, plus an additional non-vanishing term of $(f/n) \sigma^2$. We show that this term vanishes at the rate of $\sqrt{1/T}$ when using local momentum, as specified in \newalgorithm{}. Hence, reducing the error term to~\eqref{eq:errorintro}. Proving this reduction however requires a novel technique for controlling drift.
{\bf Controlling drift under momentum.}
The second key element underlying our analysis pertains to the use of Polyak's momentum for local updates in~\newalgorithm{}.
We prove that momentum eliminates the non-vanishing error due to the local variance $\sigma^2$ in the optimization error~\eqref{eqn:conv_reduction}, and thereby matches the lower bound. While such observation has been made in the case of server-based coordination~\cite{distributed-momentum, Karimireddy2021, farhadkhani2022byzantine}, it is not immediate in our setting because of the cross-coupling of the momentum drift and the drift between correct nodes' models. By carefully analyzing this coupling, we  obtain 
uniform bounds on both model and momentum drifts. We then adapt the {\em Lyapunov function} (a.k.a. potential function) to account for the model drift.

{\bf Empirical evaluation.}
We evaluate \newalgorithm{} on two benchmark image classification tasks. We consider a distributed asynchronous system including $n/5$ \byzantine{} nodes executing four different attacks.
\newalgorithm{} significantly outperforms state-of-the-art robust collaborative learning algorithms in all adversarial settings, and almost matches the performance of D-SGD in terms of learning accuracy.

\subsection{Paper Organization}
We formalize robust collaborative learning in Section~\ref{sec:problem}. 
Section~\ref{sec:algorithm} presents our algorithm as well as its convergence and robustness. Section~\ref{sec:analysis} discusses the key elements of our convergence analysis.  Section~\ref{sec:experiments} presents our empirical evaluation. We discuss future research directions in Section~\ref{sec:conclusion}. {Due to space limitations, full proofs and some auxiliary empirical results are deferred to the appendices.}

\section{Problem Statement}
\label{sec:problem}

We consider a set of $n$ nodes, $[\nodenumber] = \{1, \ldots, \nodenumber\}$, out of which at most $f < n/3$ may behave arbitrarily. We refer to such nodes as {\em \byzantine{}}. The identity of \byzantine{} nodes is a priori unknown to the remaining correct, i.e., non-\byzantine{} nodes. We assume that the nodes interact with each other using the following communication model.

{\bf Communication model.} We assume a \emph{pairwise} communication scheme where nodes exchange messages with each other over a network. The messages however need not arrive in a timely manner, i.e.,  communication is {\em asynchronous}. A correct node cannot wait to receive messages from all the other nodes since it can be indefinitely stalled by a single \byzantine{} node that chooses not to send any message. This amplifies the challenge of robustness. A simple adaption of server based solutions~\cite{brute_bulyan, kardam, chen2017distributed, yin2018byzantine, karimireddy2022byzantine, farhadkhani2022byzantine} cannot prevent the local models at the correct nodes from {\em drifting} apart, rending their local gradients useless for the others.

{\bf Robustness.} We consider an arbitrary set $\H$ comprising $n-f$ correct nodes. We denote by $\loss^{(\H)} \left( \model{}{} \right)$ the average loss function of the nodes in $\H$, defined in~\eqref{eqn:global_loss}. The ideal objective of {\em robust} learning is to design an algorithm that allows the correct nodes, under the above communication model, to minimize $\loss^{(\H)} \left( \model{}{} \right)$, despite the presence of \byzantine{} nodes. Solving this problem however is NP-hard in general, as the loss function need not be convex~\cite{boyd2004convex}. Thus, a more realizable goal is finding a {\em critical point} of $\loss^{(\H)} (\model{}{})$, assuming the point-wise loss function $\lossperdatapoint(\model{}{}, x)$ to be differentiable in $\model{}{}$. In our context, we formally define the problem of robustness through the notion of \emph{resilience}.

\begin{definition}
\label{def:Byz_learn}
An algorithm is said to be {\em $(f,\epsilon)$-resilient} if it enables each correct node $i \in \H$ to compute $\returnedmodel{\honestnode}$ such that
$$\expect{\norm{\nabla \loss^{(\H)} \left( \returnedmodel{\honestnode} \right) }^2} \leq \epsilon, $$
despite the presence of $f$ \byzantine{} nodes, where the expectation $\expect{\cdot}$ is taken over the randomness of the algorithm.
\end{definition}

We assume the gradients of the loss functions to be Lipschitz smooth and the variance of the local stochastic gradients to be bounded. These assumptions are classical to the analysis of stochastic first-order methods, and hold true for many learning problems~\cite{bottou2018optimization}. Note that, by definition of $\localloss{i}$, we have $\localgrad{i}(\model{}{}) = \condexpect{\datapoint \sim \mathcal{D}_i}{\nabla \lossperdatapoint(\model{}{}, \datapoint)}$.

\begin{assumption}[Lipschitz smoothness]
\label{asp:lip}
There exists $\lipschitz < \infty$ such that for all $i \in \H$ and all $\model{}{1}, \, \model{}{2} \in \R^d$,
\begin{equation*}
    \norm{\localgrad{i}(\model{}{1}) - \localgrad{i}(\model{}{2})}{} \leq \lipschitz \norm{\model{}{1} - \model{}{2}}{}.
\end{equation*}
\end{assumption}

\begin{assumption}[Bounded variance]
\label{asp:bnd_var}
There exists $\gradientstddev < \infty$ such that for all $i \in \H$, and all $\model{}{} \in \R^d$,
\begin{equation*}
    \condexpect{\datapoint \sim \mathcal{D}_i}{\norm{\nabla \lossperdatapoint(\model{}{}, \datapoint) - \localgrad{i}(\model{}{})}^2} \leq \gradientstddev^2.
\end{equation*}
\end{assumption}

Additionally, we assume that the diversity amongst the gradients of local loss functions is bounded, as stated below. We note that this assumption is standard in {\em heterogeneous settings}, i.e., when nodes have different data distributions~\cite{lian2018asynchronous, TangLYZL18}, especially when addressing the problem of resilience~\citep{data21}. 

\begin{assumption}[$\heter$-heterogeneous]
\label{asp:heter}
 There exists $\heter < \infty$ such that for all $\theta \in \R^d$,
 \begin{align*}
     \frac{1}{\card{\H}} \sum_{i \in \H} \norm{\nabla Q^{(i)} (\theta) - \nabla Q^{(\H)} (\theta) }^2 \leq \heter^2.
 \end{align*}
\end{assumption}
In particular, the heterogeneity bound $\heter$ can be derived based on the closeness of the underlying local data distributions at the nodes~\cite{FallahMO20}. 

{\bf Lower bound.} Under heterogeneity, it is generally impossible to achieve $(f, \, \epsilon)$-resilience for any arbitrary value of $\epsilon$~\citep{collaborativeElMhamdi21,karimireddy2022byzantine}. Specifically, we have the following lower bound.
\begin{lemma}[Theorem III,~\citet{karimireddy2022byzantine}]
\label{lem:lower_bound}
Suppose assumptions~\ref{asp:lip}, \ref{asp:bnd_var}, and \ref{asp:heter} hold true. If an algorithm is $(f,\epsilon)$-resilient, then $$\epsilon \in \Omega\left(\frac{f}{n} \,  \heter^2 \right).$$
\end{lemma}

\section{\newalgorithm{}}\label{sec:algorithm}
We describe below our algorithm, \newalgorithm{}, its (per step) computational costs and its robustness properties. 

\subsection{Description}
\newalgorithm{} enhances D-SGD~\cite{TangLYZL18} by incorporating a momentum (Mo) operation~\cite{POLYAK19641} as well as a new mixing scheme named nearest neighbor averaging (NNA). We summarize below the key elements of the local phase and the coordination phase in an iteration $t$ where each correct node $i$ maintains a local model $\model{i}{\iteration}$. { The initial models for the correct nodes are assumed identical, i.e., each correct node $i$ chooses an initial model $\model{i}{0}$ such that $\model{i}{0} \coloneqq \theta_0 \in \R^d$.} Complete execution of \newalgorithm{} is presented in Algorithm~\ref{algo}. 
    {\bf Local phase.} Each correct node $i$ samples a data point $\datapoint \sim \D_i$ and computes a local stochastic gradient 
    \begin{align}\label{eqn:local_grad}
        \gradient{i}{\iteration} = \nabla \lossperdatapoint \left( \model{i}{\iteration}, \datapoint \right).
    \end{align}
    Then node $i$ updates its current local momentum as follows
    \begin{align}
    \label{eqn:mmt_i}
        \mmt{\honestnode}{\iteration} = \mmtcoefficient \mmt{\honestnode}{\iteration-1} + (1-\mmtcoefficient) \gradient{\honestnode}{\iteration},
    \end{align}
    where $\beta \in [0, \, 1)$ is called the \emph{ momentum coefficient}, and $\mmt{i}{-1} = 0$ by convention. Lastly, node $i$ {\em partially} updates it current model $\model{i}{\iteration}$ by computing $\model{\honestnode}{\iteration + 1/2} := \model{\honestnode}{\iteration} - \learningrate{} \, \mmt{\honestnode}{\iteration}.$

    \textbf{Coordination phase.} Each correct node $i$ initializes a new vector $\cvavector{\honestnode}{0} = \model{\honestnode}{\iteration + \nicefrac{1}{2}}$. The coordination phase is composed of $K \geq 1$ rounds. In each round $k \in K$, the following interaction and mixing schemes are executed. 
    \vspace{-0.3cm}
    \begin{itemize}[leftmargin=7pt]
    \setlength{\itemsep}{0.1em}
        \item[] {\em (a) Interaction.} The nodes exchange their respective vectors $\{ \cvavector{\honestnode}{\cvaround-1}, \, i \in n\}$ with each other using {\em signed echo broadcast}~\cite{cachin2011introduction}.\footnote{For pedagogical reasons, we defer the implementation details of signed echo broadcast (SEB) to Appendix~\ref{app:cb}. Essentially, SEB prevents a \byzantine{} node from sending mismatching messages to different correct nodes. }
        Recall that a \byzantine{} node $j$ may choose to send either an arbitrary value for its vector $\cvavector{j}{\cvaround-1}$ or no message at all. Hence, to avoid getting stalled, a correct node $i$ only waits to receive $n-f-1$ messages before moving to the mixing step. 
        \item[] {\em (b) Mixing.} Each correct node $i$ updates its vector $\cvavector{\honestnode}{\cvaround-1}$ to $\cvavector{\honestnode}{\cvaround}$ by aggregating the $n-f-1$ vectors it receives with its own, using {\em nearest neighbor averaging} (NNA). For $n-f$ vectors $z^{(0)}, \ldots, \, z^{(n-f-1)}$ in $\R^d$, this aggregation is defined to be
    \begin{align*}
    \label{eq:nna}
        \cva{}\left( z^{(0)} ; \left \{z^{(i)} \right\}_{i=1}^{n-f-1} \right) \coloneqq \frac{1}{\nodenumber - 2\byzantinebound} \sum_{i = 0}^{ n-2f-1 } z^{(\tau(i))}, 
    \end{align*} 
    where $\tau$ is a permutation on $\{1, \ldots, n-f-1\}$ such that 
    $\norm{z^{(0)} - z^{(\tau(1))}} \leq \ldots \leq \norm{z^{(0)} - z^{(\tau(n-f-1))}}$.
    \end{itemize}

\begin{algorithm}[htb!]
\small
{\bf Initialization:} Initial model $\model{i}{0} \coloneqq \theta_0 \in \R^d$ , initial momentum $\mmt{\honestnode}{-1} = 0$, momentum coefficient $\mmtcoefficient \in [0, \, 1)$, total iterations $T$, learning rate $\learningrate{}$, number of coordination rounds $\numbercvarounds$, and threshold $\byzantinebound$ on the number of faulty nodes. \\

For each {\bf iteration} $t=1,\dots, \, T$, do the following: \\[-0.5cm]
    \begin{enumerate}[leftmargin=0.2cm]
    \item[] \textcolor{blue}{ \bf Local phase:} 
    \begin{enumerate}
    \item Update local momentum $\mmt{i}{t} = \beta \mmt{i}{t-1} + (1 - \beta) \gradient{i}{t}$ where $\gradient{i}{t}$ is defined as in~\eqref{eqn:local_grad}. 
    \item Partially update local model $\model{\honestnode}{\iteration + 1/2} := \model{\honestnode}{\iteration} - \learningrate{} \, \mmt{\honestnode}{\iteration}$. 
    \end{enumerate}
    \item[] \textcolor{blue}{ \bf Coordination phase:} Initialize vector $\cvavector{\honestnode}{0} := \model{\honestnode}{\iteration+1/2}$ and execute the following $K$ rounds.
    \begin{enumerate}
    \item In each {\bf round} $k=1,\dots, \, K$, do the following 
    
\fcolorbox{black}{gray!10}{
\parbox{0.39\textwidth}{
       \begin{enumerate}[leftmargin=0.5cm]
        \item Initialize $\receive{\honestnode}{\cvaround} = \emptyset$. 
        \item Broadcast vector $\cvavector{\honestnode}{\cvaround-1}$ to the other nodes.\\
        \textcolor{teal}{ \em (A \byzantine{} node $j$ may send an arbitrary value for $\cvavector{j}{\cvaround-1}$)}
        \item \textbf{While} $\vert \receive{\honestnode}{\cvaround} \vert < n-f-1$ \textbf{do}: \\
         \ \ \ \ \ Upon receiving a vector from node $j$, update $\receive{\honestnode}{\cvaround} = \receive{\honestnode}{\cvaround} \cup \{ j \}$. 
        \item Compute\\ $\cvavector{i}{k} = \textsc{NNA}\left( \cvavector{i}{k-1}   ; \left \{ \cvavector{j}{k-1} \mid j \in \receive{\honestnode}{\cvaround} \right \} \right)$.
    \end{enumerate} }}
    \item Update local model $\model{\honestnode}{\iteration+1} = \cvavector{\honestnode}{K}$.
    \end{enumerate}
    \end{enumerate}
\vspace{0.1cm}
{\bf Output:} $\returnedmodel{\honestnode} \sim \mathcal{U} \left\{ \model{\honestnode}{0}, \ldots, \, \model{\honestnode}{T-1}\right\}.$
\caption{\textbf{\newalgorithm{}} as executed by a correct node $\honestnode$}
\label{algo}
\end{algorithm}

\subsection{Computation \& Communication Costs}

Computing a local momentum as per~\eqref{eqn:mmt_i} is equivalent in terms of complexity to computing a single local gradient. Thus, the computational cost of the local phase in \newalgorithm{} is the same as that of D-SGD. Second, the coordination phase in \newalgorithm{} comprises $K$ rounds in which each correct node computes the output of NNA. This involves computing $n-f-1$ distances in $\mathbb{R}^d$ and sorting them to obtain $\tau$. The former is in $\mathcal{O}\left( n d\right)$ and the latter can be done using a sorting algorithm, e.g., {\em quicksort}~\citep{cormen2022introduction}, in $\mathcal{O}\left(n \log n \right)$. Hence, the total computation cost for the coordination phase of \newalgorithm{} is in $\mathcal{O}\left(n \left(d+\log n \right) K \right)$, compared to $\mathcal{O}\left(nd\right)$ for D-SGD. Similarly, the communication cost of \newalgorithm{} is in $\mathcal{O}\left(n K \right)$, which is a factor $K$ more than that of D-SGD. 

Constant $K$ is however usually relatively small compared to the standard costs of D-SGD. Indeed, in the main result of the paper (Theorem~\ref{thm:main_conv}) when $n \geq 11f$, we set $K = 1$. Therefore, in this case, the computational and communication complexity of \newalgorithm{} is $\mathcal{O}\left(n \left(d+\log n \right)\right)$ and $\mathcal{O}\left(n\right)$, respectively, which almost matches the $\mathcal{O}\left(nd\right)$ computational and $\mathcal{O}\left(n\right)$ communication complexity of D-SGD. Furthermore, to improve the robustness of \newalgorithm{} to $n > 5f$, we set $K = \mathcal{O} (\log n)$ which adds a $\log n$ overhead to the communication and computational costs (see Corollary~\ref{cor:fivef} in the Appendix), but remains reasonable compared to other existing solutions such as~\cite{collaborativeElMhamdi21}.

\subsection{Convergence \& Robustness}

We now present our main theoretical result demonstrating the finite time convergence of \newalgorithm{}.
Essentially, we analyze Algorithm~\ref{algo} under assumptions~\ref{asp:lip},~\ref{asp:bnd_var}, and~\ref{asp:heter}, and upon assuming a sufficiently small learning rate $\gamma$. When $n\geq 11f$, it suffices to perform one round per coordination phase, i.e., set $K = 1$.
We now state our main theorem\footnote{The dependence of $c_4$ on $n$ comes from the fact that we provide the convergence guarantee for any honest node $i$. This is stronger than the prior work \cite{koloskova2020unified,he22} where the convergence guarantee is often given for the average of the local models  $\avgmodel{t}$.}, upon introducing the following notation:
\begin{gather*}
Q^* = \min_{\theta \in \mathbb{R}^d} Q^{(\H)}\left( \model{}{} \right), \quad 
\avgmodel{t} := \frac{1}{\vert \H \vert}\sum_{i \in \H}\model{i}{t}.
\end{gather*}

\noindent \fcolorbox{black}{white}{
\parbox{0.45\textwidth}{\centering
\begin{theorem}
\label{thm:main_conv}
Suppose that assumptions~\ref{asp:lip},~\ref{asp:bnd_var} and~\ref{asp:heter} hold true, and that $n \geq 11f$. Let us denote 
\begin{gather*}
\alpha = \frac{9.88 f}{n-f} \leq  0.988, \quad \lambda = \frac{9f}{n-f},\\
c_0 := 12\left(\avgloss\left(\avgmodel{0}\right) - Q^*\right), \quad c_1 :=   \frac{18 \alpha (1 + \alpha)}{(1 - \alpha)^2},\\
c_2 := 72  L  \left (\frac{3}{n-f} + 2 c_1 + \frac{9\lambda}{2}\left(2c_1 +3\right) \right),\\ c_3:=6 \left(6 c_1 + \frac{9\lambda}{2}\left( 4c_1 + 9\right)\right) \text{and }     c_4 := \frac{9 n c_0 c_1}{c_2}.
\end{gather*}

Consider Algorithm~\ref{algo} with $K = 1$, $\gamma = \min \left\{\frac{1}{12L}, \frac{1}{L} \sqrt{\frac{2}{3c_1}}, \sqrt{\frac{c_0}{c_2LT\sigma^2}} \right\},$ and $\beta = \sqrt{ 1 - 12 \gamma L}$. Then, for all $T\geq1$ and $i \in \H$, we have
\begin{align*}
    &\expect{\norm{\nabla \avgloss \left( \returnedmodel{\honestnode} \right)}^2} \leq 2\sqrt{\frac{c_0c_2L\sigma^2}{T}} + \frac{12Lc_0}{T} \\
    &+\frac{Lc_0}{T}\sqrt{\frac{3c_1}{2}}  + \frac{36}{T} \left(\frac{\sigma^2}{n-f} \right) + \frac{c_4 L}{T} \left( 1+ \frac{\heter^2}{\sigma^2}\right) 
    + c_3 \heter^2.
\end{align*}  

\end{theorem}
}}

Using Theorem~\ref{thm:main_conv}, we can show that Algorithm~\ref{algo} guarantees $(f, \epsilon)$-resilience. Specifically, upon ignoring the higher-order terms in $T$, we obtain the following corollary.

\begin{corollary}
\label{cor:maincor}
Under the conditions stated in Theorem~\ref{thm:main_conv},
Algorithm~\ref{algo}
guarantees $(f,\epsilon)$-resilience where
\begin{align*}
\epsilon \in \mathcal{O}\left(\sqrt{\frac{\sigma^2}{T}\left( \frac{1+f}{n} \right)} + \frac{f}{n}\heter^2\right). \end{align*}
\end{corollary}

\textbf{Linear gradient overhead.} In the fault-free setting, i.e., when $f=0$, the convergence result shown in Corollary~\ref{cor:maincor} reduces to that of the conventional D-SGD~\cite{TangLYZL18}. However, when $f >0$, \newalgorithm{} induces an overhead on the number of gradients computed per correct nodes compared to D-SGD. Specifically, correct nodes in \newalgorithm{} compute $(1+f)$ times more gradients than in the fault-free case (to obtain a comparable error), which is linear in $f$. We believe this gradient overhead to be tight, as conjectured in~\cite{Karimireddy2021}. While \newalgorithm{} only imposes a linear overhead under the assumption that $f \leq \nicefrac{n}{11}$,
it can tolerate a larger fraction of \byzantine{} nodes, i.e., arbitrarily close to $\nicefrac{n}{5}$, by imposing a quadratic gradient overhead in $f$ (see Corollary~\ref{cor:fivef} in Appendix~\ref{corr_5f}). 
\section{Convergence Analysis}\label{sec:analysis}
    We now explain the key elements that we build upon to prove the convergence guarantee stated in Theorem~\ref{thm:main_conv}. Essentially, we first introduce the mixing criterion of $(\alpha, \lambda)$-reduction and, then, show how to control the drift in local updates. %

\subsection{$(\alpha, \, \lambda)$-\reduction{}}

To handle the non-linear mixing of correct momentums, we introduce the notion of $(\alpha, \, \lambda)$-\reduction{}. This notion can be seen as a relaxation of the classical linear mixing criterion of double stochasticity with a bounded spectral gap. We show that $(\alpha, \, \lambda)$-\reduction{} is sufficient to maintain tight convergence guarantees, while it can be satisfied by non-linear and even non-continuous schemes such as NNA.

\begin{definition}[{\em $(\alpha,\lambda)$-\reduction{}}]
\label{def:reduction}
Consider a coordination phase $\Psi$. For correct nodes $i \in \H$ initiating the coordination phase with vectors $\left\{z^{(i)}, \, i \in \H\right\}$, we denote by $\left\{y^{(i)}, \, i \in \H\right\}$ the vectors obtained by these nodes upon the completion of $\Psi$. Then, $\Psi$ is said to guarantee {\em $(\alpha, \, \lambda)$-reduction} if, for any $\left\{z^{(i)}, \, i \in \H\right\}$, the following holds true:
  \begin{align*}
       &\text{i)}  &\frac{1}{\card{\H}} \sum_{i\in \H} \norm{y^{(i)} - \bar{y}}^2  \quad &\leq \quad \alpha \frac{1}{\card{\H}} \sum_{i\in \H} \norm{z^{(i)} - \bar{z}}^2 \\
      &\text{ii)}   &\norm{ \bar{y}- \bar{z}}^2 \quad &\leq \quad \lambda \frac{1}{\card{\H}} \sum_{i\in \H} \norm{z^{(i)} - \bar{z}}^2
  \end{align*}
  where $\bar{z}$ and $\bar{y}$ denote the vector averages of $\left\{z^{(i)}, \, i \in \H\right\}$  and $\left\{y^{(i)}, \, i \in \H\right\}$, respectively.
\end{definition}

In \newalgorithm{}, each correct node $i$ initializes the coordination phase with vector $z^{(i)} =\cvavector{i}{0}$ and outputs $y^{(i)} = \cvavector{i}{K}$ at its completion. In Appendix~\ref{app:reductionNNA}, we show that the coordination phase of Algorithm~\ref{algo} satisfies $(\alpha, \, \lambda)$-\reduction{} for $\lambda  \in \Theta\left( \nicefrac{f}{n}\right)$, $\alpha  <1$, and  $\alpha  \in \Theta\left( \nicefrac{f}{n}\right)$ when $n \geq 11f$. Additionally, when $n > 5f$, we have $\lambda  \in \Theta\left( \nicefrac{f^2}{n}\right)$, $\alpha <1$, and $\alpha  \in  \Theta\left( \nicefrac{f}{n}\right)$ (shown in Appendix~\ref{sec:app:fivef}).

\subsection{D-SGD with $(\alpha, \, \lambda)$-\reduction{}}
To better understand the utility of $(\alpha, \, \lambda)$-\reduction{}, we first provide a convergence guarantee for \newalgorithm{} {\em without} momentum (i.e., $\beta = 0$), while assuming the communication phase to satisfy $(\alpha, \, \lambda)$-\reduction{}. Alternately, this algorithm reduces to D-SGD with 
$(\alpha, \, \lambda)$-\reduction{} mixing.

\begin{proposition}
\label{prop:dsgd_main}
Consider Algorithm~\ref{algo} with $\beta=0$. 
Suppose that assumptions~\ref{asp:lip},~\ref{asp:bnd_var} and~\ref{asp:heter} hold true, and that the coordination phase satisfies $(\alpha,\lambda)$-\reduction{} for $\alpha < 1$. Then there exists a constant learning rate $\gamma$ for which each correct node $i \in \H$ returns $\returnedmodel{\honestnode}$ such that
\begin{align*}
\expect{\norm{\nabla \loss^{(\H)} \left( \returnedmodel{\honestnode} \right) }^2} 
&\in \mathcal{O}\left( \sqrt{\frac{\sigma^2}{nT}}+ \frac{\lambda \left( \sigma^2 + \heter^2 \right)}{(1-\alpha)^2}\right).
\end{align*}
\end{proposition}

Notably, when all the nodes are correct (i.e., $f = 0$), then the coordination phase of Algorithm~\ref{algo} simply computes the average and satisfies $(\alpha, \lambda)$-\reduction{} for $\alpha = \lambda = 0$. Then, we recover the classical convergence guarantee of D-SGD without \byzantine{} nodes, i.e., $\mathcal{O}\left( \sqrt{\nicefrac{\sigma^2}{nT}}\right)$. 
However, in the presence of \byzantine{} nodes (when $\lambda \in \Theta\left(\nicefrac{f}{n}\right)$, and $\alpha <1$), Proposition \ref{prop:dsgd_main} shows an asymptotic error of $\frac{f}{n} \sigma^2 + \frac{f}{n}\heter^2$. While the term depending on $\heter^2$ is a fundamental lower bound as per Lemma~\ref{lem:lower_bound}, the one depending on $\sigma^2$ can be alleviated through the use of momentum, as we show next.

\subsection{Polyak's Momentum}
\label{sec:momentum_aanl}
Although momentum has been shown to be beneficial in the particular case of server-based coordination~\cite{Karimireddy2021, farhadkhani2022byzantine}, extending the existing analyses to our setting is not straightforward. The main bottleneck is the non-trivial drift that occurs between the local parameters maintained by correct nodes, i.e., $ \sum_{i\in \H} \norm{\model{i}{t} - \bar{\model{}{t}}}^2$. When momentum is applied, this drift gets coupled with the drift between their momentum vectors, i.e.,  $\sum_{i \in \H} \norm{\mmt{i}{t} - \AvgMmt{t}}^2$.
Indeed, an elementary analysis of this coupling suggests the possibility of uncontrolled growth of the two drifts: a high model drift increases the momentum drift and vice versa. We devise a refined analysis, showing that an appropriate learning rate, which is a function of parameter $\alpha$ in $(\alpha, \, \lambda)$-\reduction{}, ensures uniform bounds proportional to $(1-\beta) \sigma^2$ for both model and momentum drifts (Lemma~\ref{lem:drift} in Appendix~\ref{app:proofs}). This suggests that we can diminish the dependence on $\sigma^2$ by choosing a large momentum parameter close to $1$. However, an arbitrarily large momentum parameter increases the bias, i.e., 
 \begin{align*}
     \dev{t} := \frac{1}{\card{\H}} \sum_{i \in \H} \left( \mmt{i}{t} - \nabla \loss^{(i)}\left(\model{i}{t}\right)\right),
 \end{align*}
 which in turn negatively impacts the convergence. To prove that the positive effect of momentum (i.e., reduction in drift) outweighs the negative effect (i.e., increase in bias), we 
 define our Lyapunov function (or {\em potential} function) to be 
\begin{align*}
    V_t \coloneqq \expect{\avgloss\left( \avgmodel{t} \right) - Q^* + \frac{1}{4L} \norm{\dev{t}}^2}.
\end{align*}
The above Lyapunov function is inspired from the existing momentum literature~\cite{mvr19,farhadkhani2022byzantine}, but adapted to 
address the model drift.
\section{Empirical Evaluation}\label{sec:experiments}
We compare the performance of \newalgorithm{} to D-SGD and several state-of-the-art algorithms from the literature. In particular, we consider three methods, namely BRIDGE~\cite{yang2019bridge}, SCC~\cite{he22}, and LEARN~\cite{collaborativeElMhamdi21}\footnote{We do not implement BTARD~\cite{Gorbunov21} due to its weaker adversarial model and assumption of a public data pool.}. We consider two classical image classification tasks, on which we test the robustness of \newalgorithm{} under four attacks. Every experiment is repeated five times using seeds 1 to 5 for reproducibility. Our code will be made available online.

\subsection{Experimental Setup}\label{sec:expSetup}

\paragraph{Datasets.} We use the MNIST~\cite{mnist} and CIFAR-10~\cite{cifar} datasets, pre-processed as in~\cite{little} and \cite{distributed-momentum}. Refer to Appendix~\ref{app:pre_process} for more details on pre-possessing.

\begin{figure*}[!ht]
    \centering
    \includegraphics[width=71mm]{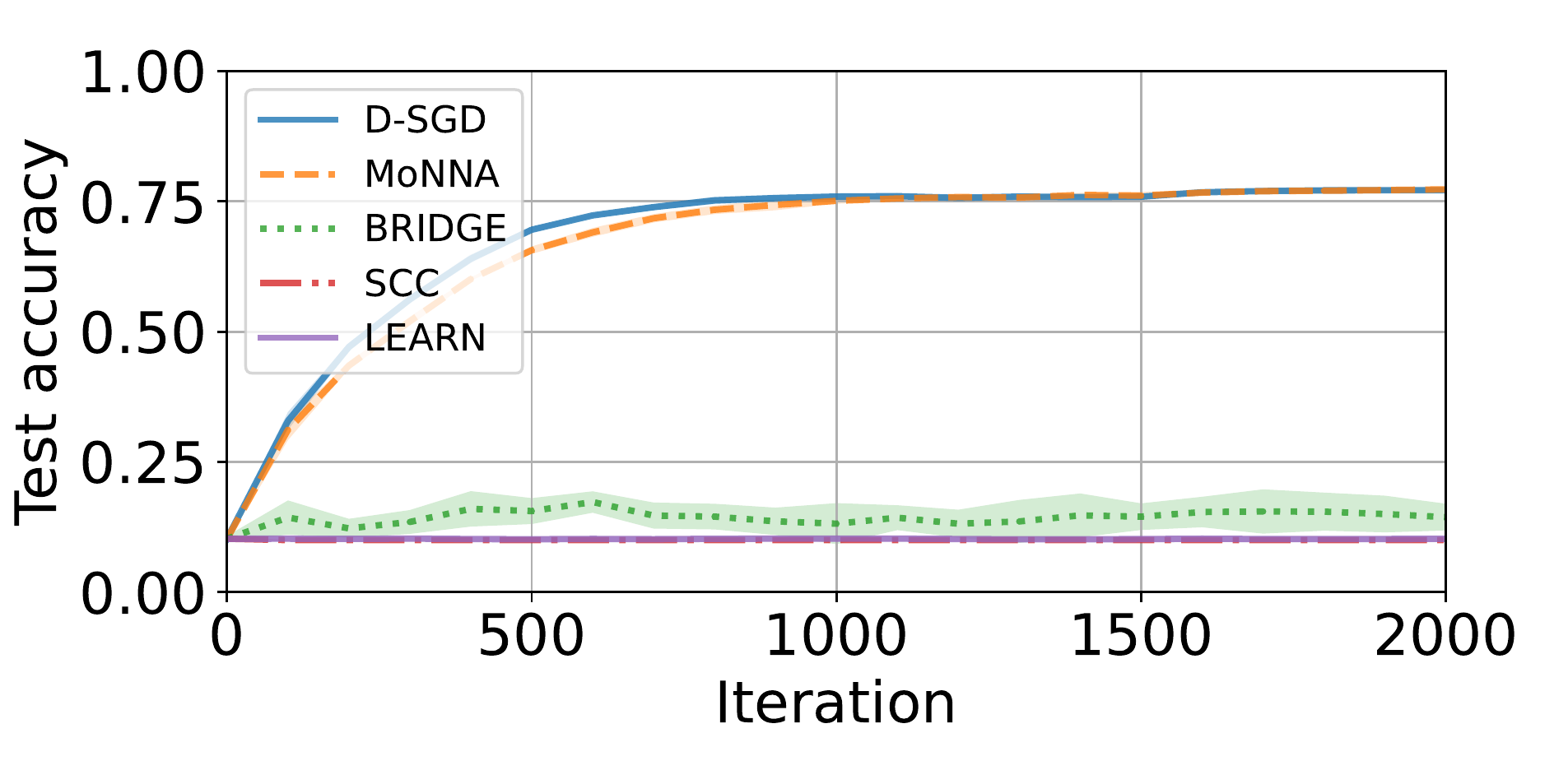}%
    \includegraphics[width=71mm]{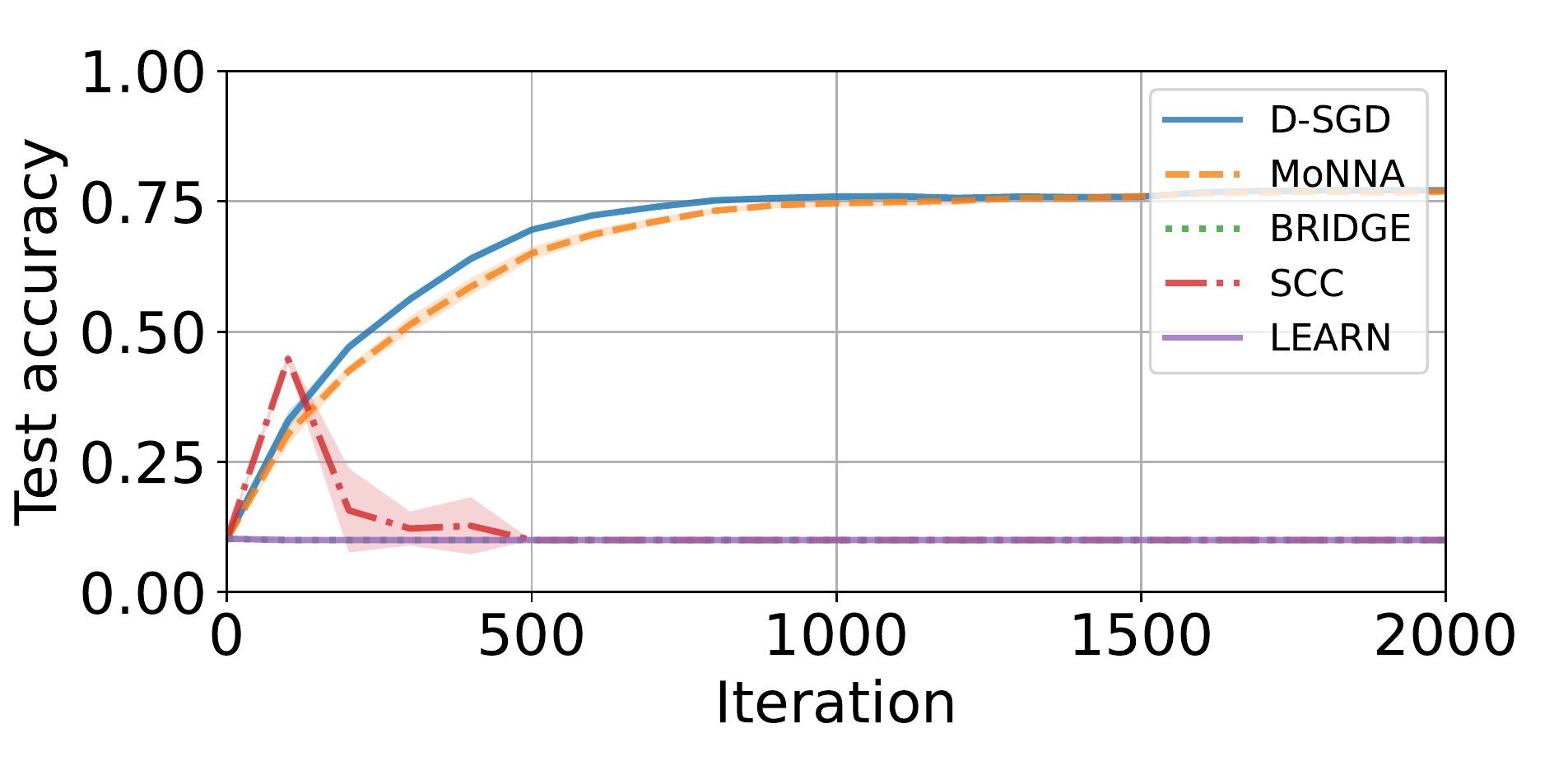}\\ \vspace{-0.2cm}    \includegraphics[width=71mm]{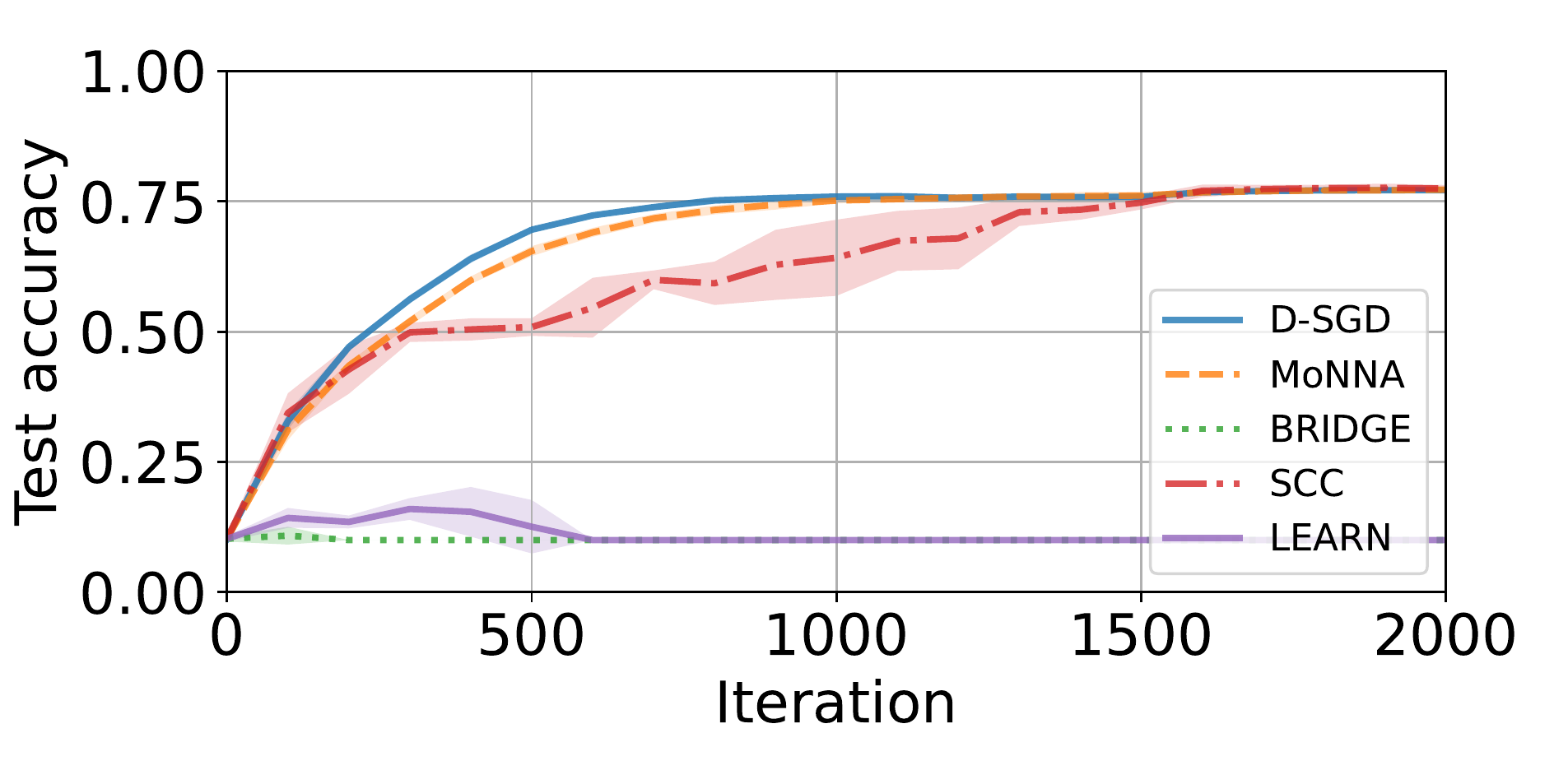}%
    \includegraphics[width=71mm]{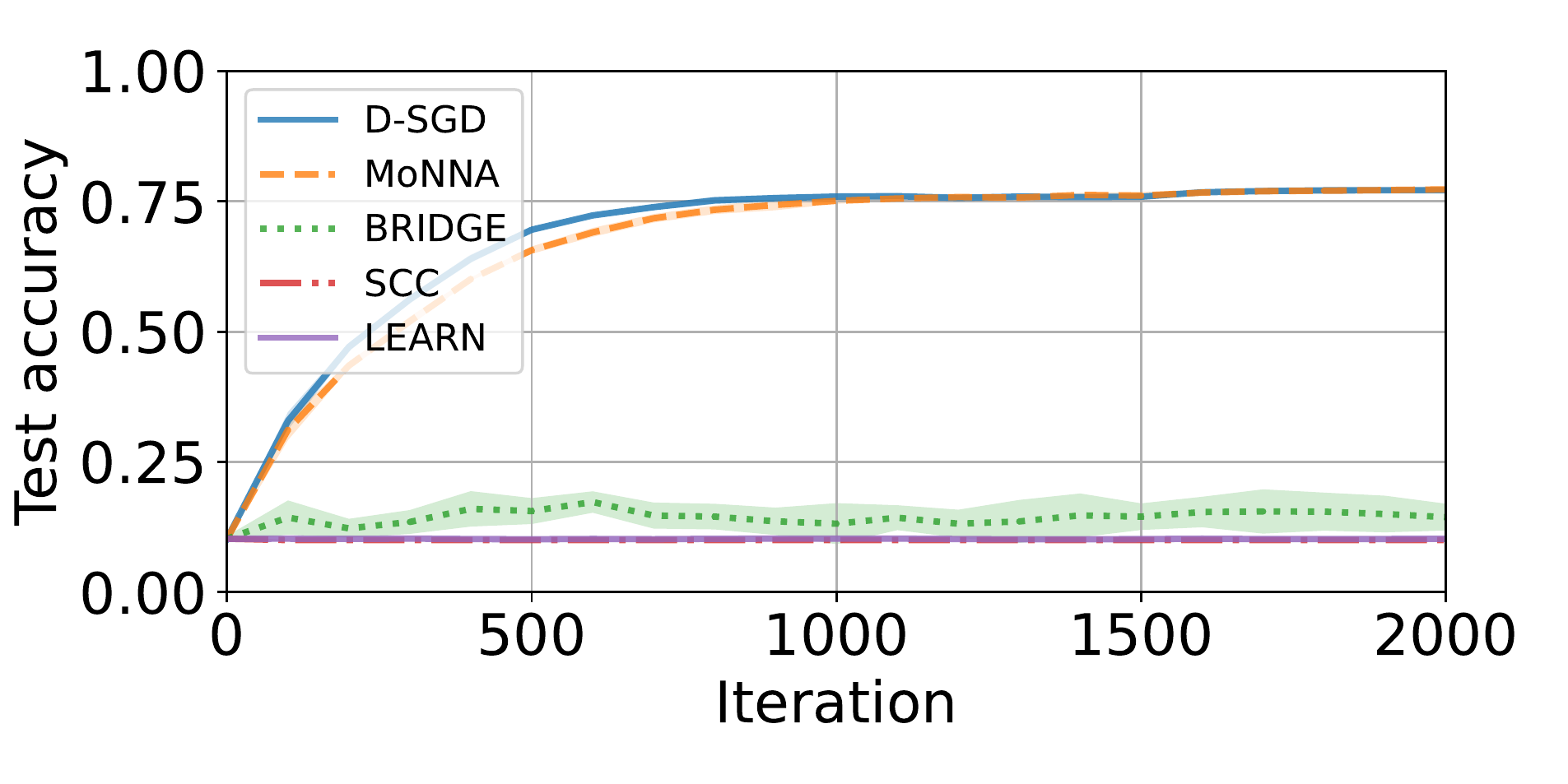}\\ \vspace{-0.2cm}
    \caption{Learning accuracies achieved on CIFAR-10 with $\alpha = 5$ by D-SGD, \newalgorithm{}, BRIDGE, SCC, and LEARN. There are $n = 16$ nodes out of which $f = 3$ are \byzantine{}. The \byzantine{} nodes execute the \textit{FOE} (row 1, left), \textit{ALIE} (row 1, right), \textit{LF} (row 2, left), and \textit{SF} (row 2, right) attacks. All algorithms except LEARN compute 100,000 gradients, while LEARN computes 2,001,000 gradients.}
    \label{fig:experiments-1}
\end{figure*}

\vspace{-0.2cm}
\paragraph{Model architecture and hyperparameters.}
For MNIST, we consider a convolutional neural network (CNN) with two convolutional layers followed by two fully-connected layers. The model is trained using a learning rate $\gamma = 0.75$ for $T = 600$ iterations. We use a total number of nodes $n = 26$, out of which $f = \nicefrac{(26-1)}{5} = 5$ are \byzantine{}. For CIFAR-10, we use a CNN with four convolutional layers and two fully-connected layers. Furthermore, we set $\gamma = 0.5$ and $T = 2000$ iterations. The distributed system in this case consists of $n = 16$ nodes, out of which $f = \nicefrac{(16-1)}{5} = 3$ are \byzantine{}. A detailed presentation of the entire experimental setup can be found in Appendix~\ref{app:model_arch}.

\vspace{-0.2cm}
\paragraph{Heterogeneity.} In order to simulate data heterogeneity, the correct nodes sample from the original datasets using a Dirichlet distribution of parameter $\alpha$, as done in~\cite{dirichlet}. We evaluate our algorithm on MNIST with $\alpha \in \{1, 5\}$, and on CIFAR-10 with $\alpha = 5$. A pictorial representation of the resulting heterogeneity as a function of $\alpha$ can be found in Appendix~\ref{app_exp_setup_distribution}.

\vspace{-0.2cm}
\paragraph{Attacks and asynchrony.}
We consider four state-of-the-art attacks performed by the \byzantine{} nodes, namely \textit{a little is enough (ALIE)}~\cite{little}, \textit{fall of empires (FOE)}~\cite{empire}, \textit{sign-flipping (SF)}~\cite{allen2020byzantine}, and \textit{label-flipping (LF)}~\cite{allen2020byzantine}. These attacks are explained in detail in Appendix~\ref{app:attacks}.
In order to emulate the ill-effects of asynchrony, we ensure that the correct nodes receive first the messages of the \byzantine{} nodes. Put differently, to construct any correct node $i$'s set of $n-f-1$ first received messages $\receive{\honestnode}{\cvaround}$, we first insert the messages sent by \byzantine{} nodes. We then complete the set by randomly sampling the remaining messages from the correct vectors.

\vspace{-0.2cm}
\paragraph{Evaluation details.}
To serve as a benchmark for \newalgorithm{}, we run D-SGD in a non-adversarial environment (i.e., without faults), and with momentum $\beta = 0.99$. We also execute \newalgorithm{} with $\beta = 0.99$ and $K = 1$ (i.e., one coordination round per iteration), and report on its performance in four adversarial settings. Furthermore, we execute SCC with $\beta = 0.9$ (fine-tuned as in~\cite{he22}), and LEARN without momentum as prescribed in~\cite{collaborativeElMhamdi21}. We use the same number of iterations $T$ for all algorithms and compare their learning accuracies and computational workloads per node.

\subsection{Experimental Results}\label{sec:experiment-results}

Figure~\ref{fig:experiments-1} showcases the performance of \newalgorithm{} compared to the other algorithms. For space limitations and better readability, we only show in Figure~\ref{fig:experiments-1} our results on CIFAR-10. The remaining results on MNIST are deferred to Appendix~\ref{app:exp_results}, and convey the same observations made hereafter.

Figure~\ref{fig:experiments-1} clearly shows the empirical superiority of \newalgorithm{} in four adversarial settings. Indeed, \newalgorithm{} is the only solution that performs consistently well under all the four attacks, almost matching the performance of D-SGD without faults. Its closest rival among the considered techniques is SCC. Even then, while SCC displays comparable performance to \newalgorithm{} under the LF attack, its learning capabilities drop significantly when tested against the remaining three attacks, especially FOE and SF that make the accuracy of SCC fall to 10\%. Moreover, BRIDGE and LEARN also present poor performances under all four attacks, their final accuracies stagnating at around 10\%.

Note also that the number of gradients each node computes when using LEARN is 20 times more than that in \newalgorithm{}. The inflated computational costs associated to LEARN are explained by the dynamic sampling technique the algorithm implements, whereby the batch-size is gradually augmented across iterations. \newalgorithm{}, BRIDGE, and SCC compute the same number of gradients per node as D-SGD since they all share a constant batch-size during the entire learning. In summary, our results show the empirical superiority of our algorithm in adversarial settings since \newalgorithm{} matches the performance of fault-free D-SGD, both in terms of learning accuracy and computational complexity.

\textbf{Remark.}
Although our empirical evaluation conveys a poor performance of BRIDGE and LEARN, it is important to note that we do not contradict the previous findings on these methods reported in~\cite{yang2019bridge} and~\cite{collaborativeElMhamdi21}, respectively, that consider very weak attack models. In short, \cite{yang2019bridge} report on an evaluation of BRIDGE assuming \byzantine{} nodes that only send random vectors, instead of executing state-of-the-art attacks. On the other hand, \cite{collaborativeElMhamdi21} report on an evaluation of  LEARN in a fault-less system, i.e., without any attack. Additionally, as opposed to \newalgorithm{} and SCC, these techniques do not use local momentum, which has been recently recognized as a key ingredient in the robustness of distributed learning algorithms~\cite{farhadkhani2022byzantine, Karimireddy2021}. 
We further comment on the necessity of momentum in Appendix~\ref{app_results_ablation_study}.

\vspace{-0.2cm}
\section{Concluding Remarks}
\label{sec:conclusion}

We present \newalgorithm{}, a novel collaborative learning algorithm that is provably robust under standard learning assumptions. We show that  \newalgorithm{} has a linear gradient computation overhead in the fraction of faulty machines.
One of our main contributions is the introduction of the new mixing criterion of $(\alpha,\lambda)$-reduction, allowing us to obtain tight convergence guarantees.
Following prior works on robust collaborative learning~\cite{collaborativeElMhamdi21, Gorbunov21}, we studied this criterion under the pair-wise communication scheme, which comes with a high communication overhead. We aim to study $(\alpha,\lambda)$-reduction under sparse communication networks or client sub-sampling schemes to reduce the communication overhead and further improve the practical applicability of our method. 

\section*{Acknowledgments}
This work has been supported in part by the Swiss National Science Foundation (SNSF) projects 200021-200477 and 200021-182542.

\bibliography{ref}
\bibliographystyle{icml2023}

\newpage
\appendix
\onecolumn

\begin{center}
    \LARGE \bf {Appendix}
\end{center}

\section*{Organization}

The appendices are organized as follows:
\begin{itemize}
    \item Appendix~\ref{app:proofs} proves the convergence of \newalgorithm{} for $n > 11f$ (Theorem~\ref{thm:main_conv}) and $n > 5f$ (Corollary~\ref{cor:fivef}).  
    \item Appendix~\ref{app_sec:dsgd_anal} analyzes  D-SGD ($\beta = 0$) under $(\alpha,\lambda)$-\reduction{} (proof of Proposition~\ref{prop:dsgd_main}).
    \item Appendix~\ref{app:cb} explains Signed Echo Broadcast which supports the reliability of our communication model.
    \item Appendix~\ref{app:exp_setup} provides  additional information on our experimental setup.
    \item Appendix~\ref{app:exp_results} provides some additional experimental results.
\end{itemize}

\section{Convergence Proof for \newalgorithm{}}
\label{app:proofs}

In this section, we derive convergence guarantees of \newalgorithm{}. 
First, we prove the main result of the paper, Theorem~\ref{thm:main_conv}, for $n> 11f$. Then, in Section~\ref{sec:app:fivef}, we show that~\newalgorithm{} can actually tolerate a larger fraction of \byzantine{} nodes (i.e., $n>5f$) but with a
slightly worse convergence rate.

We first present below the skeleton of our main proof.

\subsection{Skeleton of the Proof of Theorem~\ref{thm:main_conv}}
\label{app_sec:monna_proof}
Our proof comprises 4 key steps, listed as follows.
\begin{enumerate}[label= , leftmargin=0.7cm]
    \item \textbf{Step-I:} Demonstrating that the coordination phase of Algorithm~\ref{algo} satisfies $(\alpha,\lambda)$-\reduction{}.
    \item \textbf{Step-II:} Analyzing the \emph{parameter drift} and the \emph{momentum drift.}
    \item \textbf{Step-III:} Analyzing the {\em momentum deviation} from the true gradient.
    \item \textbf{Step-IV:} Studying the {\em growth} of loss function $\avgloss$.
\end{enumerate}
To present the technical details, we introduce the following notation.

{\bf Notation:} We denote by $\mathcal{P}_t$ the history of nodes from steps $0$ to $t$. Specifically, we define
\[\P_t \coloneqq \left\{\model{i}{0}, \ldots, \, \model{i}{t}; ~ \mmt{i}{0}, \ldots, \, \mmt{i}{t-1}; ~ i = 1, \ldots, \, n \right\}.\] 
By convention, $\P_0 = \{ \model{i}{0}; ~ i = 1, \ldots, \, n \}$. Furthermore, we denote by $\condexpect{t}{\cdot} := \expect{\cdot ~ \vline ~ \P_t}$ the conditional expectation given the history $\mathcal{P}_t$, and by $\expect{\cdot}$ the total expectation over the randomness of the algorithm; thus, $\expect{\cdot} := \condexpect{0}{ \cdots \condexpect{T}{\cdot}}$. %
We recall that $\H$ denotes the set of correct nodes, and that $\card{\H} = n-f$. For an arbitrary $t$, we denote by $\diam{*}{t}$ the variance of respective correct nodes' local values, denoted by $*_t$, i.e., $\diam{*}{t}=\frac{1}{n-f}\sum_{i \in \H} \norm{*^{(i)}_{t} - \bar{*}_{t}}^2$, where $\bar{*}_{t} = \frac{1}{n-f} \sum_{i \in \H} {*^{(i)}_{t}}$ is the average of correct values. For instance,
\begin{align*}
    & \diam{\model{}{}}{t} = \frac{1}{n-f} \sum_{i\in \H} \norm{\model{i}{t} - \bar{\model{}{t}}}^2, ~ \diam{\mmt{}{}}{t} = \frac{1}{n-f} \sum_{i \in \H} \norm{\mmt{i}{t} - \bar{\mmt{}{t}}}^2 \\
    & \text{ and } \diam{\gradient{}{}}{t} = \frac{1}{n-f} \sum_{i \in \H} \norm{\gradient{i}{t} - \bar{\gradient{}{t}}}^2.
\end{align*}
We present below technical summaries of the aforementioned $4$ steps.

\subsubsection*{ \bf Step-I: Coordination phase with \cva{} and satisfies $(\alpha,\lambda)$-\reduction{}} 
Recall the definition of $(\alpha,\lambda)$-\reduction{} from  Definition~\ref{def:reduction}. Here we show that this condition is satisfied by the  coordination phase  of Algorithm~\ref{algo}.
 Recall that in each iteration of Algorithm~\ref{algo}, each node $i$ initializes the coordination phase with input $z^{(i)} = \cvavector{i}{0}$ and obtains the output vector  $y^{(i)} = \cvavector{i}{K}$ at its completion. Therefore, using the $\diam{}{\cdot}$ notation, to obtain  $(\alpha,\lambda)$-\reduction{},  we need to prove the following two conditions
  \begin{equation*}
    \diam{\cvavector{}{\numbercvarounds}}{} \leq \alpha \diam{\cvavector{}{0}}{} \quad \textnormal{and} \quad \norm{\cvaavg{0}-\cvaavg{\numbercvarounds}}^2 \leq \lambda \diam{\cvavector{}{0}}{}.
  \end{equation*}

  Here, we prove that these conditions are satisfied when $n\geq11f$. In Section~\ref{sec:app:fivef}, Lemma~\ref{lemma:cvaboundfivef} proves that the conditions are also satisfied when $n > 5f$, but with different values for $\alpha$ and $\lambda$.

\noindent \fcolorbox{black}{white}{
\parbox{0.97\textwidth}{\centering
\begin{lemma}
\label{lemma:cva_bounds}
  Suppose that  $n \geq 11 f$. For any $\numbercvarounds \geq 1$,  the coordination phase of Algorithm~\ref{algo} guarantees $(\alpha,\lambda)$-\reduction{}
  for 
  \begin{equation*}
      \alpha = \left(\frac{9.88 f}{n-f} \right)^K \quad \textnormal{and} \quad \lambda = \frac{9f}{n-f} \cdot \min \left\{ \numbercvarounds , \frac{1}{(1-\sqrt{\alpha})^2} \right\}.
  \end{equation*}
\end{lemma}
}
}

The proof of the lemma is provided in Section~\ref{app:reductionNNA}.

\subsubsection*{ \bf Step-II: Parameter drift and the momentum drift} Second, we note that, at any step $t$, neither the momentums $\mmt{i}{t}$ nor the parameters $\theta_t^{(i)}$ of the correct nodes are guaranteed to stay close to each other even when the stochastic gradients $\gradient{i}{t}$ come from a common gradient oracle. Yet, given our lemmas \ref{lemma:cva_bounds}, and \ref{lemma:cvaboundfivef}, we show in Lemma~\ref{lem:drift} below that the {\em drift} both between the correct nodes' momentums and between their parameters can be controlled by cleverly parametrizing the momentum coefficient $\beta$. Hence, we can guarantee approximate agreement on both the parameters and the momentums of the correct nodes. 

\noindent \fcolorbox{black}{white}{
\parbox{0.97\textwidth}{\centering
\begin{lemma}
\label{lem:drift}
Suppose that assumptions~\ref{asp:lip},~\ref{asp:bnd_var}, and \ref{asp:heter} hold true. Consider Algorithm~\ref{algo} with $\gamma \leq \frac{1 - \alpha}{L \sqrt{27 \alpha \left(1 + \alpha \right)}}$, and $\beta>0$. Suppose that the coordination phase satisfies $(\alpha,\lambda)$-\reduction{} for $\alpha < 1$. For each $t \in [T]$, we obtain that
\begin{align*}
    \expect{\diam{\model{}{}}{t}} &\leq \coeffalpha\gamma^2 \left( \sigma^2 \frac{1 - \beta}{1 + \beta} + 3 \heter^2 \right),
\end{align*}
and
\begin{align*}
    \expect{\diam{\mmt{}{}}{t}} \leq3 \sigma^2 \left(\frac{1 - \beta}{1 + \beta} \right) + 9 \heter^2 + 9 L^2 \gamma^2 \coeffalpha \left( \sigma^2 \frac{1 - \beta}{1 + \beta} + 3 \heter^2 \right),
\end{align*}
where 
\begin{align*}
    \coeffalpha:= \frac{18 \alpha (1 + \alpha)}{(1 - \alpha)^2}.
\end{align*}
\end{lemma}
}}

\subsubsection*{\bf Step-III: Momentum deviation.} Next, we study the \emph{deviation} of the average correct momentum $\AvgMmt{t}$ from the average of the true gradients $\AvgGrad{t}$, at step $t$. Let us denote by 
\begin{equation*}
    \AvgGrad{t} \coloneqq \frac{1}{n-f}\sum_{i \in \H} \localgrad{i} (\model{i}{t}),
\end{equation*}
the average of the true local gradient vectors at nodes' local models.
We define 
\begin{equation*}
    \dev{t} \coloneqq \AvgMmt{t} - \AvgGrad{t}.
\end{equation*}
We now have the following lemma.

\noindent \fcolorbox{black}{white}{
\parbox{0.97\textwidth}{\centering
\begin{lemma}
\label{lem:dev}
Suppose that assumptions~\ref{asp:lip} and~\ref{asp:bnd_var} hold true. Consider Algorithm~\ref{algo}. For all $t \in [T]$, we obtain that
\begin{align*}
    \expect{\norm{\dev{t+1}}^2} & \leq 
      \beta^2 (1 + 4L \gamma)\left(1 + \frac{9}{8} L \gamma \right) \expect{\norm{\dev{t}}^2} + \frac{3}{4} \beta^2 L \gamma (1 + 4L \gamma) \expect{\norm{\nabla \avgloss \left( \avgmodel{t} \right)}^2}\\
    &+9\beta^2L^2 \left( 1 + \frac{1}{4\gamma L }\right) \left(\expect{\diam{\model{}{}}{t+1}} + \expect{\diam{\model{}{}}{t}}+\expect{\norm{\avgmodel{t+1}-\avgmodel{t+1/2}}^2}\right) \\
    &+ \frac{9}{4} \beta^2 L \gamma \left(1 + 4L \gamma \right)  L^2\expect{\diam{\model{}{}}{t}} + \frac{(1-\beta)^2 \sigma^2}{n-f}.
\end{align*}
\end{lemma}
}
}

\subsubsection*{\bf Step-IV: Growth function.}  Finally, we analyze the growth of loss function $\avgloss$ computed at the average parameter of the correct nodes $\avgmodel{t}$ along the trajectory of Algorithm~\ref{algo}. Let us denote by $\avgmodel{t} \coloneqq \nicefrac{1}{n-f} \sum_{i \in \H} \model{i}{t}$ the average parameter of the correct nodes at step $t$. Then we obtain the following lemma.  

\noindent \fcolorbox{black}{white}{
\parbox{0.97\textwidth}{\centering
\begin{lemma}
\label{lem:growth}
Suppose that assumptions~\ref{asp:lip} and~\ref{asp:bnd_var} hold true. Consider Algorithm~\ref{algo} with $\gamma \leq 1/L$. For each $t \in [T]$, we obtain that
\begin{align*}
     \expect{\avgloss(\avgmodel{t+1})- \avgloss(\avgmodel{t})} &\leq - \frac{\gamma}{2}\expect{\norm{\nabla \avgloss \left( \avgmodel{t} \right)}^2} + \frac{3\gamma}{2} \expect{\norm{\dev{t}}^2}\\
   &+ \frac{3}{2\gamma} \expect{\norm{\avgmodel{t+\nicefrac{1}{2}}-\avgmodel{t+1}}^2} 
    +   \frac{3\gamma}{2}  L^2\expect{\diam{\model{}{}}{t}}.
\end{align*}
\end{lemma}
}}

This means that Algorithm~\ref{algo} can actually be treated as DSGD with an additional error term which is proportional to the coupled drift of the momentums and the parameters at each step $t$.

\subsubsection*{\bf Combining steps I, II, III and IV} To obtain, our final convergence result, as stated in Theorem~\ref{thm:main_conv}, we combine these elements. Note however that the deviation term in Lemma~\ref{lem:growth} cannot be readily treated with a standard convergence analysis. To address this issue, we devise a new Lyapunov function
\begin{align}
    V_t \coloneqq \expect{\avgloss\left( \avgmodel{t} \right) - Q^* + \frac{1}{4L} \norm{\dev{t}}^2}. \label{eqn:def_lyap}
\end{align}
By analyzing the growth of $V_t$ along the steps of Algorithm~\ref{algo}, we prove Theorem~\ref{thm:main_conv} as follows.

\subsection{Proof for Theorem~\ref{thm:main_conv}}
\label{app:thm_main}

Recall that $\H$ denotes the set of correct nodes, and that $\card{\H} = n-f$. Consider the Lyapunov function $V_t$ defined in~\eqref{eqn:def_lyap}.
Consider an arbitrary $t \in [T]$. From Lemma~\ref{lem:growth} and Lemma~\ref{lem:dev} we obtain that
\begin{align*}
    V_{t+1} - V_t  &=  \expect{\avgloss\left( \avgmodel{t+1} \right) - \avgloss\left( \avgmodel{t} \right)} + \frac{1}{4L} \expect{\norm{\dev{t+1}}^2 - \norm{\dev{t}}^2} \\
    \leq &  - \frac{\gamma}{2}\expect{\norm{\nabla \avgloss \left( \avgmodel{t} \right)}^2} + \frac{3\gamma}{2} \expect{\norm{\dev{t}}^2} + \frac{3}{2\gamma} \expect{\norm{\avgmodel{t+\nicefrac{1}{2}}-\avgmodel{t+1}}^2} 
    +   \frac{3\gamma}{2}  L^2\expect{\diam{\model{}{}}{t}}\\
    &+\frac{1}{4L}\beta^2 (1 + 4L \gamma) \left(1 + \frac{9}{8} L \gamma \right) \expect{\norm{\dev{t}}^2} + \frac{3}{16}  \beta^2  \gamma (1 + 4L \gamma) \expect{\norm{\nabla \avgloss \left( \avgmodel{t} \right)}^2}\\
    &+\frac{9}{4}\beta^2L \left( 1 + \frac{1}{4\gamma L }\right) \left(\expect{\diam{\model{}{}}{t+1}} + \expect{\diam{\model{}{}}{t}}+\expect{\norm{\avgmodel{t+1}-\avgmodel{t+1/2}}^2}\right) \\
    &+ \frac{9}{16} \beta^2  \gamma \left(1 + 4L \gamma \right)  L^2\expect{\diam{\model{}{}}{t}} +\frac{1}{4L} \frac{(1-\beta)^2 \sigma^2}{n-f}  - \frac{1}{4L} \expect{\norm{\dev{t}}^2}.
\end{align*}
Upon re-arranging the terms on the R.H.S.~we obtain that
\begin{align} \label{eqn:before_abc}
    V_{t+1} - V_t  &\leq - \gamma \left(\frac{1}{2} - \frac{3}{16}  \beta^2 (1+4L\gamma) \right) \expect{\norm{\nabla \avgloss \left( \avgmodel{t} \right)}^2}\\
    &+ \left(\frac{3\gamma}{2} + \frac{1}{4L} \beta^2 (1 + 4L \gamma) \left(1 + \frac{9}{8} L \gamma \right) - \frac{1}{4L}\right) \expect{\norm{\dev{t}}^2} \nonumber \\ \nonumber
    & +\frac{9}{4}\beta^2L \left( 1 + \frac{1}{4\gamma L }\right) \left(\expect{\diam{\model{}{}}{t+1}} + \expect{\diam{\model{}{}}{t}}+\expect{\norm{\avgmodel{t+1}-\avgmodel{t+1/2}}^2}\right) \\ \nonumber
    &+ \frac{9}{16} \beta^2  \gamma \left(1 + 4L \gamma \right)  L^2\expect{\diam{\model{}{}}{t}} +\frac{1}{4L} \frac{(1-\beta)^2 \sigma^2}{n-f} \\\nonumber
    &+\frac{3}{2\gamma} \expect{\norm{\avgmodel{t+\nicefrac{1}{2}}-\avgmodel{t+1}}^2} 
    +   \frac{3\gamma}{2}  L^2\expect{\diam{\model{}{}}{t}}.
\end{align}
We denote,
\begin{align*}
    A &\coloneqq \frac{1}{2} - \frac{3}{16} \beta^2  (1+4L\gamma), \\
    B &\coloneqq \frac{3\gamma}{2} + \frac{1}{4L} \beta^2 (1 + 4L \gamma) \left(1 + \frac{9}{8} L \gamma \right) - \frac{1}{4L}, ~ \text{and} \\
    C &\coloneqq \frac{9}{4}\beta^2L \left( 1 + \frac{1}{4\gamma L }\right) \left(\expect{\diam{\model{}{}}{t+1}} + \expect{\diam{\model{}{}}{t}}+\expect{\norm{\avgmodel{t+1}-\avgmodel{t+1/2}}^2}\right) \\ \nonumber
    &+ \frac{9}{16} \beta^2 \gamma \left(1 + 4L \gamma \right)  L^2\expect{\diam{\model{}{}}{t}} +\frac{3}{2\gamma} \expect{\norm{\avgmodel{t+\nicefrac{1}{2}}-\avgmodel{t+1}}^2} 
    +   \frac{3\gamma}{2}  L^2\expect{\diam{\model{}{}}{t}}.
\end{align*}
Substituting from above in~\eqref{eqn:before_abc} we obtain that
\begin{align}
    V_{t+1} - V_t  \leq - \gamma A\expect{\norm{\nabla \avgloss \left( \avgmodel{t} \right)}^2} + B \expect{\norm{\dev{t}}^2} + C + \frac{1}{4L} (1 - \beta)^2 \frac{\sigma^2}{(n-f)}. \label{eqn:after_abc}
\end{align}
Now, we separately analyse the terms $A$, $B$ and $C$ below by using the following,
\begin{align}
    \gamma \leq \frac{1}{12L}, ~ \text{ and } ~ 1 - \beta^2 =  12 \gamma L. 
    \label{eq:conditions}
\end{align}

{\bf Term A.} Using the facts that $\gamma \leq 1/12L$ and that $\beta^2 < 1$, we obtain that
\begin{align}
    A = \frac{1}{2} - \frac{3}{16}  \beta^2 (1+4L\gamma) \geq \frac{1}{2} - \frac{3}{16} \left( 1+\frac{4}{12} \right) = \frac{1}{4}.
    \label{eqn:analyse_A}
\end{align}

{\bf Term B.} As $1 - \beta^2 =  12 \gamma L$ and $\beta^2 < 1$, we obtain that
\begin{align*}
    B = \frac{3\gamma}{2} - \frac{1}{4L} \left(1 - \beta^2 \right) + \frac{1}{4L} \beta^2 \left( \frac{41}{8} L\gamma +\frac{9}{2}L^2\gamma^2  \right) \leq 
    \frac{3\gamma}{2} - 3 \gamma   + \frac{1}{4L}\left( \frac{41}{8} L\gamma +\frac{9}{2}L^2\gamma^2  \right).
\end{align*}
As $\gamma \leq 1/12L$, from above we obtain that
\begin{align}
\label{eqn:analyse_B}
    B \leq \frac{3\gamma}{2} - 3 \gamma   + \frac{\gamma}{4}\left( \frac{41}{8} + \frac{9}{2}L \gamma  \right) \leq - \frac{3\gamma}{2} + \frac{\gamma}{4}\left( \frac{41}{8} + \frac{9}{24}  \right) \leq 0.
\end{align}

{\bf Term C.} Using the fact that $\beta^2 \leq 1$, we obtain that
\begin{align*}
    C  &= \frac{9}{4} \beta^2L \left( 1 + \frac{1}{4\gamma L }\right) \left(\expect{\diam{\model{}{}}{t+1}} + \expect{\diam{\model{}{}}{t}}+\expect{\norm{\avgmodel{t+1}-\avgmodel{t+1/2}}^2}\right) \\ 
    &+ \frac{9}{16} \beta^2  \gamma \left(1 + 4L \gamma \right)  L^2{\expect{\diam{\model{}{}}{t}}} +\frac{3}{2\gamma} \expect{\norm{\avgmodel{t+\nicefrac{1}{2}}-\avgmodel{t+1}}^2} 
    +   \frac{3\gamma}{2}  L^2\expect{\diam{\model{}{}}{t}}\\
    &\leq \expect{\diam{\model{}{}}{t}}  \left(\frac{9L}{16\gamma L}(1+4\gamma L) + \frac{9}{16} \gamma L^2 \left(1 + 4L \gamma \right) + \frac{3\gamma}{2}  L^2  \right) \\
    &+ \expect{\diam{\model{}{}}{t+1}} \frac{9L}{16\gamma L}(1+4\gamma L) +\left(\frac{3}{2\gamma}+\frac{9L}{16\gamma L}(1+4\gamma L) \right) \expect{\norm{\avgmodel{t+\nicefrac{1}{2}}-\avgmodel{t+1}}^2}.
\end{align*}
Using the fact $\gamma \leq 1/12L$ we then have
\begin{align}\nonumber
    C &\leq \expect{\diam{\model{}{}}{t}}  \left(\frac{9}{16\gamma } \left( \frac{4}{3} \right) + \frac{9}{16} \left( \frac{1}{144 \gamma} \right) \left( \frac{4}{3} \right) + \frac{3}{288 \gamma}  \right) \\\nonumber
    &+ \expect{\diam{\model{}{}}{t+1}} \frac{9}{16\gamma } \left( \frac{4}{3} \right) +\left(\frac{3}{2\gamma}+\frac{9}{16\gamma } \left( \frac{4}{3} \right) \right) \expect{\norm{\avgmodel{t+\nicefrac{1}{2}}-\avgmodel{t+1}}^2}\\
    &\leq \frac{1}{\gamma} \expect{\diam{\model{}{}}{t}} + \frac{1}{\gamma} \expect{\diam{\model{}{}}{t+1}}
    +\frac{9}{4\gamma} \expect{\norm{\avgmodel{t+\nicefrac{1}{2}}-\avgmodel{t+1}}^2}.
    \label{eq:c_bound_temp}
\end{align}
From Lemma \ref{lemma:cva_bounds}, we have
\begin{align*}
    \expect{\norm{\avgmodel{t + \nicefrac{1}{2}}-\avgmodel{t+1}}^2} \leq \lambda \expect{\diam{\model{}{}}{t+\nicefrac{1}{2}}}.
\end{align*}
From Algorithm \ref{algo}, we have for all $i \in \H$, $\model{i}{t+\nicefrac{1}{2}} = \model{i}{t} - \gamma \mmt{i}{t}$. Therefore, by definition of $\Gamma(\cdot)$, $\diam{\model{}{}}{t+\nicefrac{1}{2}} \leq 2 \diam{\model{}{}}{t} + 2 \gamma^2 \diam{\mmt{}{}}{t}$. Thus, from above we obtain that
\begin{align*}
    \expect{\norm{\avgmodel{t+\nicefrac{1}{2}}-\avgmodel{t+1}}^2} \leq \lambda \left(2\expect{\diam{\model{}{}}{t}} + 2 \gamma^2 \expect{\diam{\mmt{}{}}{t}} \right)
\end{align*}
Substituting from above in~\eqref{eq:c_bound_temp} we obtain that
\begin{align*}
    C \leq  \frac{1}{\gamma}\left( 1+\frac{9\lambda}{2} \right) \expect{\diam{\model{}{}}{t}} + \frac{1}{\gamma} \expect{\diam{\model{}{}}{t+1}} 
    +\frac{9\lambda \gamma}{2}  \expect{\diam{\mmt{}{}}{t}}.
\end{align*}
By invoking Lemma \ref{lem:drift}, we obtain from above that
\begin{align*}
    C &\leq \left( 2+\frac{9\lambda}{2} \right) \coeffalpha \gamma \left( \sigma^2 \left( \frac{1 - \beta}{1 + \beta} \right) + 3 \heter^2 \right) \\
    &+ \frac{9\lambda \gamma}{2}  \left(3 \sigma^2 \left(\frac{1 - \beta}{1 + \beta} \right) + 9 \heter^2 + 9 L^2 \gamma^2 \coeffalpha  \left( \sigma^2 \left( \frac{1 - \beta}{1 + \beta} \right) + 3 \heter^2 \right) \right).
\end{align*}
Upon re-arranging the terms, and using the facts that $\gamma \leq 1/12L$, we obtain that
\begin{align}
    C &\leq \gamma  \heter^2 \left(6\coeffalpha + \frac{9\lambda}{2}\left( 3\coeffalpha +9+\frac{3\coeffalpha }{16}\right)\right) +\gamma \sigma^2 \left( \frac{1 - \beta}{1 + \beta} \right) \left( 2\coeffalpha + \frac{9\lambda}{2}\left( \coeffalpha +3+\frac{ \coeffalpha }{16}\right) \right) \nonumber \\
    &\leq \gamma \heter^2 \left(6\coeffalpha + \frac{9\lambda}{2}\left( 4\coeffalpha + 9\right)\right) + \gamma \sigma^2 \left( \frac{1 - \beta}{1 + \beta} \right) \left( 2\coeffalpha + \frac{9\lambda}{2}\left(2\coeffalpha +3\right) \right).
    \label{eq:c_bound_0}
\end{align}
Note that 
\begin{align*}
    \frac{1-\beta}{1+\beta} = \frac{1-\beta^2}{(1+\beta)^2} \leq 1 - \beta^2 = 12 \gamma L.
\end{align*}
Substituting from above in \eqref{eq:c_bound_0} we obtain that
\begin{align}
    C \leq  \gamma \heter^2 \left(6 \coeffalpha + \frac{9\lambda}{2}\left( 4\coeffalpha + 9\right)\right)
    +12 \gamma^2 \sigma^2L\left( 2\coeffalpha + \frac{9\lambda}{2}\left(2\coeffalpha +3\right) \right).
    \label{eq:c_bound}
\end{align}

{\bf Combining A, B and C.} Substituting from~\eqref{eqn:analyse_A} and ~\eqref{eqn:analyse_B} in~\eqref{eqn:after_abc},  we obtain that
\begin{align*}
     V_{t+1} - V_t  &\leq - \gamma A \expect{\norm{\nabla \avgloss \left( \avgmodel{t} \right)}^2} + B \expect{\norm{\dev{t}}^2} + C + \frac{1}{4L} (1 - \beta)^2 \frac{\sigma^2}{(n-f)} \\
    & \leq - \frac{\gamma}{4} \expect{\norm{\nabla \avgloss \left( \avgmodel{t} \right)}^2}+ C + (1 - \beta)^2 \frac{\sigma^2}{4 L (n-f)}.
\end{align*}
Note that, as $\beta \in (0, \, 1)$,  $1 - \beta = ( 1 - \beta^2)/(1 + \beta) \leq 1 - \beta^2$. Using this above we obtain that
\begin{align*}
     V_{t+1} - V_t  & \leq - \frac{\gamma}{4} \expect{\norm{\nabla \avgloss \left( \avgmodel{t} \right)}^2} + C + (1 - \beta^2)^2 \frac{\sigma^2}{4 L (n-f)}.
\end{align*}
Recall that $1 - \beta^2 = 12 \gamma L$. Therefore,
\begin{align*}
    V_{t+1} - V_t \leq - \frac{\gamma}{4} \expect{\norm{\nabla \avgloss \left( \avgmodel{t} \right)}^2}+ C + 36 \gamma^2 L \frac{\sigma^2}{n-f}.
\end{align*}
This implies that
\begin{align}
   \expect{\norm{\nabla \avgloss \left( \avgmodel{t} \right)}^2} \leq \left(V_{t} - V_{t+1} \right) \frac{4}{\gamma}+ \frac{4}{\gamma}C + 144 \gamma L \frac{\sigma^2}{n-f}. \label{eqn:before_sum}
\end{align}

By taking the average on both sides from $t = 0$ to $T-1$, we obtain that
\begin{align}
   \frac{1}{T}\sum_{t=0}^{T-1}\expect{\norm{\nabla \avgloss \left( \avgmodel{t} \right)}^2} \leq \left(V_{0} - V_{T} \right) \frac{4}{\gamma T}+ \frac{4}{\gamma}C + 144 \gamma L \frac{\sigma^2}{n-f}. 
   \label{eq:total_decrease}
\end{align}

{\bf Analysis on $V_t$.} Recall that $Q^* = \inf_{\model{}{}} \avgloss(\model{}{})$. Note that for any $t$, 
\[V_t= \expect{\avgloss(\avgmodel{t}) - Q^*} + \frac{1}{4L} \expect{\norm{\dev{t}}^2} \geq \expect{\avgloss(\avgmodel{t}) - Q^*} \geq 0.\] 
Thus, $V_{T} \geq 0$. Using this in~\eqref{eq:total_decrease} we obtain that
\begin{align}
   \frac{1}{T}\sum_{t=0}^{T-1}\expect{\norm{\nabla \avgloss \left( \avgmodel{t} \right)}^2} \leq \left(V_{0} \right) \frac{4}{\gamma T}+ \frac{4}{\gamma}C + 144 \gamma L \frac{\sigma^2}{n-f}. 
   \label{eq:before_V1}
\end{align}
Recall that 
$$V_0 = \expect{\avgloss\left(\avgmodel{0}\right)-Q^* + \frac{1}{4L} \norm{\dev{0}}^2}.$$ 
Recall that, by Definition~\eqref{eq:deviation_definition} of $\dev{t}$, we have $\dev{0} = \AvgMmt{0} - \AvgGrad{0}$. Thus, under Assumption~\ref{asp:bnd_var}, we have
\begin{align*}
    \expect{\norm{\dev{0}}^2} &= \expect{\norm{ (1 - \beta) \AvgNoisyGrad{0} - \AvgGrad{0}}^2} \leq 2 (1 - \beta)^2 \expect{\norm{ \AvgNoisyGrad{0} - \AvgGrad{0}}^2} + 2 \beta^2 \norm{\AvgGrad{0}}^2\\
    &\leq 2 (1 - \beta)^2 \left(\frac{\sigma^2}{n-f} \right) + 2 \beta^2 \norm{\AvgGrad{0}}^2.
\end{align*}
{Recall, from Algorithm~\ref{algo}, that for each correct node $i$, the initial model $\model{i}{0}$ is identical, denoted by $\theta_0$. Therefore, we have $\AvgGrad{0} = \nabla \avgloss \left( \avgmodel{0} \right) $. Substituting this in the above we obtain that
\begin{align}
    \expect{\norm{\dev{0}}^2} \leq 2 (1 - \beta)^2 \left(\frac{\sigma^2}{n-f} \right) + 2 \beta^2 \norm{\nabla \avgloss \left( \avgmodel{0} \right)}^2 . \label{eqn:after_same_init_model}
\end{align}}
Recall that $ 1 - \beta^2 = 12 \gamma L$. Thus, $(1 - \beta)^2 \leq (1 - \beta^2)^2/(1 + \beta)^2 \leq (1 - \beta^2)^2 = 144 \gamma^2 L^2$. Substituting this in~\eqref{eqn:after_same_init_model}, and using the fact that $\beta^2 < 1$, we obtain that
\begin{align*}
    \expect{\norm{\dev{0}}^2} \leq 288 \gamma^2 L^2 \left(\frac{\sigma^2}{n-f} \right) +  2 \norm{\nabla \avgloss \left( \avgmodel{0} \right)}^2.
\end{align*}
Therefore, 
\begin{align*}
    V_0 &\leq \avgloss\left(\avgmodel{0}\right)-Q^* + \frac{1}{4L} \left( 288 \gamma^2 L^2 \left(\frac{\sigma^2}{n-f} \right) + 2 \norm{\nabla \avgloss \left( \avgmodel{0} \right)}^2 \right) \\ &=   \avgloss\left(\avgmodel{0}\right) - Q^* +  72 \gamma^2 L \left(\frac{\sigma^2}{n-f} \right) + \frac{1}{2L} \norm{\nabla \avgloss \left( \avgmodel{0} \right)}^2 .
\end{align*}
{Note that, as $\avgloss$ is $L$-smooth (see~Remark~\ref{remark:initial_gradient}), $\norm{\nabla \avgloss \left( \avgmodel{0} \right)}^2 \leq 2L \left( \avgloss\left(\avgmodel{0}\right) - Q^* \right)$.} Using this in the above yields
\begin{align}
\label{eq:V_zero_bound}
    V_0 \leq 2(\avgloss\left(\avgmodel{0}\right) - Q^*) +  72 \gamma^2 L \left(\frac{\sigma^2}{n-f} \right).
\end{align}
where in the last inequality we used Remark~\ref{remark:initial_gradient} below.
Substituting from above in~\eqref{eq:before_V1} we obtain that
\begin{align}
    & \frac{1}{T}\sum_{t=0}^{T-1}\expect{\norm{\nabla \avgloss \left( \avgmodel{t} \right)}^2} \leq \frac{8\left(\avgloss\left(\avgmodel{0}\right) - Q^*\right)}{\gamma T} + \frac{288 \gamma L}{T} \left(\frac{\sigma^2}{n-f} \right)  + \frac{4}{\gamma}C + 144 \gamma L \frac{\sigma^2}{n-f}. \label{eqn:after_V2}
\end{align}

Now, note that for any correct node $i \in \H$, we have
\begin{align*}
    \expect{\norm{\nabla \avgloss \left( \model{i}{t} \right)}^2} &\leq \frac{3}{2}\expect{\norm{\nabla \avgloss \left( \avgmodel{t} \right)}^2} +  3\expect{\norm{\nabla \avgloss \left( \avgmodel{t} \right)-\nabla \avgloss \left( \model{i}{t} \right)}^2}\\
    &\leq \frac{3}{2}\expect{\norm{\nabla \avgloss \left( \avgmodel{t} \right)}^2} +  3L^2\expect{\norm{ \avgmodel{t} - \model{i}{t}}^2}\\
     &\leq \frac{3}{2}\expect{\norm{\nabla \avgloss \left( \avgmodel{t} \right)}^2} +  3L^2\sum_{j \in \H}\expect{\norm{ \avgmodel{t} - \model{j}{t}}^2}\\
     &\leq \frac{3}{2}\expect{\norm{\nabla \avgloss \left( \avgmodel{t} \right)}^2} +  3L^2 n \expect{\diam{\model{}{t}}{}}.
\end{align*}
where the second inequality follows from Assumption~\ref{asp:lip}.
Combining this with~\eqref{eqn:after_V2}, we obtain that
\begin{align*}
    & \frac{1}{T}\sum_{t=0}^{T-1} \expect{\norm{\nabla \avgloss \left( \model{i}{t} \right)}^2} \leq \frac{12\left(\avgloss\left(\avgmodel{0}\right) - Q^*\right)}{\gamma T} + \frac{432 \gamma L}{T} \left(\frac{\sigma^2}{n-f} \right)  + \frac{6}{\gamma}C + 216 \gamma L \frac{\sigma^2}{n-f}+3L^2 n \expect{\diam{\model{}{t}}{}}. 
\end{align*}

Substituting $C$ from \eqref{eq:c_bound} in above and using the bound on $\expect{\diam{\model{}{t}}{}}$ from Lemma~\ref{lem:drift},  we obtain that
\begin{align*}
     \frac{1}{T}\sum_{t=0}^{T-1} \expect{\norm{\nabla \avgloss \left( \model{i}{t} \right)}^2} &\leq \frac{12\left(\avgloss\left(\avgmodel{0}\right) - Q^*\right)}{\gamma T} + \frac{432 \gamma L}{T} \left(\frac{\sigma^2}{n-f} \right) + 216 \gamma L \frac{\sigma^2}{n-f}\\
    &+ 6\heter^2 \left(6 \coeffalpha + \frac{9\lambda}{2}\left( 4\coeffalpha + 9\right)\right) + 72 \gamma \sigma^2L\left( 2\coeffalpha + \frac{9\lambda}{2}\left(2\coeffalpha +3\right) \right)\\
    &+3L^2n\coeffalpha\gamma^2 \left(2 \sigma^2 \left( \frac{1 - \beta}{1 + \beta} \right) + 3 \heter^2 \right)\\
    &\leq \frac{12\left(\avgloss\left(\avgmodel{0}\right) - Q^*\right)}{\gamma T} + \frac{432 \gamma L}{T} \left(\frac{\sigma^2}{n-f} \right) \\
    &+ 72 \gamma L \sigma^2 \left (\frac{3}{n-f} + 2 \coeffalpha + \frac{9\lambda}{2}\left(2\coeffalpha +3\right) \right)\\
    &+ 6 \heter^2 \left(6 \coeffalpha + \frac{9\lambda}{2}\left( 4\coeffalpha + 9\right)\right)\\
    &+3L^2n\coeffalpha\gamma^2 \left(2 \sigma^2  + 3 \heter^2 \right)
\end{align*}

We now define
\begin{equation*}
    c_0 := 12\left(\avgloss\left(\avgmodel{0}\right) - Q^*\right), \quad c_1 := \coeffalpha =  \frac{18 \alpha (1 + \alpha)}{(1 - \alpha)^2}, 
\end{equation*}
\begin{equation*}
c_2 := 72  L  \left (\frac{3}{n-f} + 2 c_1 + \frac{9\lambda}{2}\left(2c_1 +3\right) \right) \text{, and } c_3:=6 \left(6 c_1 + \frac{9\lambda}{2}\left( 4c_1 + 9\right)\right).
\end{equation*}

Then
\begin{align*}
    \frac{1}{T}\sum_{t=0}^{T-1} \expect{\norm{\nabla \avgloss \left( \model{i}{t} \right)}^2}  &\leq \frac{c_0}{\gamma T} + c_2 \gamma L \sigma^2  + \frac{432 \gamma L}{T} \left(\frac{\sigma^2}{n-f} \right) 
    + c_3 \heter^2 + 9c_1n\gamma^2L^2(\sigma^2 + \heter^2).
\end{align*}

Now recall that 
\begin{equation*}
    \gamma = \min \left\{\frac{1}{12L}, \frac{1}{L} \sqrt{\frac{2}{3c_1}}, \sqrt{\frac{c_0}{c_2LT\sigma^2}} \right\},
\end{equation*}
and thus
\begin{align*}
    \frac{1}{\gamma} = \max \left\{{12L},  L\sqrt{\frac{3c_1}{2}}, \sqrt{\frac{c_2LT\sigma^2}{c_0}} \right\} \leq {12L}+L\sqrt{\frac{3c_1}{2}}+\sqrt{\frac{c_2LT\sigma^2}{c_0}}.
\end{align*}
Therefore,
\begin{align*}
    \frac{1}{T}\sum_{t=0}^{T-1} \expect{\norm{\nabla \avgloss \left( \model{i}{t} \right)}^2}  &\leq 2\sqrt{\frac{c_0c_2L\sigma^2}{T}} + \frac{12Lc_0}{T}+\frac{Lc_0}{T}\sqrt{\frac{3c_1}{2}}  + \frac{36}{T} \left(\frac{\sigma^2}{n-f} \right) 
    + c_3 \heter^2 +  9c_1nL( 1+ \frac{\heter^2}{\sigma^2}) \frac{c_0}{c_2T}
\end{align*}
Denoting $c_4 := 9c_1nc_0/c_2$, we have
\begin{align}\label{eq:bound_on_local_model}
    \frac{1}{T}\sum_{t=0}^{T-1} \expect{\norm{\nabla \avgloss \left( \model{i}{t} \right)}^2}  &\leq 2\sqrt{\frac{c_0c_2L\sigma^2}{T}} + \frac{12Lc_0}{T}+\frac{Lc_0}{T}\sqrt{\frac{3c_1}{2}}  + \frac{36}{T} \left(\frac{\sigma^2}{n-f} \right) 
     +  ( 1+ \frac{\heter^2}{\sigma^2}) \frac{c_4L}{T} + c_3 \heter^2.
\end{align}
As $ \returnedmodel{\honestnode} \sim \mathcal{U}\{ \model{i}{0}, \dots, \model{i}{T-1} \}$, we get
\begin{align*}
    \expect{\norm{\nabla \avgloss \left( \returnedmodel{\honestnode} \right)}^2} = \frac{1}{T}\sum_{t=0}^{T-1}\expect{\norm{\nabla \avgloss \left( \model{i}{t} \right)}^2}.
\end{align*}
Substituting from above in~\eqref{eq:bound_on_local_model} concludes the proof.

\begin{remark}
\label{remark:initial_gradient}
If a function $Q$ is Lipschitz smooth, with coefficient $L$, then for all $x \in \mathbb{R}^d$, $\norm{\nabla Q(x) }^2 \leq 2L \left( Q(x)- Q^* \right)$ where $Q^*$ denotes the minimum value of $Q$. Proof of this fact is as follows.
\end{remark}

\begin{proof}
By the Lipschitzness of $\nabla Q$, for any $x,y \in \mathbb{R}^d$, we have (see~Lemma~1.2.3~\cite{nesterov2018lectures})
\begin{align*}
    Q(y) \leq Q(x) + \langle \nabla Q(x), y-x\rangle+\frac{L}{2} \Vert y- x \Vert^2.
\end{align*}
Let $x$ be an arbitrary vector in $\R^d$, and $y = x - \frac{1}{L} \nabla Q(x)$. Thus, from the above we obtain that
\begin{align*}
    Q\left(x - \frac{1}{L} \nabla Q(x)\right) \leq Q(x)   -  \frac{1}{L}  \Vert\nabla Q(x) \Vert^2  +  \frac{1}{2L}  \Vert\nabla Q(x) \Vert^2 = Q(x) -  \frac{1}{2L}  \Vert\nabla Q(x) \Vert^2.
\end{align*}
As $Q^*$ is the minimum value of $Q$, we have
\begin{align*}
    Q^* \leq Q\left(x - \frac{1}{L} \nabla Q(x)\right) \leq Q(x) -  \frac{1}{2L}  \Vert\nabla Q(x) \Vert^2.
\end{align*}
Rearranging the terms we obtain that
\begin{align*}
     \Vert\nabla Q(x) \Vert^2 \leq 2L (Q(x)- Q^*).
\end{align*}
Recall that $x$ in the above can be any vector in $\R^d$. The above completes the proof.
\end{proof} 

\subsection{Proof of Corollary \ref{cor:maincor}}
\label{sec:coro_proof}
\begin{proof}
Note that ignoring the higher order terms in the bound of Theorem~\ref{thm:main_conv}, we have 
\begin{equation*}
    \frac{1}{T}\sum_{t=1}^T\expect{\norm{\nabla \avgloss \left( \model{i}{t} \right)}^2} \in \mathcal{O}\left( \sqrt{\frac{c_0 c_2 \sigma^2}{T}} +  c_3 \heter^2  \right).
\end{equation*}
Now note also that in Theorem~\ref{thm:main_conv} for $n\geq11f$, we have $\alpha \leq 0.988 < 1$. This implies that $\frac{ 1 + \alpha}{(1 - \alpha)^2} \in \mathcal{O}(1)$. Therefore,
\begin{align*}
    c_1 = \frac{18 \alpha (1 + \alpha)}{(1 - \alpha)^2} \in \mathcal{O}(\alpha).
\end{align*}
Next, by noting that $n\geq 2f$, we obtain that
\begin{align*}
    c_2 = 72  L  \left (\frac{3}{n-f} + 3 c_1 + \frac{9\lambda}{2}\left(2c_1 +3\right) \right) \in \mathcal{O}\left(\frac{1}{n} + \alpha + \lambda \right),
\end{align*}
and 
\begin{align*}
     c_3 =7 \left(6 c_1 + \frac{9\lambda}{2}\left( 4c_1 + 9\right)\right) \in \mathcal{O}\left(\alpha + \lambda \right).
\end{align*}
Finally, note that $c_0$ is a constant depending on the initial model and thus $c_0 \in \mathcal{O}(1)$. Therefore,
\begin{align*}
    \frac{1}{T}\sum_{t=1}^T\expect{\norm{\nabla \avgloss \left( \model{i}{t} \right)}^2} &\in \mathcal{O}\left(\sqrt{\frac{\sigma^2}{T}\left( \frac{1}{n}+\alpha + \lambda \right)} + (\alpha + \lambda)\heter^2\right).
\end{align*} 
Now note that, we have $\alpha = \frac{9.88f}{n-f} \in \mathcal{O}(\frac{f}{n})$, and $\lambda = \frac{9f}{n-f}\in \mathcal{O}(\frac{f}{n})$. Therefore, 
\begin{equation*}
    \frac{1}{T}\sum_{t=1}^T\expect{\norm{\nabla \avgloss \left( \model{i}{t} \right)}^2} \in \mathcal{O}\left(\sqrt{\frac{\sigma^2}{T}\left( \frac{1+f}{n} \right)} + \frac{f}{n} \heter^2\right).
\end{equation*}
This completes the proof.
\end{proof}

\subsection{Convergence of \newalgorithm{} for $n > 5f$}\label{corr_5f}
\label{sec:app:fivef}
Note that Theorem~\ref{thm:main_conv}  is stated for the case where $n\geq11f$. This comes from the fact that we need the number of correct nodes to be sufficiently large to guarantee  $(\alpha,\lambda)$-\reduction{} as stated in Lemma~\ref{lemma:cva_bounds}. However, by setting $\numbercvarounds \in \mathcal{O}(\log(n))$ we can still guarantee  $(\alpha,\lambda)$-\reduction{} for $n>5f$ as stated in the following Lemma. The proof of this Lemma is given in Section~\ref{sec:proofoffivef}.

\noindent \fcolorbox{black}{white}{
\parbox{0.97\textwidth}{\centering
\begin{lemma}
\label{lemma:cvaboundfivef}
Suppose that there exists $\cvabyzbound>0$ such that $n \geq (5+\cvabyzbound)f$. For  $\numbercvarounds = \frac{\log(8(n-f))}{2\log(\frac{3+\cvabyzbound}{\cvabyzbound})} \in \mathcal{O}(\log(n))$, the coordination phase of Algorithm~\ref{algo} guarantees $(\alpha,\lambda)$-\reduction{} for 
  \begin{align*}
     \alpha = \frac{2f}{n-f} \leq \frac{1}{2} \quad\text{and}\quad\lambda = \left( \frac{3+\cvabyzbound}{\cvabyzbound} \right)^2 \frac{(8f)^2}{n-f}.
  \end{align*}
\end{lemma}
}}

Replacing Lemma~\ref{lemma:cva_bounds} by Lemma \ref{lemma:cvaboundfivef}, and following the same steps as the proof of Theorem~\ref{thm:main_conv} we can show the following result which is essentially a convergence proof for \newalgorithm{} while tolerating a larger fraction of \byzantine{} nodes ($n>5f$).

\noindent \fcolorbox{black}{white}{
\parbox{0.97\textwidth}{\centering
\begin{corollary}
\label{cor:fivef}
Suppose that assumptions~\ref{asp:lip},~\ref{asp:bnd_var} and~\ref{asp:heter} hold true. Suppose also that there exists $\cvabyzbound > 0$ such that $n \geq (5 + \cvabyzbound )f$. Denote

\begin{gather*}
\alpha =  \frac{2f}{n-f}, \quad \lambda =  \left( \frac{3+\cvabyzbound}{\cvabyzbound} \right)^2 \frac{(8f)^2}{n-f} \quad
c_0 := 12\left(\avgloss\left(\avgmodel{0}\right) - Q^*\right),\\ c_1 :=   \frac{18 \alpha (1 + \alpha)}{(1 - \alpha)^2}, \quad
c_2 := 72  L  \left (\frac{3}{n-f} + 3 c_1 + \frac{9\lambda}{2}\left(2c_1 +3\right) \right),\text{and } c_3:=7 \left(6 c_1 + \frac{9\lambda}{2}\left( 4c_1 + 9\right)\right).
\end{gather*}

Consider Algorithm~\ref{algo} with $K = \frac{\log(8(n-f))}{2\log(\frac{3+\cvabyzbound}{\cvabyzbound})} \in \mathcal{O}(\log(n))$, $\gamma = \min \left\{\frac{1}{12L}, \frac{1}{L} \sqrt{\frac{2}{3c_1}}, \sqrt{\frac{c_0}{c_2LT\sigma^2}} \right\},$ and $\beta = \sqrt{ 1 - 12 \gamma L}$. Then, for all $T\geq1$, we obtain that 
\begin{align*}
    &\expect{\norm{\nabla \avgloss \left( \returnedmodel{\honestnode} \right)}^2} \leq 2\sqrt{\frac{c_0c_2L\sigma^2}{T}} 
    + \frac{12Lc_0}{T}+\frac{Lc_0}{T}\sqrt{\frac{3c_1}{2}}  + \frac{36}{T} \left(\frac{\sigma^2}{n-f} \right) 
    + c_3 \heter^2.
\end{align*}

\end{corollary}
}}

\begin{remark}
Ignoring higher order terms and following the same reasoning as that of proof of Corollary~\ref{cor:maincor}, we can show that for $n>5f$, \newalgorithm{} guarantees $(f,\epsilon)$-resilience for 
\begin{align*}
\epsilon \in \mathcal{O}\left(\sqrt{\frac{(1+f)^2}{n T} } + \frac{f^2}{n} \heter^2 \right). 
\end{align*}
\end{remark}

\subsection{Proof of Lemma \ref{lemma:cva_bounds}}
\label{app:reductionNNA}

Throughout the proof, we make use of the following notation.

{\bf Notation:} Recall that $\receive{i}{k}$ is the set of indices received by node $i$ at coordination round $k$. 
Let $\psi_k^{(i)}: [n-f-1]\rightarrow \receive{i}{k}$ be a bijection that sorts the elements in $\receive{i}{k}$ based on the distance of their corresponding vector to $\cvavector{i}{k-1}$, i.e., $$\norm{\cvavector{\psi_k^{(i)}(1)}{k-1}-\cvavector{i}{k-1}} \leq \ldots \leq \norm{\cvavector{\psi_k^{(i)}(n-f-1)}{k-1}-\cvavector{i}{k-1}}.$$
We then denote by 
\begin{equation}
\label{eq:filter}
    \filter{i}{k}:= \left\{ \psi_k^{(i)}(j): j \in [n-2f-1] \right\}\cup \{i\},
\end{equation}
the set of indices of the vectors selected by the \cva{} function. From \eqref{eq:nna}, we then have 
\begin{align}
    \cvavector{i}{k} = \textsc{NNA}\left( \cvavector{i}{k-1}   ; \left \{ \cvavector{j}{k-1} \mid j \in \receive{\honestnode}{\cvaround} \right \} \right) = \frac{1}{\nodenumber - 2\byzantinebound} \sum_{j\in \filter{i}{k}} \cvavector{j}{k-1}. \label{eqn:app_def_nna}
\end{align}

We first prove the following few useful lemmas.

\begin{lemma}
\label{lem:diam_equal}
    For any set $\{x^{(i)}\}_{i \in S}$ of $\card{S}$ vectors, we have $$\diam{x}{} = \frac{1}{\card{S}} \sum_{i\in S} \norm{ x^{(i)}- \bar{x}}^2 = \frac{1}{2} \cdot \frac{1}{\card{S}^2} \sum_{i,j \in S} \norm{ x^{(i)}- x^{(j)}}^2 .$$
\end{lemma}
\begin{proof}
\begin{align*}
     \frac{1}{\card{S}^2} \sum_{i,j \in S} \norm{ x^{(i)}- x^{(j)}}^2 &= \frac{1}{\card{S}^2} \sum_{i,j \in S} \norm{ (x^{(i)} -\bar{x}) - (x^{(j)} - \bar{x})}^2\\ &= \frac{1}{\card{S}^2}\sum_{i,j \in S}\left[ \norm{x^{(i)} -\bar{x} }^2 + \norm{x^{(j)} -\bar{x} }^2 + 2 \iprod{x^{(i)} -\bar{x}}{x^{(j)} -\bar{x}}\right] \\
     &= \frac{2}{\card{S}} \sum_{i,j \in S} \norm{x^{(i)} -\bar{x} }^2  + \frac{2}{\card{S}^2}  \sum_{i \in S} \iprod{x^{(i)} - \bar{x}}{ \sum_{j \in S} (x^{(j)} -\bar{x})}.
\end{align*}

Now as $\sum_{j \in S} (x^{(j)} -\bar{x}) = 0$, we have
\begin{align*}
    \frac{1}{\card{S}^2} \sum_{i,j \in S} \norm{ x^{(i)}- x^{(j)}}^2 = \frac{2}{\card{S}} \sum_{i,j \in S} \norm{x^{(i)} -\bar{x} }^2.
\end{align*}
\end{proof}

\begin{lemma}
\label{lemma:set_size}
  For any pair of correct nodes $p, \, q$ and coordination round $k \in [\numbercvarounds]$ we obtain that
  \begin{align}
      \card{\filter{q}{k} \setminus \filter{p}{k}} = \card{\filter{p}{k} \setminus \filter{q}{k}} \leq 2f. 
  \end{align}
\end{lemma}
\begin{proof}
Consider an arbitrary pair of correct nodes $p, \, q$, and coordination round $k$. By definition of set $\filter{i}{k}$ for all $i \in \H$ in~\eqref{eq:filter} we obtain that 
  \begin{align*}
      \card{\filter{q}{k} \setminus \filter{p}{k}} = \card{\filter{q}{k} \cup \filter{p}{k}} - \card{\filter{p}{k}} \leq n - (n-2f) = 2f.
  \end{align*}
\end{proof}

\begin{lemma}
\label{lemma:cva_diameter}
  If $n \geq 11f$ then for each coordination round $k \in [\numbercvarounds]$ we obtain that
  \begin{align*}
    \diam{\cvavector{}{k}}{} \leq \alpha \diam{\cvavector{}{k-1}}{} \quad \textnormal{for} \quad \alpha = \frac{9.88 f}{n-f}.
\end{align*}
\end{lemma}
\begin{proof}
Consider two arbitrary correct nodes $p$ and $q$ in $\H$, and an arbitrary $k \in [\numbercvarounds]$. We first introduce some sets that will be used later in the proof. We denote by \textbf{$F_p$} the set of \byzantine{} nodes whose local parameters are selected by $p$ (using the NNA rule) but not by $q$ in the $k$-th coordination round, i.e., 
\begin{equation*}
    F_p \coloneqq \left\{i \in [n] \setminus \H ~ \vline ~ i \in \filter{p}{k} \setminus \filter{q}{k} \right\}.
\end{equation*}
Similarly, $F_q \coloneqq \left\{i \in [n] \setminus \H ~ \vline ~ i \in \filter{q}{k} \setminus \filter{p}{k} \right\}$.
Recall that by Lemma \ref{lemma:set_size} we have $\card{\filter{p}{k} \setminus \filter{q}{k}} \leq 2f$. We consider an arbitrary subset $\H_p$ comprising correct nodes selected by node $p$ in round $k$ such that $\card{\H_p} + \card{F_p} = 2f$ and $\filter{p}{k} \setminus \filter{q}{k} \subseteq \H_p$, i.e., 
\begin{equation*}
    \H_p:= \left\{i \in \H \cap \filter{p}{k} ~ \vline ~ \card{\H_p} + \card{F_p} = 2f, \, \filter{p}{k} \setminus \filter{q}{k} \subseteq \H_p \right\}.
\end{equation*}
Similarly, $\H_q:= \left\{i \in \H \cap \filter{q}{k} ~ \vline ~ \card{\H_q} + \card{F_q} = 2f, \, \filter{q}{k} \setminus \filter{p}{k} \subseteq \H_q \right\}$.
We let $f_p := \card{F_p}$ and $f_q := \card{F_q}$. Note that $f_p + f_q \leq f$. 
We sort the nodes in $\H_p$ based on the distance of their vectors to $\cvavector{q}{k-1}$ (with ties broken arbitrarily). Let $\H_p[i]$ denote the $i$-th element in $\H_p$ after the sorting. Thus, we have $\norm{\cvavector{q}{k-1}-\cvavector{\H_p[i]}{k-1}}\leq \norm{\cvavector{q}{k-1}-\cvavector{\H_p[i+1]}{k-1}}$. We do the similar operation on $\H_q$. 

By definition of NNA~\eqref{eqn:app_def_nna}, we obtain that
\begin{align*}
    \norm{\cvavector{p}{k}-\cvavector{q}{k}} &= \norm{\frac{1}{n-2f} \sum_{j \in \filter{p}{k}} \cvavector{j}{k-1}-\frac{1}{n-2f} \sum_{j \in \filter{q}{k}} \cvavector{j}{k-1}} \\ 
    & = \frac{1}{n-2f} \norm{ \sum_{j \in F_p} \cvavector{j}{k-1} +\sum_{j \in \H_p} \cvavector{j}{k-1} - \sum_{j \in F_q} \cvavector{j}{k-1} -\sum_{j \in \H_q} \cvavector{j}{k-1}}\\
    &= \frac{1}{n-2f} \left\lVert \left(\sum_{j \in F_p} \cvavector{j}{k-1} - \sum_{j \in [f_p]} \cvavector{\H_q[j]}{k-1}\right) + \left( \sum_{j \in [f+1,2f-f_p]} \cvavector{\H_p[j]}{k-1} - \sum_{j \in [f_p+1,f]} \cvavector{\H_q[j]}{k-1} \right) \right. \\
    &- \left.  \left (\sum_{j \in F_q} \cvavector{j}{k-1} - \sum_{j \in [f_q]} \cvavector{\H_p[j]}{k-1} \right) - \left( \sum_{j \in [f+1,2f-f_q]} \cvavector{\H_q[j]}{k-1} - \sum_{j \in [f_q+1,f]} \cvavector{\H_p[j]}{k-1} \right) \right\lVert.
\end{align*}
Therefore,
\begin{align*}
    &\norm{\cvavector{p}{k}-\cvavector{q}{k}} = \frac{1}{n-2f} \left\lVert \left(\sum_{j \in F_p} (\cvavector{j}{k-1} - \cvavector{p}{k-1}) -\sum_{j \in [f_p]} (\cvavector{\H_q[j]}{k-1}- \cvavector{p}{k-1})\right) \right. \\
    &+ \left( \sum_{j \in [f+1,2f-f_p]} (\cvavector{\H_p[j]}{k-1} - \cvavector{p}{k-1})-\sum_{j \in [f_p+1,f]} (\cvavector{\H_q[j]}{k-1}- \cvavector{p}{k-1}) \right) \\
    &-  \left(\sum_{j \in F_q} (\cvavector{j}{k-1}- \cvavector{q}{k-1}) - \sum_{j \in [f_q]} (\cvavector{\H_p[j]}{k-1} - \cvavector{q}{k-1})  \right) \\
    &- \left. \left( \sum_{j \in [f+1,2f-f_q]} (\cvavector{\H_q[j]}{k-1}  - \cvavector{q}{k-1}) - \sum_{j \in [f_q+1,f]} (\cvavector{\H_p[j]}{k-1} - \cvavector{q}{k-1}) \right) \right\lVert.
\end{align*}
Using triangle inequality above we obtain that
\begin{align}\nonumber
    &\norm{\cvavector{p}{k}-\cvavector{q}{k}}
    \leq \frac{1}{n-2f} \left[ \left(\sum_{j \in F_p} \norm{\cvavector{j}{k-1} - \cvavector{p}{k-1}} + \sum_{j \in [f_p]} \norm{\cvavector{\H_q[j]}{k-1}- \cvavector{p}{k-1}}\right) \right. \\\nonumber
    &+ \left( \sum_{j \in [f+1,2f-f_p]} \norm{\cvavector{\H_p[j]}{k-1} - \cvavector{p}{k-1}} + \sum_{j \in [f_p+1,f]} \norm{\cvavector{\H_q[j]}{k-1}- \cvavector{p}{k-1}} \right) \\\nonumber
    &+  \left (\sum_{j \in F_q} \norm{\cvavector{j}{k-1}- \cvavector{q}{k-1}} + \sum_{j \in [f_q]} \norm{\cvavector{\H_p[j]}{k-1} - \cvavector{q}{k-1}}  \right)\\\nonumber
    &+ \left. \left( \sum_{j \in [f+1,2f-f_q]} \norm{\cvavector{\H_q[j]}{k-1}  - \cvavector{q}{k-1}} + \sum_{j \in [f_q+1,f]} \norm{\cvavector{\H_p[j]}{k-1} - \cvavector{q}{k-1}} \right) \right]
\end{align}
As the right hand side above is a summation over $4f$ terms, we obtain that
\begin{align}\nonumber
    &\norm{\cvavector{p}{k}-\cvavector{q}{k}}^2 \leq \frac{4f}{(n-2f)^2} \left[ \left(\sum_{j \in F_p} \norm{\cvavector{j}{k-1} - \cvavector{p}{k-1}}^2
    +\sum_{j \in [f_p]} \norm{\cvavector{\H_q[j]}{k-1}- \cvavector{p}{k-1}}^2\right) \right. \\ \nonumber
    &+ \left( \sum_{j \in [f+1,2f-f_p]} \norm{\cvavector{\H_p[j]}{k-1} - \cvavector{p}{k-1}}^2
    +\sum_{j \in [f_p+1,f]} \norm{\cvavector{\H_q[j]}{k-1}- \cvavector{p}{k-1}}^2 \right)
     \\ \nonumber
    &+  \left (\sum_{j \in F_q} \norm{\cvavector{j}{k-1}- \cvavector{q}{k-1}}^2 + \sum_{j \in [f_q]} \norm{\cvavector{\H_p[j]}{k-1} - \cvavector{q}{k-1}}^2  \right)\\ \nonumber
    &+ \left. \left( \sum_{j \in [f+1,2f-f_q]} \norm{\cvavector{\H_q[j]}{k-1}  - \cvavector{q}{k-1}}^2
    +\sum_{j \in [f_q+1,f]} \norm{\cvavector{\H_p[j]}{k-1} - \cvavector{q}{k-1}}^2 \right)
    \right]\\ \nonumber
    &\leq \frac{4f}{(n-2f)^2} \left[\sum_{j \in F_p} \norm{\cvavector{j}{k-1} - \cvavector{p}{k-1}}^2
    +\sum_{j \in [f]} \norm{\cvavector{\H_q[j]}{k-1}- \cvavector{p}{k-1}}^2 \right.
     + \sum_{j \in [f+1,2f-f_p]} \norm{\cvavector{\H_p[j]}{k-1} - \cvavector{p}{k-1}}^2\\
    &+ \sum_{j \in F_q} \norm{\cvavector{j}{k-1}- \cvavector{q}{k-1}}^2 + \sum_{j \in [f]} \norm{\cvavector{\H_p[j]}{k-1} - \cvavector{q}{k-1}}^2 
    + \left. \sum_{j \in [f+1,2f-f_q]} \norm{\cvavector{\H_q[j]}{k-1}  - \cvavector{q}{k-1}}^2
    \right].
        \label{eq:distance_bound_total}
\end{align}
Note that $\filter{p}{k}$ contains at least $f_p$ \byzantine{} nodes. Thus there are at most $n-2f-f_p$ correct nodes in $\filter{p}{k}$. This implies that there are at least $f+f_p$ correct nodes that are not selected by node $p$. We define $\H_p'$ to be a subset of $f + f_p$ correct nodes not in $\filter{p}{k}$ that are farthest from $\cvavector{p}{k-1}$. We sort the nodes in $\H_p'$ such that $\norm{\cvavector{p}{k-1}-\cvavector{\H_p'[i]}{k-1}}\leq \norm{\cvavector{p}{k-1}-\cvavector{\H_p'[i+1]}{k-1}}$ for $i = 1, \ldots, f + f_p - 1$. Note that for each \byzantine{} node in $\filter{p}{k}$ there is a correct node in set $\receive{p}{\cvaround} \setminus \filter{p}{k}$. Thus, by definition of $\filter{p}{k}$ in~\eqref{eq:filter} we obtain that 
\begin{equation}
\label{eq:first_distance}
    \sum_{j \in F_p} \norm{\cvavector{j}{k-1} - \cvavector{p}{k-1}}^2 \leq  \sum_{j \in [f+1,f+f_p]} \norm{\cvavector{\H_p'[j]}{k-1} - \cvavector{p}{k-1}}^2
\end{equation}
By definition of $\H_p'$, for each $j \in [f]$, $\norm{\cvavector{p}{k-1}-\cvavector{\H_q[j]}{k-1}}\leq \norm{\cvavector{p}{k-1}-\cvavector{\H_p'[j]}{k-1}}$. Thus,
\begin{equation}
\label{eq:second_distance}
     \sum_{j \in [f]} \norm{\cvavector{\H_q[j]}{k-1} - \cvavector{p}{k-1}}^2 \leq \sum_{j \in [f]} \norm{\cvavector{\H_p'[j]}{k-1} - \cvavector{p}{k-1}}^2.
\end{equation}
From~\eqref{eq:first_distance} and~\eqref{eq:second_distance} we obtain that 
\begin{align*}
    \sum_{j \in F_p} \norm{\cvavector{j}{k-1} - \cvavector{p}{k-1}}^2 + \sum_{j \in [f]} \norm{\cvavector{\H_q[j]}{k-1} - \cvavector{p}{k-1}}^2 \leq \sum_{j \in \H_p'} \norm{\cvavector{j}{k-1} - \cvavector{p}{k-1}}^2.
\end{align*}
Therefore, we have
\begin{align}
\label{eq:whole_distance_p}
    \sum_{j \in F_p} \norm{\cvavector{j}{k-1} - \cvavector{p}{k-1}}^2
    +\sum_{j \in [f]} \norm{\cvavector{\H_q[j]}{k-1}- \cvavector{p}{k-1}}^2 
     + \sum_{j \in [f+1,2f-f_p]} \norm{\cvavector{\H_p[j]}{k-1} - \cvavector{p}{k-1}}^2 \leq \sum_{j \in \H} \norm{\cvavector{j}{k-1} - \cvavector{p}{k-1}}^2.
\end{align}
Similarly,
\begin{align}
\label{eq:whole_distance_q}
    \sum_{j \in F_q} \norm{\cvavector{j}{k-1} - \cvavector{q}{k-1}}^2
    +\sum_{j \in [f]} \norm{\cvavector{\H_p[j]}{k-1}- \cvavector{q}{k-1}}^2 
     + \sum_{j \in [f+1,2f-f_q]} \norm{\cvavector{\H_q[j]}{k-1} - \cvavector{q}{k-1}}^2 \leq \sum_{j \in \H} \norm{\cvavector{j}{k-1} - \cvavector{q}{k-1}}^2.
\end{align}
Substituting from~\eqref{eq:whole_distance_p} and~\eqref{eq:whole_distance_q} in~\eqref{eq:distance_bound_total} we obtain that
\begin{align*}
    \norm{\cvavector{p}{k}-\cvavector{q}{k}}^2 \leq \frac{4f}{(n-2f)^2} \left[\sum_{j \in \H} \norm{\cvavector{j}{k-1} - \cvavector{p}{k-1}}^2 + \sum_{j \in \H} \norm{\cvavector{j}{k-1} - \cvavector{q}{k-1}}^2 \right].
\end{align*}
As the above holds true for an arbitrary pair of correct nodes $p$ and $q$, by averaging over all such possible pairs we obtain that 
\begin{align*}
    \frac{1}{(n-f)^2} \sum_{p,q \in \H} \norm{\cvavector{p}{k}-\cvavector{q}{k}}^2 \leq \frac{8f(n-f)}{(n-2f)^2}\frac{1}{(n-f)^2} \sum_{p,q \in \H} \norm{\cvavector{p}{k-1}-\cvavector{q}{k-1}}^2.
\end{align*}
Recall the notation $\diam{\cdot}{}$. The above implies that
\begin{align*}
    \diam{\cvavector{}{k}}{} \leq \frac{8f(n-f)}{(n-2f)^2} \diam{\cvavector{}{k-1}}{}.
\end{align*}
As $n \geq 11f$, $\frac{(n-f)^2}{(n-2f)^2} \leq \frac{100}{81}$. Using this above proves the lemma, i.e., 
\begin{align*}
    \diam{\cvavector{}{k}}{} \leq \frac{800f}{81(n-f)} \diam{\cvavector{}{k-1}}{} \leq \frac{9.88 f}{n-f}  \diam{\cvavector{}{k-1}}{}.
\end{align*}
\end{proof}

We now present below the proof of Lemma \ref{lemma:cva_bounds}. For convenience, let us recall the lemma below.

\noindent \fcolorbox{black}{white}{
\parbox{0.97\textwidth}{\centering
\begin{replemma}{lemma:cva_bounds}
    Suppose that  $n \geq 11 f$. For any $\numbercvarounds \geq 1$,  the coordination phase of Algorithm~\ref{algo} guarantees $(\alpha,\lambda)$-\reduction{}
  for 
  \begin{equation*}
      \alpha = \left(\frac{9.88 f}{n-f} \right)^K \quad \textnormal{and} \quad \lambda = \frac{9f}{n-f} \cdot \min \left\{ \numbercvarounds , \frac{1}{(1-\sqrt{\alpha})^2} \right\}.
  \end{equation*}
\end{replemma}
}
}

\begin{proof}
The first condition of $(\alpha,\lambda)$-\reduction{} stated in Definition~\ref{def:reduction}, i.e., $\diam{\cvavector{}{\numbercvarounds}}{} \leq \alpha \diam{\cvavector{}{0}}{}$, follows trivially from Lemma \ref{lemma:cva_diameter} for the stated value of $\alpha$. We show below the second condition, i.e., $\norm{\cvaavg{0}-\cvaavg{\numbercvarounds}}^2 \leq \lambda \diam{\cvavector{}{0}}{}$, for the stated $\lambda$.

For doing so, we first consider an arbitrary round $k \in [\numbercvarounds]$. For each correct node $i$, by definition of NNA operator in~\eqref{eqn:app_def_nna} we have that
\begin{align*}
    \cvavector{i}{k} - \cvaavg{k-1} & = \frac{1}{n-2f} \sum_{j\in\filter{i}{k}} \cvavector{j}{k-1} - \frac{1}{n-f} \sum_{j \in \H} \cvavector{j}{k-1} \\
    & = \frac{1}{n-2f} \sum_{j\in\filter{i}{k}} \left( \cvavector{j}{k-1} - \cvavector{i}{k-1}\right) - \frac{1}{n-f} \sum_{j \in \H} \left( \cvavector{j}{k-1} - \cvavector{i}{k-1}\right) .
\end{align*}    
Upon decomposing the right hand side we obtain that
\begin{align*}   
    \cvavector{i}{k} - \cvaavg{k-1} & = \left(\frac{1}{n-2f} - \frac{1}{n-f} \right) \sum_{j\in\filter{i}{k} \cap \H} \left(\cvavector{j}{k-1} -  \cvavector{i}{k-1}\right) + \frac{1}{n-2f} \sum_{j\in\filter{i}{k} \setminus \H} \left( \cvavector{j}{k-1} - \cvavector{i}{k-1} \right) \\
    & - \frac{1}{n-f} \sum_{j\in \H \setminus\filter{i}{k}} \left( \cvavector{j}{k-1} - \cvavector{i}{k-1}\right). 
\end{align*}    
Thus,
\begin{align*}
    \cvavector{i}{k} - \cvaavg{k-1} = \frac{1}{(n-f)(n-2f)} & \left(f \sum_{j\in\filter{i}{k} \cap \H} (\cvavector{j}{k-1}-\cvavector{i}{k-1} ) + (n-f)\sum_{j\in\filter{i}{k} \setminus \H} (\cvavector{j}{k-1}-\cvavector{i}{k-1} ) \right.\\
    & \left.- (n-2f)\sum_{j\in \H \setminus \filter{i}{k}} (\cvavector{j}{k-1}-\cvavector{i}{k-1} ) \right).
\end{align*}
By taking norm on both sides and then applying the triangle inequality we obtain that
\begin{align}\label{eq:average_distance}
    \norm{\cvavector{i}{k} - \cvaavg{k-1}} &\leq \frac{1}{(n-f)(n-2f)} \left( f \sum_{j\in\filter{i}{k} \cap \H} \norm{\cvavector{j}{k-1}-\cvavector{i}{k-1}}  \right.\\
    &+ \left.(n-f)\sum_{j\in\filter{i}{k} \setminus \H} \norm{\cvavector{j}{k-1}-\cvavector{i}{k-1}} +(n-2f)\sum_{j\in \H \setminus \filter{i}{k}} \norm{\cvavector{j}{k-1}-\cvavector{i}{k-1}}\right).\nonumber
\end{align}
Now let $v := \card{\filter{i}{k} \cap \H}$. We then have $v = \card{\filter{i}{k}} + \card{ \H} - \card{\filter{i}{k} \cup \H} \geq n-2f+n-f-n = n -3f$. Also, $\card{\filter{i}{k} \setminus \H} = n-2f-v$ and $\card{\H \setminus \filter{i}{k}} = n -f - v$.
There for the number $A(v)$ of items that are added in \eqref{eq:average_distance} is
\begin{equation}
    A(v) = fv + (n-2f-v)(n-f) + (n-2f)(n-f-v) = 2 (n-2f)(n-f-v),
\end{equation}
which is decreasing in $v$. There the maximum of $A(v)$ is reached for $v = n -3f$ and we have $A(v) \leq 4f(n-2f)$.
Therefore, \eqref{eq:average_distance} yields
\begin{align}\nonumber
    \norm{\cvavector{i}{k} - \cvaavg{k-1}}^2 &\leq \frac{4f(n-2f)}{(n-f)^2(n-2f)^2} \left( f \sum_{j\in\filter{i}{k} \cap \H} \norm{\cvavector{j}{k-1}-\cvavector{i}{k-1}}^2  \right.\\\nonumber
    &+ \left.(n-f)\sum_{j\in\filter{i}{k} \setminus \H} \norm{\cvavector{j}{k-1}-\cvavector{i}{k-1}}^2 +(n-2f)\sum_{j\in \H \setminus \filter{i}{k}} \norm{\cvavector{j}{k-1}-\cvavector{i}{k-1}}^2\right) \\\nonumber
    &\leq \frac{4f(n-2f)}{(n-f)^2(n-2f)^2} \left( f \sum_{j\in \H} \norm{\cvavector{j}{k-1}-\cvavector{i}{k-1}}^2  \right.\\ \nonumber
    &+ \left.(n-f)\sum_{j\in \H} \norm{\cvavector{j}{k-1}-\cvavector{i}{k-1}}^2 +(n-2f)\sum_{j\in \H } \norm{\cvavector{j}{k-1}-\cvavector{i}{k-1}}^2\right)\\ \nonumber
    &\leq \frac{4f(n-2f)(2n-2f)}{(n-f)^2(n-2f)^2}\sum_{j\in \H } \norm{\cvavector{j}{k-1}-\cvavector{i}{k-1}}^2\\
    &= \frac{8f}{(n-f)(n-2f)}\sum_{j\in \H } \norm{\cvavector{j}{k-1}-\cvavector{i}{k-1}}^2. \label{eq:single_distance_average}
\end{align}
But now note that 
\begin{align*}
    \norm{\cvaavg{k}-\cvaavg{k-1}}^2 &= \norm{\frac{1}{n-f}\sum_{i \in \H} \cvavector{i}{k} - \cvaavg{k-1}}^2\\
    &\leq \frac{1}{n-f} \sum_{i \in \H} \norm{\cvavector{i}{k} - \cvaavg{k-1}}^2.
\end{align*}
Combining above with \eqref{eq:single_distance_average} then yields 
\begin{align*}
    \norm{\cvaavg{k}-\cvaavg{k-1}}^2  \leq \frac{8f}{n-2f} \cdot \frac{1}{(n-f)^2}\sum_{i,j \in \H }\norm{\cvavector{j}{k-1}-\cvavector{i}{k-1}}^2. 
\end{align*}
Using the notation $\diam{\cdot}{}$, we then have
\begin{align}\label{eq:single_iteration_noise_bound}
    {\norm{\cvaavg{k}-\cvaavg{k-1}}^2} \leq \frac{8f}{n-2f} \diam{\cvavector{}{k-1}}{} \leq \frac{8f\alpha^{k-1} }{n-2f}\diam{\cvavector{}{0}}{},
\end{align}
where in the second inequality we used Lemma \ref{lemma:cva_diameter}.
Now note that 
\begin{align*}
    {\norm{\cvaavg{\numbercvarounds}-\cvaavg{0}}^2} &= {\norm{\sum_{k \in [\numbercvarounds]}(\cvaavg{k}-\cvaavg{k-1})}^2}\\
    &= {\iprod{\sum_{k \in [\numbercvarounds]}(\cvaavg{k}-\cvaavg{k-1})}{\sum_{k \in [\numbercvarounds]}(\cvaavg{k}-\cvaavg{k-1})}}\\
    &= {\sum_{k,l \in [\numbercvarounds]} \iprod{\cvaavg{k}-\cvaavg{k-1}}{\cvaavg{l}-\cvaavg{l-1}}}.
\end{align*}
By the Cauchy–Schwarz inequality we then have
\begin{align*}
    {\norm{\cvaavg{\numbercvarounds}-\cvaavg{0}}^2} &\leq \sum_{k,l \in [\numbercvarounds]}\sqrt{{\norm{\cvaavg{k}-\cvaavg{k-1}}^2}\cdot{\norm{\cvaavg{l}-\cvaavg{l-1}}^2}}.
\end{align*}
Combining this with \ref{eq:single_iteration_noise_bound}, we obtain that
\begin{align*}
    {\norm{\cvaavg{\numbercvarounds}-\cvaavg{0}}^2} &\leq \frac{8f}{n-2f} \diam{\cvavector{}{0}}{} \sum_{k,l \in [\numbercvarounds]}\sqrt{\alpha^{k-1}\alpha^{l-1}}\\
    &= \frac{8f}{n-2f} \diam{\cvavector{}{0}}{} \sum_{k \in [\numbercvarounds]} \left(\sqrt{\alpha}\right)^{k-1}\sum_{l \in [\numbercvarounds]} \left(\sqrt{\alpha}\right)^{l-1}.
\end{align*}
Now since $\alpha < 1$ we have $\sum_{k \in [\numbercvarounds]} \left(\sqrt{\alpha}\right)^{k-1} \leq \numbercvarounds$, and thus
\begin{align}
\label{eq:noise_first_inequality}
    {\norm{\cvaavg{\numbercvarounds}-\cvaavg{0}}^2} &\leq \frac{8f\numbercvarounds^2}{n-2f} \diam{\cvavector{}{0}}{}.
\end{align}
Moreover, we have 
\begin{align*}
    \sum_{k \in [\numbercvarounds]} \left(\sqrt{\alpha}\right)^{k-1} \leq \sum_{k =1}^\infty \left(\sqrt{\alpha}\right)^{k-1} = \frac{1}{1-\sqrt{\alpha}},
\end{align*}
and thus 
\begin{align}
\label{eq:noise_second_inequality}
    {\norm{\cvaavg{\numbercvarounds}-\cvaavg{0}}^2} &\leq \frac{8f}{n-2f}\cdot\frac{1}{(1-\sqrt{\alpha})^2} \diam{\cvavector{}{0}}{}.
\end{align}
Combining \eqref{eq:noise_first_inequality} and \eqref{eq:noise_second_inequality} and noting that $\frac{8f}{n-2f} \leq \frac{9f}{n-f}$ for $n\geq11f$ proves the lemma.
\end{proof}

\subsection{Proof of Lemma~\ref{lem:drift}}
\label{app:lem_drift}

We recall the lemma below.

\noindent \fcolorbox{black}{white}{
\parbox{0.97\textwidth}{\centering
\begin{replemma}{lem:drift}
Suppose that assumptions~\ref{asp:lip},~\ref{asp:bnd_var}, and \ref{asp:heter} hold true. Consider Algorithm~\ref{algo} with $\gamma \leq \frac{1 - \alpha}{L \sqrt{27 \alpha \left(1 + \alpha \right)}}$, and $\beta>0$. Suppose that the coordination phase satisfies $(\alpha,\lambda)$-\reduction{} for $\alpha < 1$. For each $t \in [T]$, we obtain that
\begin{align*}
    \expect{\diam{\model{}{}}{t}} &\leq \coeffalpha\gamma^2 \left( \sigma^2 \frac{1 - \beta}{1 + \beta} + 3 \heter^2 \right),
\end{align*}
and
\begin{align*}
    \expect{\diam{\mmt{}{}}{t}} \leq3 \sigma^2 \left(\frac{1 - \beta}{1 + \beta} \right) + 9 \heter^2 + 9 L^2 \gamma^2 \coeffalpha \left( \sigma^2 \frac{1 - \beta}{1 + \beta} + 3 \heter^2 \right),
\end{align*}
where 
\begin{align*}
    \coeffalpha:= \frac{18 \alpha (1 + \alpha)}{(1 - \alpha)^2}.
\end{align*}
\end{replemma}
}}
\begin{proof}
Consider an arbitrary step $t \in [T]$. The proof comprises 3 steps. \\

{\bf Step i.} In this step, we analyse the growth of $\expect{\diam{\model{}{}}{t}}$. From Algorithm \ref{algo} recall that for all $i \in \H$, we have $\model{i}{t+\nicefrac{1}{2}} = \model{i}{t} - \gamma \mmt{i}{t}$. As $(x + y)^2 \leq (1 + c) x^2 + (1 + \nicefrac{1}{c}) y^2$ for any $c > 0$, we obtain for all $i, \, j \in \H$ that
\begin{align*}
    \expect{\norm{\model{i}{t+\nicefrac{1}{2}} - \model{j}{t+\nicefrac{1}{2}} }^2} & \leq \expect{\norm{\model{i}{t} - \model{j}{t} - \gamma \left( \mmt{i}{t} - \mmt{j}{t} \right)}^2} \\
    & \leq (1 + c) \expect{\norm{\model{i}{t} - \model{j}{t}}^2} + \left(1 + \frac{1}{c} \right) \gamma^2 \expect{\norm{\mmt{i}{t} - \mmt{j}{t}}^2}.
\end{align*}
Thus, by definition of notation $\diam{*}{t}$ and using Lemma~\ref{lem:diam_equal}, we have
\begin{align}
    \expect{\diam{\model{}{}}{t+\nicefrac{1}{2}}} \leq (1 + c) \expect{\diam{\model{}{}}{t}} + \left(1 + \frac{1}{c} \right) \gamma^2 \expect{\diam{\mmt{}{}}{t}}. \label{eqn:drift_before_alpha}
\end{align}
Recall that, the coordination phase of Algorithm~\ref{algo} satisfies $(\alpha,\lambda)$-reduction. Thus, for all $t$, we have $\diam{\model{}{}}{t+1} \leq \alpha \diam{\model{}{}}{t+\nicefrac{1}{2}}$. Substituting from above we obtain that
\begin{align}
   \expect{\diam{\model{}{}}{t+1}} \leq (1 + c) \alpha  \expect{\diam{\model{}{}}{t}} + \left(1 + \frac{1}{c} \right) \alpha \gamma^2 \expect{\diam{\mmt{}{}}{t}}. \label{eqn:drift_1}
\end{align}

{\bf Step ii.} In this step, we analyse the growth of $\expect{\diam{\mmt{}{}}{t}}$. From the definition of momentum in~\eqref{eqn:mmt_i}, we obtain for all $i \, j \in \H$ that
\begin{align*}
    & \expect{\norm{\mmt{i}{t} - \mmt{j}{t}}^2} = \expect{\norm{ (1 - \beta) \sum_{s = 1}^t \beta^{t - s} \left( \gradient{i}{s} - \gradient{j}{s} \right) }^2} \\
    & = (1 - \beta)^2 \expect{\norm{ \sum_{s = 1}^t \beta^{t - s} \left(\gradient{i}{s} - \localgrad{i}\left( \model{i}{s}\right) + \localgrad{i}\left( \model{i}{s}\right) - \localgrad{j}\left( \model{j}{s}\right) + \localgrad{j}\left( \model{j}{s}\right) - \gradient{j}{s} \right) }^2}.
\end{align*}

Using the fact that $(x + y + z)^2 \leq 3x^2 + 3 y^2 + 3 z^2$, from above we obtain that
\begin{align}
    \expect{\norm{\mmt{i}{t} - \mmt{j}{t}}^2} &\leq  3 (1 - \beta)^2 \expect{\norm{ \sum_{s = 1}^t \beta^{t - s} \left(\gradient{i}{s} - \localgrad{i}\left( \model{i}{s}\right) \right)}^2} \nonumber \\ &+ 3 (1 - \beta)^2 \expect{\norm{ \sum_{s = 1}^t \beta^{t - s} \left(\gradient{j}{s} - \localgrad{j}\left( \model{j}{s}\right) \right)}^2} \nonumber \\
    & + 3 (1 - \beta)^2 \expect{\norm{ \sum_{s = 1}^t \beta^{t - s} \left( \localgrad{i}\left( \model{i}{s}\right) - \localgrad{j}\left( \model{j}{s}\right) \right)}^2}. \label{eqn:before_At_i}
\end{align}

Consider an arbitrary $i \in \H$, and denote
\begin{align}
    A_t \coloneqq \expect{\norm{ \sum_{s = 1}^t \beta^{t - s} \left(\gradient{i}{s} - \localgrad{i}\left( \model{i}{s}\right) \right)}^2}. \label{eqn:define_A_t}
\end{align}
Note that 
\begin{align*}
    A_t &= \expect{\norm{ \sum_{s = 1}^t \beta^{t - s} \left(\gradient{i}{s} - \localgrad{i}\left( \model{i}{s}\right) \right)}^2}\\ & = \expect{\norm{ \sum_{s = 1}^{t - 1} \beta^{t - s} \left(\gradient{i}{s} - \localgrad{i}\left( \model{i}{s}\right) \right) + \left( \gradient{i}{t} - \localgrad{i}\left( \model{i}{t}\right) \right)}^2}. 
\end{align*}
From above we obtain that 
\begin{align*}
    A_t &= \expect{\norm{ \sum_{s = 1}^{t - 1} \beta^{t - s} \left(\gradient{i}{s} - \localgrad{i}\left( \model{i}{s}\right) \right)}^2} + \expect{\norm{\gradient{i}{t} - \localgrad{i}\left( \model{i}{t}\right)}^2} \\ &+ \expect{\iprod{\sum_{s = 1}^{t - 1} \beta^{t - s} \left(\gradient{i}{s} - \localgrad{i}\left( \model{i}{s}\right) \right)}{\gradient{i}{t} - \localgrad{i}\left( \model{i}{t}\right)}}.
\end{align*}
Recall that in the above, $\expect{\cdot} = \condexpect{1}{ \ldots \condexpect{t}{\cdot}}$. Thus, due to Assumption~\ref{asp:bnd_var}, we have $\expect{\norm{\gradient{i}{t} - \localgrad{i} Q\left( \model{i}{t}\right)}^2} \leq \sigma^2$. Using this above we obtain that
\begin{align}
    A_t &\leq \expect{\norm{ \sum_{s = 1}^{t - 1} \beta^{t - s} \left(\gradient{i}{s} - \localgrad{i}\left( \model{i}{s}\right) \right)}^2} + \sigma^2 \nonumber \\ &+ \expect{\iprod{\sum_{s = 1}^{t - 1} \beta^{t - s} \left(\gradient{i}{s} - \localgrad{i}\left( \model{i}{s}\right) \right)}{\gradient{i}{t} - \localgrad{i}\left( \model{i}{t}\right)}}. \label{eqn:At_label_1}
\end{align}
Also,  by tower rule we have
\begin{align*}
    &\expect{\iprod{\sum_{s = 1}^{t - 1} \beta^{t - s} \left(\gradient{i}{s} - \localgrad{i}\left( \model{i}{s}\right) \right)}{\gradient{i}{t} - \localgrad{i}\left( \model{i}{t}\right)}} = \\
    &\condexpect{1}{ \ldots \condexpect{t}{\iprod{\sum_{s = 1}^{t - 1} \beta^{t - s} \left(\gradient{i}{s} - \localgrad{i}\left( \model{i}{s}\right) \right)}{\gradient{i}{t} - \localgrad{i}\left( \model{i}{t}\right)}} }.
\end{align*}
By the definition of conditional expectation $\condexpect{t}{\cdot}$, we have
\begin{align*}
    \condexpect{t}{\iprod{\sum_{s = 1}^{t - 1} \beta^{t - s} \left(\gradient{i}{s} - \localgrad{i}\left( \model{i}{s}\right) \right)}{\gradient{i}{t} - \localgrad{i}\left( \model{i}{t}\right)}} = \\  \iprod{\sum_{s = 1}^{t - 1} \beta^{t - s} \left(\gradient{i}{s} - \localgrad{i}\left( \model{i}{s}\right) \right)}{\condexpect{t}{\gradient{i}{t} - \localgrad{i}\left( \model{i}{t}\right)}}.
\end{align*}
By Assumption~\ref{asp:bnd_var}, we obtain that $\condexpect{t}{\gradient{i}{t} - \localgrad{i}\left( \model{i}{t}\right)} = \localgrad{i}\left( \model{i}{t}\right) - \localgrad{i}\left( \model{i}{t}\right) = 0$. Using this above implies that
\begin{align*}
    \expect{\iprod{\sum_{s = 1}^{t - 1} \beta^{t - s} \left(\gradient{i}{s} - \localgrad{i}\left( \model{i}{s}\right) \right)}{\gradient{i}{t} - \localgrad{i}\left( \model{i}{t}\right)}}  = 0.
\end{align*}
Substituting from above in~\eqref{eqn:At_label_1} we obtain that
\begin{align*}
    A_t  &\leq \expect{\norm{ \sum_{s = 1}^{t - 1} \beta^{t - s} \left(\gradient{i}{s} - \localgrad{i}\left( \model{i}{s}\right) \right)}^2} + \sigma^2 \\ &= \beta^2 \expect{\norm{ \sum_{s = 1}^{t - 1} \beta^{t - 1 - s} \left(\gradient{i}{s} - \localgrad{i}\left( \model{i}{s}\right) \right)}^2} + \sigma^2 \\ &= \beta^2 A_{t-1} + \sigma^2.
\end{align*}
Note from the definition of $A_t$ in~\eqref{eqn:define_A_t} that, under Assumption~\ref{asp:bnd_var}, $A_1 \leq \sigma^2$. Thus, from above we obtain that
\begin{align*}
    A_t \coloneqq \expect{\norm{ \sum_{s = 1}^t \beta^{t - s} \left(\gradient{i}{s} - \localgrad{i}\left( \model{i}{s}\right) \right)}^2} \leq \sigma^2 \sum_{s = 0}^{t-1} \beta^{2s} \leq \frac{\sigma^2}{1 - \beta^2}
\end{align*}
Substituting from above in~\eqref{eqn:before_At_i}, computing the pair-wise average, and using Lemma~\ref{lem:diam_equal}, we obtain that
\begin{align}\nonumber
    &\expect{\diam{\mmt{}{}}{t}} = \frac{1}{2(n-f)^2} \sum_{i,j \in \H} \expect{\norm{\mmt{i}{t} - \mmt{j}{t}}^2} \\ &\leq 3 (1 - \beta)^2 \, \frac{\sigma^2 }{1 - \beta^2}+\frac{3 (1 - \beta)^2 }{2(n-f)^2} \sum_{i,j \in \H} \expect{\norm{ \sum_{s = 1}^t \beta^{t - s} \left(\localgrad{i}\left( \model{i}{s}\right) - \localgrad{j}\left( \model{j}{s}\right) \right)}^2} . \label{eqn:after_At_i}
\end{align}
Let us denote
\begin{align}
    C_t \coloneqq \frac{1}{(n-f)^2} \sum_{i,j \in \H} \expect{\norm{ \sum_{s = 1}^t \beta^{t - s} \left(\localgrad{i}\left( \model{i}{s}\right) - \localgrad{j}\left( \model{j}{s}\right) \right)}^2}. \label{eqn:def_Ct_i}
\end{align}
Now note that
\begin{align*}
    &C_t = \frac{1}{(n-f)^2} \sum_{i,j \in \H} \expect{\norm{ \beta \sum_{s = 1}^{t-1} \beta^{t - 1 - s} \left(\localgrad{i}\left( \model{i}{s}\right) - \localgrad{j}\left( \model{j}{s}\right) \right) + \left( \localgrad{i}\left( \model{i}{t}\right) - \localgrad{j}\left( \model{j}{t} \right) \right)}^2}
\end{align*}
By Jensen's inequality, we obtain that
\begin{align*}
    C_t &\leq \frac{1}{(n-f)^2} \sum_{i,j \in \H} \beta \expect{\norm{\sum_{s = 1}^{t-1} \beta^{t - 1 - s} \left(\localgrad{i}\left( \model{i}{s}\right) - \localgrad{j}\left( \model{j}{s}\right) \right)}^2} \\ &+ \frac{1}{(n-f)^2} \sum_{i,j \in \H} (1 - \beta) \expect{\norm{\frac{1}{1 - \beta} \, \left(\localgrad{i}\left( \model{i}{t}\right) - \localgrad{j}\left( \model{j}{t} \right) \right)}^2} \\
    &= \beta C_{t-1} + \frac{1}{(n-f)^2(1-\beta)} \sum_{i,j \in \H} \expect{\norm{ \localgrad{i}\left( \model{i}{t}\right) - \localgrad{j}\left( \model{j}{t} \right) }^2}.
\end{align*}
Now using the fact that $(x + y + z)^2 \leq 3x^2 + 3 y^2 + 3 z^2$, we obtain that
\begin{align*}
    &\norm{\localgrad{i}\left( \model{i}{t}\right) - \localgrad{j}\left( \model{j}{t} \right) }^2 \\ &= \norm{\localgrad{i}\left( \model{i}{t}\right) - \localgrad{i}\left( \avgmodel{t} \right) + \localgrad{i}\left( \avgmodel{t}\right) - \localgrad{j}\left( \avgmodel{t} \right) + \localgrad{j}\left( \avgmodel{t}\right) - \localgrad{j}\left( \model{j}{t} \right)}^2 \\ 
    &\leq 3  \norm{\localgrad{i}\left( \model{i}{t}\right) - \localgrad{i}\left( \avgmodel{t} \right) }^2 + 3 \norm{\localgrad{i}\left( \avgmodel{t}\right) - \localgrad{j}\left( \avgmodel{t} \right) }^2 + 3 \norm{\localgrad{j}\left( \avgmodel{t}\right) - \localgrad{j}\left( \model{j}{t} \right)}^2.
\end{align*}

By Assumption~\ref{asp:lip}, we have that $\norm{\localgrad{i}\left( \model{i}{t}\right) - \localgrad{i}\left( \avgmodel{t} \right) }^2 \leq L^2 \norm{\model{i}{t} - \avgmodel{t}}^2$ and $\norm{\localgrad{j}\left( \model{j}{t}\right) - \localgrad{j}\left( \avgmodel{t} \right) }^2 \leq L^2 \norm{\model{j}{t} - \avgmodel{t}}^2$. Using this above we obtain that
\begin{align}\nonumber
    C_t &\leq \beta C_{t - 1} +  \frac{3L^2}{(n-f)^2(1-\beta)} \sum_{i,j \in \H} \left(\expect{\norm{\model{i}{t} - \avgmodel{t}}^2} +\expect{\norm{\model{j}{t} - \avgmodel{t}}^2} \right)\\ &+ \frac{3}{(n-f)^2(1-\beta)} \sum_{i,j \in \H} \norm{\localgrad{i}\left( \avgmodel{t}\right) - \localgrad{j}\left( \avgmodel{t} \right) }^2. \label{eq:new_heter_step}
\end{align}
Now by Lemma~\ref{lem:diam_equal} and Assumption~\ref{asp:heter}, we have 
\begin{align*}
    &\frac{1}{(n-f)^2}\sum_{i,j \in \H} \norm{\localgrad{i}\left( \avgmodel{t}\right) - \localgrad{j}\left( \avgmodel{t} \right) }^2 = \frac{2}{n-f} \sum_{i \in \H} \norm{\localgrad{i}\left( \avgmodel{t}\right) - \localgrad{\H}\left( \avgmodel{t} \right)}^2 \leq 2\heter^2.
\end{align*}
Combining above with~\eqref{eq:new_heter_step}, we obtain that
\begin{align}\nonumber
    C_t &\leq \beta C_{t - 1} + \frac{6L^2}{(n-f)(1-\beta)} \sum_{i \in \H} \expect{\norm{\model{i}{t} - \avgmodel{t}}^2} + \frac{12\heter^2}{1-\beta}\\
    &= \beta C_{t - 1} + \frac{6L^2}{(1-\beta)} \expect{\diam{\model{}{}}{t}} + \frac{6\heter^2}{1-\beta}.
    \label{eq:c_tstep}
\end{align}

Now note that by definition, $C_0 = 0$. Thus, from above we obtain that
\begin{align*}
    C_t &\leq \frac{6L^2}{1 - \beta} \sum_{s = 1}^t \beta^{t - s} \diam{\model{}{}}{s} + \frac{6\heter^2}{1-\beta} \sum_{s = 1}^t \beta^{t - s} \\
    &\leq \frac{6L^2}{1 - \beta} \sum_{s = 1}^t \beta^{t - s} \diam{\model{}{}}{s} + \frac{6\heter^2}{1-\beta} \sum_{s = 0}^\infty \beta^{s} \\
    &= \frac{6L^2}{1 - \beta} \sum_{s = 1}^t \beta^{t - s} \diam{\model{}{}}{s} + \frac{6\heter^2}{(1-\beta)^2}.
\end{align*}
Substituting from above in~\eqref{eqn:after_At_i} we obtain that
\begin{align}
    \expect{\diam{\mmt{}{}}{t}} \leq 3 \sigma^2 \left(\frac{1 - \beta}{1 + \beta} \right) + 9 (1 - \beta) L^2 \sum_{s = 1}^t \beta^{t - s} \expect{\diam{\model{}{}}{s}} + 9 \heter^2 .  \label{eqn:drift_2}
\end{align}
Recall from~\eqref{eqn:drift_1} that
\begin{align}
   \expect{\diam{\model{}{}}{t+1}} \leq (1 + c) \alpha  \expect{\diam{\model{}{}}{t}} + \left(1 + \frac{1}{c} \right) \alpha \gamma^2 \expect{\diam{\mmt{}{}}{t}}. \label{eqn:drift_1_1}
\end{align}
In the next and the final step we use the results derived in~\eqref{eqn:drift_2} and~\eqref{eqn:drift_1_1} above to conclude the proof.\\

{\bf Step iii.} 
Now in~\eqref{eqn:drift_1_1} we let
\begin{align}
    c = \frac{1 - \alpha}{ 2 \alpha} > 0.
\end{align}
Substituting this in~\eqref{eqn:drift_1_1} we obtain that
\begin{align*}
   \expect{\diam{\model{}{}}{t+1}} \leq \frac{1 + \alpha}{2} \expect{\diam{\model{}{}}{t}} + \left(\frac{1 + \alpha}{1 - \alpha}\right) \alpha \gamma^2 \expect{\diam{\mmt{}{}}{t}}.
\end{align*}
Substituting from~\eqref{eqn:drift_2} above we obtain that
\begin{align} \nonumber
   \expect{\diam{\model{}{}}{t+1}} & \leq \left(\frac{1 + \alpha}{2}\right) \expect{\diam{\model{}{}}{t}} \\ \nonumber &+ \frac{\alpha (1 + \alpha)}{1 - \alpha} \gamma^2 \left( 3 \sigma^2 \left(\frac{1 - \beta}{1 + \beta} \right) + 9 \heter^2 + 9 (1 - \beta) L^2 \sum_{s = 1}^t \beta^{t - s} \expect{\diam{\model{}{}}{s}}  \right)  \nonumber \\ \nonumber
   & = \left(\frac{1 + \alpha}{2}\right) \expect{\diam{\model{}{}}{t}} + \frac{ \alpha (1 + \alpha)}{1 - \alpha} \left(3 \sigma^2 \frac{1 - \beta}{1 + \beta}  + 9 \heter^2 \right) \\  &+ 9 (1 - \beta) \gamma^2 L^2 \left( \frac{\alpha (1 + \alpha)}{1 - \alpha} \right) \sum_{s = 1}^t \beta^{t - s} \expect{\diam{\model{}{}}{s}} . \label{eqn:before_gamma_L}
\end{align}
\textcolor{black}{As we assume that $\gamma \leq  \frac{1 - \alpha}{L \sqrt{27 \alpha \left(1 + \alpha \right)}}$}, we have
\begin{align*}
    \gamma^2 L^2 \leq \frac{\left( 1 - \alpha\right)^2}{27 \alpha \left( 1 + \alpha \right)}.
\end{align*}
Using the above in~\eqref{eqn:before_gamma_L}, we obtain that
\begin{align*}
   \expect{\diam{\model{}{}}{t+1}} &\leq \left(\frac{1 + \alpha}{2}\right) \expect{\diam{\model{}{}}{t}} + \frac{ \alpha (1 + \alpha)  \gamma^2}{1 - \alpha} \left( \frac{1 - \beta}{1 + \beta} 3\sigma^2 + 9 \heter ^2  \right) \\& + \left(\frac{1 - \alpha}{3} \right)(1 - \beta) \sum_{s = 1}^t \beta^{t - s} \expect{\diam{\model{}{}}{s}} .
\end{align*}
For convenience, we denote
\begin{align*}
    D =  \frac{ \alpha (1 + \alpha) \gamma^2}{1 - \alpha} \left(3 \sigma^2 \frac{1 - \beta}{1 + \beta} + 9\heter^2\right).
\end{align*}
Thus,
\begin{align}
   \expect{\diam{\model{}{}}{t+1}} \leq \left(\frac{1 + \alpha}{2}\right) \expect{\diam{\model{}{}}{t}} + D + \left(\frac{1 - \alpha}{3} \right) (1 - \beta) \sum_{s = 1}^t \beta^{t - s} \expect{\diam{\model{}{}}{s}} . \label{eqn:after_gamma_L}
\end{align}
As~\eqref{eqn:after_gamma_L} above holds true for an arbitrary $t \in [T]$, we reason below by mathematical induction that for all $t$,
\begin{align}
     \expect{\diam{\model{}{}}{t}} \leq \frac{6 D}{1 - \alpha}. \label{eqn:induction_assum_1}
\end{align}
First, note that, as $\model{i}{1} = \model{j}{1}$ for all $i, \, j \in \H$, the above is trivially true for $t = 1$. Second, let us assume that~\eqref{eqn:induction_assum_1} is true for all $t \leq k$. Then, from~\eqref{eqn:after_gamma_L} we obtain that
\begin{align*}
    \expect{\diam{\model{}{}}{k+1}} & \leq \left(\frac{1 + \alpha}{2}\right) \frac{6 D}{1 - \alpha} + D + \left(\frac{1 - \alpha}{3} \right) (1 - \beta) \sum_{s = 1}^k \beta^{k - s} \frac{6 D}{1 - \alpha} \\
    & = \left(\frac{1 + \alpha}{2}\right) \frac{6 D}{1 - \alpha} + D + 2 D =  \frac{6 D}{1 - \alpha}.
\end{align*}
Thus,~\eqref{eqn:induction_assum_1} holds true for $k+1$. Therefore, for all $t \in [T]$, we have
\begin{align*}
    \expect{\diam{\model{}{}}{t}} &\leq  \frac{ \alpha (1 + \alpha) \gamma^2}{(1 - \alpha)^2} \left(18 \sigma^2 \frac{1 - \beta}{1 + \beta} + 54\heter^2\right)
\end{align*}
We now define $\coeffalpha:= \frac{18 \alpha (1 + \alpha)}{(1 - \alpha)^2}$. We then have 
\begin{align*}
     \expect{\diam{\model{}{}}{t}} &\leq \coeffalpha\gamma^2 \left( \sigma^2 \frac{1 - \beta}{1 + \beta} + 3 \heter^2 \right).
\end{align*}

Combining this with \eqref{eqn:drift_2}, we then obtain
\begin{align*}
     \expect{\diam{\mmt{}{}}{t}} &\leq 3 \sigma^2 \left(\frac{1 - \beta}{1 + \beta} \right) + 9 (1 - \beta) L^2 \sum_{s = 1}^t \beta^{t - s} \expect{\diam{\model{}{}}{s}} + 9 \heter^2\\
      &\leq  3 \sigma^2 \left(\frac{1 - \beta}{1 + \beta} \right) + 9 \heter^2 + 9 (1 - \beta) L^2 \sum_{s = 1}^t \beta^{t - s} \left( \coeffalpha\gamma^2 \left( \sigma^2 \frac{1 - \beta}{1 + \beta} + 3 \heter^2 \right) \right) \\
       &\leq  3 \sigma^2 \left(\frac{1 - \beta}{1 + \beta} \right) + 9 \heter^2 + 9 (1 - \beta) L^2 \sum_{s = 0}^\infty \beta^{s} \left( \coeffalpha\gamma^2 \left( \sigma^2 \frac{1 - \beta}{1 + \beta} + 3 \heter^2 \right) \right) \\
      & =  3 \sigma^2 \left(\frac{1 - \beta}{1 + \beta} \right) + 9 \heter^2 + 9 L^2 \gamma^2 \coeffalpha \left( \sigma^2 \frac{1 - \beta}{1 + \beta} + 3 \heter^2 \right) 
\end{align*}
\end{proof}

\subsection{Proof of Lemma \ref{lem:dev}}
\label{app:lem_deviation}

Recall that 
\begin{equation*}
    \AvgGrad{t} \coloneqq \frac{1}{n-f}\sum_{i \in \H} \localgrad{i} (\model{i}{t}),
\end{equation*}
and that
\begin{equation}
\label{eq:deviation_definition}
    \dev{t} \coloneqq \AvgMmt{t} - \AvgGrad{t}.
\end{equation}
Also, we recall the lemma below.

\noindent \fcolorbox{black}{white}{
\parbox{0.97\textwidth}{\centering
\begin{replemma}{lem:dev}
Suppose that assumptions~\ref{asp:lip} and~\ref{asp:bnd_var} hold true. Consider Algorithm~\ref{algo}. For all $t \in [T]$, we obtain that
\begin{align*}
    \expect{\norm{\dev{t+1}}^2} & \leq 
      \beta^2 (1 + 4L \gamma)(1 + \frac{9}{8} L \gamma) \expect{\norm{\dev{t}}^2} + \frac{3}{4} \beta^2 L \gamma (1 + 4L \gamma) \expect{\norm{\nabla \avgloss \left( \avgmodel{t} \right)}^2}\\
    &+9\beta^2L^2 \left( 1 + \frac{1}{4\gamma L }\right) \left(\expect{\diam{\model{}{}}{t+1}} + \expect{\diam{\model{}{}}{t}}+\expect{\norm{\avgmodel{t+1}-\avgmodel{t+1/2}}^2}\right) \\
    &+ \frac{9}{4} \beta^2 L \gamma \left(1 + 4L \gamma \right)  L^2\expect{\diam{\model{}{}}{t}} + \frac{(1-\beta)^2 \sigma^2}{n-f}.
\end{align*}
\end{replemma}
}}

\begin{proof}
   Consider an arbitrary step $t\geq1$. Recall from Algorithm \ref{algo} that
   \begin{align*}
       \AvgMmt{\iteration+1} := \mmtcoefficient \AvgMmt{\iteration} + (1 - \mmtcoefficient) \AvgNoisyGrad{\iteration+1}.
   \end{align*}
   Combining this with \eqref{eq:deviation_definition} we obtain that
   \begin{equation*}
       \dev{t+1} = \beta \AvgMmt{t} + (1-\beta) \AvgNoisyGrad{t+1}  -\AvgGrad{t+1}.
   \end{equation*}
   Adding and subtracting $\beta \AvgGrad{t}$ and $\beta \AvgGrad{t+1}$ on the R.H.S.~we then obtain that
   \begin{align*}
       \dev{t+1} = \beta \left(\AvgMmt{t}- \AvgGrad{t} \right) + \beta \left( \AvgGrad{t}-\AvgGrad{t+1} \right) + (1-\beta) \left(\AvgNoisyGrad{t+1}-\AvgGrad{t+1} \right),
   \end{align*}
   which yields
   \begin{align*}
       \norm{\dev{t+1}}^2 &= \beta^2 \norm{\AvgMmt{t}- \AvgGrad{t}}^2 + \beta^2 \norm{\AvgGrad{t}-\AvgGrad{t+1}}^2 + (1-\beta)^2 \norm{\AvgNoisyGrad{t+1}-\AvgGrad{t+1}}^2\\ &+ 2 \beta^2 \iprod{\AvgMmt{t}- \AvgGrad{t}}{\AvgGrad{t}-\AvgGrad{t+1}}
       + 2 \beta (1-\beta) \iprod{\AvgNoisyGrad{t+1}-\AvgGrad{t+1}}{\AvgMmt{t}- \AvgGrad{t}}\\ &+ 2 \beta (1-\beta) \iprod{\AvgGrad{t}-\AvgGrad{t+1}}{\AvgNoisyGrad{t+1}-\AvgGrad{t+1}}.
   \end{align*}
   Applying the conditional expectation $\condexpect{t+1}{\cdot}$ on both sides, and noting that $\AvgMmt{t}$, $\AvgGrad{t}$, and $\AvgGrad{t+1}$ are deterministic values when given $\P_{t+1}$, we obtain that
   \begin{align*}
       \condexpect{t+1}{\norm{\dev{t+1}}^2} &= \beta^2 \condexpect{t+1}{\norm{\dev{t}}^2} + \beta^2 \norm{\AvgGrad{t}-\AvgGrad{t+1}}^2 + (1-\beta)^2 \condexpect{t+1}{\norm{\AvgNoisyGrad{t+1}-\AvgGrad{t+1}}^2}\\ &+2\beta^2 \iprod{\dev{t}}{\AvgGrad{t}-\AvgGrad{t+1}}
       + 2 \beta (1-\beta) \iprod{\condexpect{t+1}{\AvgNoisyGrad{t+1}-\AvgGrad{t+1}}}{\AvgMmt{t}- \AvgGrad{t}}\\
       &+ 2 \beta (1-\beta) \iprod{\AvgGrad{t}-\AvgGrad{t+1}}{\condexpect{t+1}{\AvgNoisyGrad{t+1}-\AvgGrad{t+1}}}.
   \end{align*}
   Owing to Assumption~\ref{asp:bnd_var}, $\condexpect{t+1}{\gradient{i}{t+1}} = \localgrad{i}\left(\model{i}{t+1}\right)$ for all $i \in \H$. Thus, we have $\condexpect{t+1}{\AvgNoisyGrad{t+1}-\AvgGrad{t+1}} = 0$. Therefore, 
   \begin{align}
   \label{eq:deviation_bound}
       &\condexpect{t+1}{\norm{\dev{t+1}}^2}=  \beta^2 {\norm{\dev{t}}^2} + \beta^2 \norm{\AvgGrad{t}-\AvgGrad{t+1}}^2 + (1-\beta)^2 \condexpect{t+1}{\norm{\AvgNoisyGrad{t+1}-\AvgGrad{t+1}}^2} +2\beta^2 \iprod{\dev{t}}{\AvgGrad{t}-\AvgGrad{t+1}}.
   \end{align}
     Note that
   \begin{align}\nonumber
       (1-\beta)^2 \condexpect{t+1}{\norm{\AvgNoisyGrad{t+1}-\AvgGrad{t+1}}^2} & = (1-\beta)^2 \condexpect{t+1}{\norm{\frac{1}{n-f}\sum_{i \in \H} \left(\gradient{i}{t+1} - \localgrad{i}\left(\model{i}{t+1}\right) \right)}^2} \nonumber \\
       &= \frac{(1-\beta)^2}{(n-f)^2} \sum_{i,j \in \H} \condexpect{t+1}{\iprod{\gradient{i}{t+1} -\localgrad{i}\left(\model{i}{t+1}\right)}{\gradient{j}{t+1} - \localgrad{j}\left(\model{j}{t+1}\right)}} \nonumber \\
       &\stackrel{(a)}{=} \frac{(1-\beta)^2}{(n-f)^2} \sum_{i \in \H} \condexpect{t+1}{\norm{\gradient{i}{t+1} - \localgrad{i}\left(\model{i}{t+1}\right)}^2} \stackrel{(b)}{\leq} \frac{(1-\beta)^2 \sigma^2}{n-f},
       \label{eq:deviation_third_term}
   \end{align}  
   where (a) uses the facts that the gradient estimations are independent and $\condexpect{t+1}{\gradient{i}{t+1}} - \localgrad{i}\left(\model{i}{t+1}\right) = 0$, and (b) is due to Assumption~\ref{asp:bnd_var}.
   Substituting from~\eqref{eq:deviation_third_term} in~\eqref{eq:deviation_bound} we obtain that (upon applying Cauchy-Schwartz inequality)
   \begin{align*}
       \condexpect{t+1}{\norm{\dev{t+1}}^2} &\leq \nonumber \beta^2 {\norm{\dev{t}}^2} + \beta^2 \norm{\AvgGrad{t}-\AvgGrad{t+1}}^2 +2\beta^2 \iprod{\dev{t}}{\AvgGrad{t}-\AvgGrad{t+1}} + \frac{(1-\beta)^2 \sigma^2}{n-f}\\
       &\leq \beta^2 {\norm{\dev{t}}^2} + \beta^2 \norm{\AvgGrad{t}-\AvgGrad{t+1}}^2 +2\beta^2 \norm{\dev{t}}\norm{\AvgGrad{t}-\AvgGrad{t+1}} + \frac{(1-\beta)^2 \sigma^2}{n-f}.
   \end{align*}
Upon taking total expectation on both sides above we obtain that
\begin{align*}
    \expect{\norm{\dev{t+1}}^2} \leq \beta^2 \expect{\norm{\dev{t}}^2} + \beta^2 \expect{\norm{\AvgGrad{t}-\AvgGrad{t+1}}^2} +2\beta^2 \expect{\norm{\dev{t}}\norm{\AvgGrad{t}-\AvgGrad{t+1}}} + \frac{(1-\beta)^2 \sigma^2}{n-f}.
\end{align*}
As $2xy \leq cx^2+\frac{y^2}{c}$ for all $c > 0$, by substituting $c = 4\gamma L$ we obtain that $2 \norm{\dev{t}}\norm{\AvgGrad{t}-\AvgGrad{t+1}} \leq 4\gamma L  \norm{\dev{t}}^2 + \frac{1}{4\gamma L } \norm{\AvgGrad{t}-\AvgGrad{t+1}}^2$. 
Using this above we obtain that
\begin{align}
    \expect{\norm{\dev{t+1}}^2} \leq \beta^2 (1 + 4\gamma L) \expect{\norm{\dev{t}}^2} + \beta^2 \left( 1 + \frac{1}{4\gamma L }\right) \expect{\norm{\AvgGrad{t}-\AvgGrad{t+1}}^2} + \frac{(1-\beta)^2 \sigma^2}{n-f}.
       \label{eq:temp1}
\end{align}
Note that as $\AvgGrad{t} \coloneqq \frac{1}{n-f} \sum_{i \in \H} \localgrad{i}\left(\model{i}{t}\right)$, by Assumption~\ref{asp:lip} we obtain that
\begin{align*}
    \norm{\AvgGrad{t}-\AvgGrad{t+1}} &= \frac{1}{n-f} \sum_{i \in \H} \norm{\localgrad{i}\left(\model{i}{t}\right)-\localgrad{i}\left(\model{i}{t+1}\right)} \leq \frac{L}{n-f} \sum_{i \in \H} \norm{\model{i}{t}-\model{i}{t+1}}.
\end{align*}
This implies that
\begin{align}
    \norm{\AvgGrad{t}-\AvgGrad{t+1}}^2 \leq L^2 \left( \frac{1}{n-f} \sum_{i \in \H} \norm{\model{i}{t}-\model{i}{t+1}} \right)^2 \leq \frac{L^2}{n-f} \sum_{i \in \H} \norm{\model{i}{t}-\model{i}{t+1}}^2. \label{eqn:lip_dev_1}
\end{align}
By triangle inequality we obtain that
\begin{align*}
    \norm{\model{i}{t}-\model{i}{t+1}} \leq  \norm{\model{i}{t+1}-\avgmodel{t+1}}+ \norm{\avgmodel{t+1}-\avgmodel{t+1/2}} + \norm{\avgmodel{t+1/2}-\model{i}{t}}.
\end{align*}
Recall that $\avgmodel{t+1/2} = \avgmodel{t} - \gamma \AvgMmt{t}$. Thus, $\norm{\avgmodel{t+1/2}-\model{i}{t}} \leq \norm{\avgmodel{t} - \model{i}{t}} + \gamma \norm{\AvgMmt{t}}$ and
\begin{align*}
    \norm{\model{i}{t}-\model{i}{t+1}} \leq  \norm{\model{i}{t+1}-\avgmodel{t+1}}+ \norm{\avgmodel{t+1}-\avgmodel{t+1/2}} + \norm{\avgmodel{t} - \model{i}{t}} + \gamma \norm{\AvgMmt{t}}.
\end{align*}
As $(x+y)^2 \leq (1+c)x^2 + (1+\frac{1}{c})y^2$, taking square on both sides for $c = 2$ we obtain that
\begin{align*}
    \norm{\model{i}{t}-\model{i}{t+1}}^2 \leq  3\left( \norm{\model{i}{t+1}-\avgmodel{t+1}} +  \norm{\avgmodel{t+1}-\avgmodel{t+1/2}} + \norm{\avgmodel{t} - \model{i}{t}} \right)^2+ \frac{3}{2} \gamma^2 \norm{\AvgMmt{t}}^2.
\end{align*}
And thus
\begin{align}
    \norm{\model{i}{t}-\model{i}{t+1}}^2 \leq  9 \norm{\model{i}{t+1}-\avgmodel{t+1}}^2 +9  \norm{\avgmodel{t+1}-\avgmodel{t+1/2}}^2 + 9\norm{\avgmodel{t} - \model{i}{t}}^2+ \frac{3}{2} \gamma^2 \norm{\AvgMmt{t}}^2. \label{eqn:para_growth_1}
\end{align}
Note that for any $t$, by definition of $\avgmodel{t}$ we have
\begin{align*}
    \norm{\avgmodel{t} - \model{i}{t}}^2 = \norm{\frac{1}{n-f} \sum_{j \in \H} \left(\model{j}{t} - \model{i}{t} \right)}^2 \leq \left( \frac{1}{n-f} \sum_{j \in \H} \norm{ \model{j}{t} - \model{i}{t} } \right)^2 \leq \frac{1}{n-f} \sum_{j \in \H} \norm{\model{j}{t} - \model{i}{t}}^2.
\end{align*}
Substituting from above in~\eqref{eqn:para_growth_1} we obtain that
\begin{align*}
    \norm{\model{i}{t}-\model{i}{t+1}}^2 \leq  \frac{9}{n-f} \left(\sum_{j \in \H} \norm{\model{j}{t+1} - \model{i}{t+1}}^2 + \sum_{j \in \H} \norm{\model{j}{t} - \model{i}{t}}^2\right) + 9 \norm{\avgmodel{t+1}-\avgmodel{t+1/2}}^2 + \frac{3}{2} \gamma^2 \norm{\AvgMmt{t}}^2. 
\end{align*}
Substituting from above in~\eqref{eqn:lip_dev_1} we obtain that
\begin{align*}
    \norm{\AvgGrad{t}-\AvgGrad{t+1}}^2 & \leq \frac{9 L^2}{(n-f)^2} \left(\sum_{i, \, j \in \H} \norm{\model{j}{t+1} - \model{i}{t+1}}^2 + \sum_{i, \, j \in \H} \norm{\model{j}{t} - \model{i}{t}}^2\right)\\ 
    &+ \frac{9 L^2 }{n-f} \sum_{i \in \H} \norm{\avgmodel{t+1}-\avgmodel{t+1/2}}^2 + \frac{3}{2}\frac{L^2 \gamma^2}{n - f } \sum_{i \in \H} \norm{\AvgMmt{t}}^2.
\end{align*}
Taking total expectation on both sides above, and using the notation $\diam{*}{t}$, we obtain that
\begin{align*}
    \expect{ \norm{\AvgGrad{t}-\AvgGrad{t+1}}^2 } \leq 9 L^2 \left(\expect{\diam{\model{}{}}{t+1}} + \expect{\diam{\model{}{}}{t}}\right) + 9 L^2 \expect{\norm{\avgmodel{t+1}-\avgmodel{t+1/2}}^2} + \frac{3}{2}L^2 \gamma^2 \expect{\norm{\AvgMmt{t}}^2}.
\end{align*}
By definition of $e_t$ we have $\norm{\AvgMmt{t}} \leq \norm{\AvgMmt{t}-\AvgGrad{t}}+\norm{\AvgGrad{t}} = \norm{\dev{t}}+\norm{\AvgGrad{t}}$. Therefore, $\norm{\AvgMmt{t}}^2 \leq 3 \norm{\dev{t}}^2 + \frac{3}{2} \norm{\AvgGrad{t}}^2$. Using this above we obtain that
\begin{align}\nonumber
    \expect{ \norm{\AvgGrad{t}-\AvgGrad{t+1}}^2 } &\leq 9 L^2 \left(\expect{\diam{\model{}{}}{t+1}} + \expect{\diam{\model{}{}}{t}}+\expect{\norm{\avgmodel{t+1}-\avgmodel{t+1/2}}^2}\right) \\
    &+\frac{3}{2}L^2 \gamma^2 \left( 3 \expect{\norm{\dev{t}}^2} + \frac{3}{2} \expect{\norm{\AvgGrad{t}}^2} \right). \label{eqn:lip_dev_2}
\end{align}
Substituting from~\eqref{eqn:lip_dev_2} above in~\eqref{eq:temp1} we obtain that
\begin{align*}
    \expect{\norm{\dev{t+1}}^2} & \leq \beta^2 (1 + 4\gamma L ) \expect{\norm{\dev{t}}^2} + \beta^2 \left( 1 + \frac{1}{4\gamma L }\right) \frac{9}{4}L^2 \gamma^2 \left(  2\expect{\norm{\dev{t}}^2} + \expect{\norm{\AvgGrad{t}}^2} \right)\\
    &+ \beta^2 \left( 1 + \frac{1}{4\gamma L }\right)  9 L^2 \left(\expect{\diam{\model{}{}}{t+1}} + \expect{\diam{\model{}{}}{t}}+\expect{\norm{\avgmodel{t+1}-\avgmodel{t+1/2}}^2}\right) + \frac{(1-\beta)^2 \sigma^2}{n-f} \\
    & = \beta^2 (1 + 4L \gamma)(1 + \frac{9}{8} L \gamma) \expect{\norm{\dev{t}}^2} + \frac{9}{16} \beta^2 L \gamma (1 + 4L \gamma) \expect{\norm{\AvgGrad{t}}^2}\\
    &+ \beta^2 \left( 1 + \frac{1}{4\gamma L }\right)  9 L^2 \left(\expect{\diam{\model{}{}}{t+1}} + \expect{\diam{\model{}{}}{t}}+\expect{\norm{\avgmodel{t+1}-\avgmodel{t+1/2}}^2}\right) + \frac{(1-\beta)^2 \sigma^2}{n-f}.
\end{align*}
Now note also that ${\norm{\AvgGrad{t}}^2} \leq 4 \norm{\AvgGrad{t}-\nabla \avgloss \left( \avgmodel{t} \right)}^2 + \frac{4}{3}\norm{\nabla \avgloss \left( \avgmodel{t} \right)}^2$; thus
\begin{align}\nonumber
    \expect{\norm{\dev{t+1}}^2} & \leq 
      \beta^2 (1 + 4L \gamma)(1 + \frac{9}{8} L \gamma) \expect{\norm{\dev{t}}^2} + \frac{3}{4} \beta^2 L \gamma (1 + 4L \gamma) \expect{\norm{\nabla \avgloss \left( \avgmodel{t} \right)}^2}\\\nonumber
    &+\beta^2 \left( 1 + \frac{1}{4\gamma L }\right)  9 L^2 \left(\expect{\diam{\model{}{}}{t+1}} + \expect{\diam{\model{}{}}{t}}+\expect{\norm{\avgmodel{t+1}-\avgmodel{t+1/2}}^2}\right) + \frac{(1-\beta)^2 \sigma^2}{n-f}\\
    &+ \frac{9}{4} \beta^2 L \gamma \left(1 + 4L \gamma \right) \expect{\norm{\AvgGrad{t}-\nabla \avgloss \left( \avgmodel{t} \right)}^2}.
    \label{eq:dev_bound_final}
\end{align}
But now note that 
\begin{align*}
    \expect{\norm{\AvgGrad{t}-\nabla \avgloss \left( \avgmodel{t} \right)}^2} &= \expect{\norm{\frac{1}{n-f} \sum_{i \in \H} \localgrad{i}(\model{i}{t}) - \frac{1}{n-f} \sum_{i \in \H} \localgrad{i}\left(\avgmodel{t}\right)}^2}\\
    &= \expect{\norm{\frac{1}{n-f} \sum_{i \in \H} \left(\localgrad{i}(\model{i}{t}) - \localgrad{i}\left(\avgmodel{t}\right)\right)}^2} \\
    &\leq \frac{1}{n-f} \sum_{i \in \H} \expect{\norm{\localgrad{i}(\model{i}{t}) - \localgrad{i}\left(\avgmodel{t}\right)}^2}\\
    &\leq \frac{L^2}{n-f} \sum_{i \in \H} \expect{\norm{\model{i}{t} - \avgmodel{t}}^2} \leq  L^2\expect{\diam{\model{}{}}{t}}.
\end{align*}
Combining this with \eqref{eq:dev_bound_final}, we obtain
\begin{align*}
    \expect{\norm{\dev{t+1}}^2} & \leq 
      \beta^2 (1 + 4L \gamma)(1 + \frac{9}{8} L \gamma) \expect{\norm{\dev{t}}^2} + \frac{3}{4} \beta^2 L \gamma (1 + 4L \gamma) \expect{\norm{\nabla \avgloss \left( \avgmodel{t} \right)}^2}\\
    &+9\beta^2L^2 \left( 1 + \frac{1}{4\gamma L }\right) \left(\expect{\diam{\model{}{}}{t+1}} + \expect{\diam{\model{}{}}{t}}+\expect{\norm{\avgmodel{t+1}-\avgmodel{t+1/2}}^2}\right) \\
    &+ \frac{9}{4} \beta^2 L \gamma \left(1 + 4L \gamma \right)  L^2\expect{\diam{\model{}{}}{t}} + \frac{(1-\beta)^2 \sigma^2}{n-f},
\end{align*}
which is the lemma.
\end{proof}

\subsection{Proof of Lemma~\ref{lem:growth}}
\label{app:lem_growth}

We recall the lemma below. Also, recall that $\avgmodel{t} \coloneqq \nicefrac{1}{n-f} \sum_{i \in \H} \model{i}{t}$ and $\avgloss(\model{}{}) = \nicefrac{1}{n-f} \sum_{i \in \H} \localloss{i}(\model{}{}) $.

\noindent \fcolorbox{black}{white}{
\parbox{0.97\textwidth}{\centering
\begin{replemma}{lem:growth}
Suppose that assumptions~\ref{asp:lip} and~\ref{asp:bnd_var} hold true. Consider Algorithm~\ref{algo} with $\gamma \leq 1/L$. For each $t \in [T]$, we obtain that
\begin{align*}
     \expect{\avgloss(\avgmodel{t+1})- \avgloss(\avgmodel{t})} &\leq - \frac{\gamma}{2}\expect{\norm{\nabla \avgloss \left( \avgmodel{t} \right)}^2} + \frac{3\gamma}{2} \expect{\norm{\dev{t}}^2}\\
   &+ \frac{3}{2\gamma} \expect{\norm{\avgmodel{t+\nicefrac{1}{2}}-\avgmodel{t+1}}^2} 
    +   \frac{3\gamma}{2}  L^2\expect{\diam{\model{}{}}{t}}.
\end{align*}
\end{replemma}
}}

\begin{proof}
Consider an arbitrary $t \in [T]$. We define $\effectgrad{t}:=\frac{\avgmodel{t}-\avgmodel{t+1}}{\gamma}$, the step taken by the average of local models at iteration $t$. By the smoothness of the loss function (Assumption~\ref{asp:lip}), we have 
\begin{align*}
    \avgloss(\avgmodel{t+1})- \avgloss(\avgmodel{t}) &\leq \iprod{\avgmodel{t+1} - \avgmodel{t}}{\nabla \avgloss\left( \avgmodel{t} \right)} + \frac{L}{2} \norm{\avgmodel{t+1} - \avgmodel{t}}^2 \nonumber \\
    & = - \gamma \iprod{\effectgrad{t}}{ \nabla \avgloss \left( \avgmodel{t} \right)} + \frac{L\gamma^2}{2} \norm{\effectgrad{t}}^2.
\end{align*}
Using the fact that $\gamma \leq 1 / L$, we obtain that
\begin{align*}
    \avgloss(\avgmodel{t+1})- \avgloss(\avgmodel{t}) &\leq - \gamma \iprod{\effectgrad{t}}{ \nabla \avgloss \left( \avgmodel{t} \right)} + \frac{\gamma}{2} \norm{\effectgrad{t}}^2. 
\end{align*}
Now note that $-\iprod{x}{y}+\frac{\norm{x}^2}{2} = -\frac{\norm{y}^2}{2}+\frac{\norm{x-y}^2}{2}$; thus
\begin{align*}
    \avgloss(\avgmodel{t+1})- \avgloss(\avgmodel{t}) &\leq - \frac{\gamma}{2}\norm{\nabla \avgloss \left( \avgmodel{t} \right)}^2 + \frac{\gamma}{2} \norm{\effectgrad{t}-\nabla \avgloss \left( \avgmodel{t} \right)}^2\\
    &\leq - \frac{\gamma}{2}\norm{\nabla \avgloss \left( \avgmodel{t} \right)}^2 + \frac{\gamma}{2} \norm{\frac{\avgmodel{t}-\avgmodel{t+1}}{\gamma}-\nabla \avgloss \left( \avgmodel{t} \right)}^2\\
    &=- \frac{\gamma}{2}\norm{\nabla \avgloss \left( \avgmodel{t} \right)}^2 + \frac{\gamma}{2} \norm{\frac{\avgmodel{t}-\avgmodel{t+\nicefrac{1}{2}}+\avgmodel{t+\nicefrac{1}{2}}-\avgmodel{t+1}}{\gamma} - \AvgGrad{t} + \AvgGrad{t}-\nabla \avgloss \left( \avgmodel{t} \right)}^2\\
    &= - \frac{\gamma}{2}\norm{\nabla \avgloss \left( \avgmodel{t} \right)}^2 + \frac{\gamma}{2} \norm{\AvgMmt{t}+\frac{\avgmodel{t+\nicefrac{1}{2}}-\avgmodel{t+1}}{\gamma} - \AvgGrad{t} + \AvgGrad{t}-\nabla \avgloss \left( \avgmodel{t} \right)}^2\\
    &\leq - \frac{\gamma}{2}\norm{\nabla \avgloss \left( \avgmodel{t} \right)}^2 + \frac{3\gamma}{2} \norm{\AvgMmt{t}- \AvgGrad{t}}^2 + \frac{3}{2\gamma}\norm{\avgmodel{t+\nicefrac{1}{2}}-\avgmodel{t+1}}^2 +   \frac{3\gamma}{2} \norm{\AvgGrad{t}-\nabla \avgloss \left( \avgmodel{t} \right)}^2,
\end{align*}
where  $\AvgGrad{t} = \frac{1}{n-f} \sum_{i \in \H} \localgrad{i}\left(\model{i}{t}\right)$.
Now recall from Lemma \ref{lem:dev} that we define $\dev{t} = \AvgMmt{t} - \AvgGrad{t}$. 
Thus, taking the expectation from both sides of the above, we obtain that
\begin{align}\nonumber
    \expect{\avgloss(\avgmodel{t+1})- \avgloss(\avgmodel{t})} &\leq - \frac{\gamma}{2}\expect{\norm{\nabla \avgloss \left( \avgmodel{t} \right)}^2} + \frac{3\gamma}{2} \expect{\norm{\dev{t}}^2} + \frac{3}{2\gamma} \expect{\norm{\avgmodel{t+\nicefrac{1}{2}}-\avgmodel{t+1}}^2} \\
    &+   \frac{3\gamma}{2} \expect{ \norm{\AvgGrad{t}-\nabla \avgloss \left( \avgmodel{t} \right)}^2}.
    \label{eq:growth_bound}
\end{align}
Now note that 
\begin{align*}
    \expect{\norm{\AvgGrad{t}-\nabla \avgloss \left( \avgmodel{t} \right)}^2} &= \expect{\norm{\frac{1}{n-f} \sum_{i \in \H} \localgrad{i}(\model{i}{t}) - \frac{1}{n-f} \sum_{i \in \H} \localgrad{i}\left(\avgmodel{t}\right)}^2}\\
    &= \expect{\norm{\frac{1}{n-f} \sum_{i \in \H} \left(\localgrad{i}(\model{i}{t}) - \localgrad{i}\left(\avgmodel{t}\right)\right)}^2} \\
    &\leq \frac{1}{n-f} \sum_{i \in \H} \expect{\norm{\localgrad{i}(\model{i}{t}) - \localgrad{i}\left(\avgmodel{t}\right)}^2}\\
    &\leq \frac{L^2}{n-f} \sum_{i \in \H} \expect{\norm{\model{i}{t} - \avgmodel{t}}^2} \leq  L^2\expect{\diam{\model{}{}}{t}}.
\end{align*}
Combining this with \eqref{eq:growth_bound}, we obtain that
\begin{align*}
   \expect{\avgloss(\avgmodel{t+1})- \avgloss(\avgmodel{t})} &\leq - \frac{\gamma}{2}\expect{\norm{\nabla \avgloss \left( \avgmodel{t} \right)}^2} + \frac{3\gamma}{2} \expect{\norm{\dev{t}}^2}\\
   &+ \frac{3}{2\gamma} \expect{\norm{\avgmodel{t+\nicefrac{1}{2}}-\avgmodel{t+1}}^2} 
    +   \frac{3\gamma}{2}  L^2\expect{\diam{\model{}{}}{t}}.
\end{align*}
which is the lemma.
\end{proof}

\subsection{Proof of Lemma \ref{lemma:cvaboundfivef}}
\label{sec:proofoffivef}
In this section, we prove that if $n > 5f$, then $\numbercvarounds \in \mathcal{O}(\log(n))$ coordination rounds is enough to guarantee $(\alpha,\lambda)$-\reduction{}.
\begin{lemma}
\label{lemma:max_contraction}
Consider the coordination phase of Algorithm \ref{algo}.  Suppose that there exists $\cvabyzbound > 0$ such that $n \geq (5+\cvabyzbound)f$. For any $k\geq 1$ we have
\begin{align*}
    \max_{i,j \in \H}\norm{\cvavector{i}{k}-\cvavector{j}{k}}\leq \left(\frac{3}{n-2f}\right)^k \max_{i,j \in \H}\norm{\cvavector{i}{0}-\cvavector{j}{0}}.
\end{align*}
\end{lemma}
\begin{proof}
  Consider an arbitrary round $k \in [\numbercvarounds]$. Consider two correct nodes $p,q \in \H$. We then have
\begin{align}\nonumber
    \norm{\cvavector{p}{k}-\cvavector{q}{k}} &= \norm{\frac{1}{n-2f} \sum_{j \in \filter{p}{k}} \cvavector{j}{k-1}-\frac{1}{n-2f} \sum_{j \in \filter{q}{k}} \cvavector{j}{k-1}} \\ 
    &= \frac{1}{n-2f} \norm{ \sum_{j \in \filter{p}{k} \setminus \filter{q}{k}} \cvavector{j}{k-1}- \sum_{j \in \filter{q}{k} \setminus \filter{p}{k}} \cvavector{j}{k-1}}.
    \label{eq:cva_5f_temp}
\end{align}
Now similar to Lemma \ref{lemma:cva_diameter}, we define:
\begin{equation*}
    F_p:= \left\{i: i \notin \H , i \in \filter{p}{k}, i \notin \filter{q}{k} \right\}.
\end{equation*}
\begin{equation*}
    F_q:= \left\{i: i \notin \H , i \in \filter{q}{k}, i \notin \filter{p}{k} \right\}.
\end{equation*}
\begin{equation*}
    \H_p:= \left\{i: i \in \H , i \in \filter{p}{k}, i \notin \filter{q}{k} \right\}.
\end{equation*}
\begin{equation*}
    \H_q:= \left\{i: i \in \H , i \in \filter{q}{k}, i \notin \filter{p}{k} \right\}.
\end{equation*}
We also define $f_p:=\card{F_p}$ and $f_q:= \card{F_q}$. We also order these four sets such that, e.g,. $F_p[i]$ refers to a unique element in set $F_p$. Without loss of generality, we assume $\card{\H_p} \geq f_q$ and $\card{\H_q} \geq f_p$ \footnote{Otherwise we add sufficiently many correct vectors from $\filter{p}{k}\cap \filter{q}{k}$ to both $\H_p$ and $\H_q$ such that $\card{\H_p} \geq f_q$ and $\card{\H_q} \geq f_p$.}. Now from \eqref{eq:cva_5f_temp}, we obtain that 
\begin{align*}
    &(n-2f)\norm{\cvavector{p}{k}-\cvavector{q}{k}} =  \norm{ 
    \sum_{j \in F_p} \cvavector{j}{k-1} 
    +\sum_{j \in \H_p} \cvavector{j}{k-1}
    - \sum_{j \in F_q} \cvavector{j}{k-1} 
    -\sum_{j \in \H_q} \cvavector{j}{k-1}}\\
    &= \norm{ 
    \sum_{j \in [f_p]} \left( \cvavector{F_p[j]}{k-1} - \cvavector{\H_q[j]}{k-1}\right) - \sum_{j \in [f_q]} \left( \cvavector{F_q[j]}{k-1} - \cvavector{\H_p[j]}{k-1}\right)
    +\sum_{j \in [\card{\H_p}-f_q]} \left( \cvavector{\H_p[f_q+j]}{k-1} - \cvavector{\H_q[f_p+j]}{k-1}\right)}.
\end{align*}
By triangle inequality, we then have
\begin{align*}
    (n-2f)\norm{\cvavector{p}{k}-\cvavector{q}{k}} &\leq \sum_{j \in [f_p]} \norm{ \cvavector{F_p[j]}{k-1} - \cvavector{\H_q[j]}{k-1}} + \sum_{j \in [f_q]} \norm{ \cvavector{F_q[j]}{k-1} - \cvavector{\H_p[j]}{k-1}}\\
    &+\sum_{j \in [\card{\H_p}-f_q]} \norm{ \cvavector{\H_p[f_q+j]}{k-1} - \cvavector{\H_q[f_p+j]}{k-1}}\\
    &\leq \sum_{j \in [f_p]} \norm{ \cvavector{F_p[j]}{k-1} - \cvavector{p}{k-1}} + \norm{ \cvavector{p}{k-1} - \cvavector{\H_q[j]}{k-1}}\\
    &+ \sum_{j \in [f_q]} \norm{ \cvavector{F_q[j]}{k-1} - \cvavector{q}{k-1}} + \norm{ \cvavector{q}{k-1} - \cvavector{\H_p[j]}{k-1}}\\
    &+\sum_{j \in [\card{\H_p}-f_q]} \norm{ \cvavector{\H_p[f_q+j]}{k-1} - \cvavector{\H_q[f_p+j]}{k-1}}.
\end{align*}
Now note that for each \byzantine{} node $j^* \in \filter{p}{k}$, there is at least one correct vector received by node $p$ and filtered out by the \cva{} function. Therefore, we must have $\norm{\cvavector{j^*}{k-1}-\cvavector{p}{k-1}} \leq \max_{i,j \in \H}\norm{\cvavector{i}{k-1}-\cvavector{j}{k-1}}$. Thus,
\begin{align*}
    (n-2f)\norm{\cvavector{p}{k}-\cvavector{q}{k}} 
    &\leq \sum_{j \in [f_p]} \max_{i,j \in \H}\norm{\cvavector{i}{k-1}-\cvavector{j}{k-1}} + \max_{i,j \in \H}\norm{\cvavector{i}{k-1}-\cvavector{j}{k-1}}\\
    &+ \sum_{j \in [f_q]} \max_{i,j \in \H}\norm{\cvavector{i}{k-1}-\cvavector{j}{k-1}} + \max_{i,j \in \H}\norm{\cvavector{i}{k-1}-\cvavector{j}{k-1}}\\
    &+\sum_{j \in [\card{\H_p}-f_q]} \max_{i,j \in \H}\norm{\cvavector{i}{k-1}-\cvavector{j}{k-1}}\\
    &\leq \left(2(f_p+f_q)+(\card{\H_p}-f_q)\right) \max_{i,j \in \H}\norm{\cvavector{i}{k-1}-\cvavector{j}{k-1}}\\
    &= \left(2(f_p+f_q)+(\card{\filter{p}{k} \setminus \filter{q}{k}}-f_p-f_q)\right) \max_{i,j \in \H}\norm{\cvavector{i}{k-1}-\cvavector{j}{k-1}}.
\end{align*}
Now recall from Lemma \ref{lemma:set_size} that $\card{\filter{q}{k} \setminus \filter{p}{k}} \leq 2f$. Therefore,
\begin{align*}
    (n-2f)\norm{\cvavector{p}{k}-\cvavector{q}{k}} &\leq (f_p+f_q+2f)  \max_{i,j \in \H}\norm{\cvavector{i}{k-1}-\cvavector{j}{k-1}} \\
    &\leq 3f \max_{i,j \in \H}\norm{\cvavector{i}{k-1}-\cvavector{j}{k-1}}.
\end{align*}
Equivalently,
\begin{align*}
    \norm{\cvavector{p}{k}-\cvavector{q}{k}} &\leq (f_p+f_q+2f)  \max_{i,j \in \H}\norm{\cvavector{i}{k-1}-\cvavector{j}{k-1}} \\
    &\leq \frac{3f}{n-2f} \max_{i,j \in \H}\norm{\cvavector{i}{k-1}-\cvavector{j}{k-1}}.
\end{align*}
As the above inequality holds for any choice of $p$ and $q$, we obtain that
\begin{align*}
    \max_{i,j \in \H}\norm{\cvavector{i}{k}-\cvavector{j}{k}} &\leq \frac{3f}{n-2f} \max_{i,j \in \H}\norm{\cvavector{i}{k-1}-\cvavector{j}{k-1}}.
\end{align*}
As the above holds for any $k\geq 1$, we obtain that
\begin{align*}
    \max_{i,j \in \H}\norm{\cvavector{i}{k}-\cvavector{j}{k}} &\leq \left(\frac{3f}{n-2f}\right)^k \max_{i,j \in \H}\norm{\cvavector{i}{0}-\cvavector{j}{0}}.
\end{align*}
This is the lemma.
\end{proof}
\begin{lemma}
\label{lemma:contraction5f}
Consider the coordination phase of Algorithm \ref{algo}. Assume there exists $\cvabyzbound>0$ such that $n \geq (5+\cvabyzbound)f$. For  $\numbercvarounds = \frac{\log(8(n-f))}{2\log(\frac{3+v}{v})} \in \mathcal{O}(\log(n))$, we have
\begin{align*}
    \diam{\cvavector{}{}}{\numbercvarounds} \leq \frac{2f}{n-f} \diam{\cvavector{}{}}{0}.
\end{align*}
\end{lemma}
\begin{proof}
For $k = \numbercvarounds$ in Lemma \ref{lemma:max_contraction}, we have
 \begin{align*}
    \max_{i,j \in \H}\norm{\cvavector{i}{\numbercvarounds}-\cvavector{j}{\numbercvarounds}} &\leq \left(\frac{3f}{n-2f}\right)^\numbercvarounds \max_{i,j \in \H}\norm{\cvavector{i}{0}-\cvavector{j}{0}}.
\end{align*}
Now squaring both sides, we obtain that 
\begin{align*}
    {\max_{i,j \in \H}\norm{\cvavector{i}{\numbercvarounds}-\cvavector{j}{\numbercvarounds}}^2} \leq \left(\frac{3f}{n-2f}  \right)^{2\numbercvarounds} {\max_{i,j \in \H}\norm{\cvavector{i}{0}-\cvavector{j}{0}}^2}.
\end{align*}
Now as we assume $n\geq (5+\cvabyzbound)f$, we have
\begin{equation*}
   \frac{3f}{n-2f}\leq\frac{3f}{(3+\cvabyzbound)f}= \frac{3}{3+\cvabyzbound} < 1.
\end{equation*}
We then have
\begin{align*}
    {\max_{i,j \in \H}\norm{\cvavector{i}{\numbercvarounds}-\cvavector{j}{\numbercvarounds}}^2} \leq \left(\frac{3}{3+\cvabyzbound}  \right)^{2\numbercvarounds-1} \left(\frac{3f}{n-2f}  \right) {\max_{i,j \in \H}\norm{\cvavector{i}{0}-\cvavector{j}{0}}^2}.
\end{align*}
Now note that as $n > 5f$, we have $4(n-2f)\geq 3 (n-f)$, thus, 
\begin{align}
    {\max_{i,j \in \H}\norm{\cvavector{i}{\numbercvarounds}-\cvavector{j}{\numbercvarounds}}^2} \leq \left(\frac{3}{3+\cvabyzbound}  \right)^{2\numbercvarounds-1} \left(\frac{4f}{n-f}  \right) {\max_{i,j \in \H}\norm{\cvavector{i}{0}-\cvavector{j}{0}}^2}.
    \label{eq:max_to_max}
\end{align}
Now note that 
\begin{align}\nonumber
     \diam{\cvavector{}{}}{\numbercvarounds} &= \frac{1}{(n-f)^2}\sum_{i, \, j \in \H}{\norm{\cvavector{i}{\numbercvarounds}-\cvavector{j}{\numbercvarounds}}^2}\\
     &\leq \frac{1}{(n-f)^2}\sum_{i, \, j \in \H}{\max_{i,j \in \H}\norm{\cvavector{i}{\numbercvarounds}-\cvavector{j}{\numbercvarounds}}^2} = {\max_{i,j \in \H}\norm{\cvavector{i}{\numbercvarounds}-\cvavector{j}{\numbercvarounds}}^2}.
     \label{eq:diam_to_max}
\end{align}
Note also that 
\begin{align}\nonumber
    {\max_{i,j \in \H}\norm{\cvavector{i}{0}-\cvavector{j}{0}}^2} &\leq 4{\max_{i \in \H}\norm{\cvavector{i}{0}-\cvaavg{0}}^2}\\\nonumber
    &\leq 4{\sum_{i \in \H}\norm{\cvavector{i}{0}-\cvaavg{0}}^2}\\
    &\leq 4(n-f) \frac{1}{n-f} {\sum_{i \in \H}\norm{\cvavector{i}{0}-\cvaavg{0}}^2} = 4(n-f) \diam{\cvavector{}{}}{0}.
    \label{eq:max_to_diam}
\end{align}
Combining \eqref{eq:max_to_max}, \eqref{eq:diam_to_max}, and \eqref{eq:max_to_diam}, we then obtain that
\begin{align*}
    \diam{\cvavector{}{}}{\numbercvarounds} \leq \left(\frac{3}{3+\cvabyzbound}  \right)^{2\numbercvarounds-1} \left(\frac{2f}{n-f}  \right) 8(n-f)\diam{\cvavector{}{}}{0}.
\end{align*}
Setting $\numbercvarounds = \ceil{\frac{\log(8(n-f))}{2\log(\frac{3+\cvabyzbound}{3})}}+1$, we then obtain that
\begin{align*}
    \diam{\cvavector{}{}}{\numbercvarounds} \leq \frac{2f}{n-f} \diam{\cvavector{}{}}{0}.
\end{align*}
This is what we wanted.
\end{proof}
We recall Lemma \ref{lemma:cvaboundfivef} below.

\noindent \fcolorbox{black}{white}{
\parbox{0.97\textwidth}{\centering
\begin{replemma}{lemma:cvaboundfivef}
Suppose that there exists $\cvabyzbound>0$ such that $n \geq (5+\cvabyzbound)f$. For  $\numbercvarounds = \frac{\log(8(n-f))}{2\log(\frac{3+\cvabyzbound}{\cvabyzbound})} \in \mathcal{O}(\log(n))$, the coordination phase of Algorithm~\ref{algo} guarantees $(\alpha,\lambda)$-\reduction{} for 
  \begin{align*}
     \alpha = \frac{2f}{n-f} \leq \frac{1}{2} \quad\text{and}\quad\lambda = \left( \frac{3+\cvabyzbound}{\cvabyzbound} \right)^2 \frac{(8f)^2}{n-f}.
  \end{align*}
\end{replemma}
}}

\begin{proof}
The first inequality is already proved by Lemma \ref{lemma:contraction5f}. Here we prove the second inequality.
  Consider an arbitrary round $k \in [\numbercvarounds]$. For a correct node $i\in\H$ we have
\begin{align*}
    &\norm{\cvavector{i}{k} - \cvaavg{k-1}} = \norm{\frac{1}{n-2f} \sum_{j\in\filter{i}{k}} \cvavector{j}{k-1} - \frac{1}{n-f} \sum_{j \in \H} \cvavector{j}{k-1}} \\
    &= \norm{\left(\frac{1}{n-2f} - \frac{1}{n-f} \right) \sum_{j\in\filter{i}{k} \cap \H} \cvavector{j}{k-1} + \frac{1}{n-2f} \sum_{j\in\filter{i}{k} \setminus \H} \cvavector{j}{k-1} - \frac{1}{n-f} \sum_{j\in \H \setminus\filter{i}{k}} \cvavector{j}{k-1} } \\
    &= \frac{1}{(n-f)(n-2f)} \norm{f \sum_{j\in\filter{i}{k} \cap \H} (\cvavector{j}{k-1}-\cvavector{i}{k-1} ) + (n-f)\sum_{j\in\filter{i}{k} \setminus \H} (\cvavector{j}{k-1}-\cvavector{i}{k-1} ) - (n-2f)\sum_{j\in \H \setminus \filter{i}{k}} (\cvavector{j}{k-1}-\cvavector{i}{k-1} )}.
\end{align*}
By the triangle inequality, we then have
\begin{align*}
    (n-f)(n-2f)\norm{\cvavector{i}{k} - \cvaavg{k-1}} &\leq f \sum_{j\in\filter{i}{k} \cap \H} \norm{\cvavector{j}{k-1}-\cvavector{i}{k-1}} + (n-f)\sum_{j\in\filter{i}{k} \setminus \H} \norm{\cvavector{j}{k-1}-\cvavector{i}{k-1}}\\
    &+ (n-2f)\sum_{j\in \H \setminus \filter{i}{k}} \norm{\cvavector{j}{k-1}-\cvavector{i}{k-1}}.
\end{align*}
Now note that for any \byzantine{} node $j^* \in \filter{i}{k}$ there is at least one correct $i^*$ such that $\norm{\cvavector{j^*}{k-1}-\cvavector{i}{k-1}} \leq \norm{\cvavector{i^*}{k-1}-\cvavector{i}{k-1}}$ (as otherwise $i^*$ would have been selected instead of $j^*$). Therefore, we must have $\norm{\cvavector{j^*}{k-1}-\cvavector{i}{k-1}} \leq \max_{p,q \in \H}\norm{\cvavector{p}{k-1}-\cvavector{q}{k-1}}$. And clearly for any correct node $i^*$ we have $\norm{\cvavector{i^*}{k-1}-\cvavector{i}{k-1}} \leq \max_{p,q \in \H}\norm{\cvavector{p}{k-1}-\cvavector{q}{k-1}}$. Therefore,
\begin{align*}
    \norm{\cvavector{i}{k} - \cvaavg{k-1}} &\leq \frac{f\card{\filter{i}{k} \cap \H}+(n-f)\card{\filter{i}{k} \setminus \H}+(n-2f)\card{\H \setminus \filter{i}{k}}}{(n-f)(n-2f)}\max_{p,q \in \H}\norm{\cvavector{p}{k-1}-\cvavector{q}{k-1}}.
\end{align*}
Now let $v := \card{\filter{i}{k} \cap \H}$. We then have $v = \card{\filter{i}{k}} + \card{ \H} - \card{\filter{i}{k} \cup \H} \geq n-2f+n-f-n = n -3f$. Also, $\card{\filter{i}{k} \setminus \H} = n-2f-v$ and $\card{\H \setminus \filter{i}{k}} = n -f - v$. Now we define
\begin{equation}
    A(v) ;= fv + (n-2f-v)(n-f) + (n-2f)(n-f-v) = 2 (n-2f)(n-f-v),
\end{equation}
which is decreasing in $v$. Then the maximum of $A(v)$ is reached for $v = n -3f$ and we have $A(v) \leq 4f(n-2f)$. Therefore, 
\begin{align*}
    \norm{\cvavector{i}{k} - \cvaavg{k-1}} &\leq \frac{4f}{n-f}\max_{p,q \in \H}\norm{\cvavector{p}{k-1}-\cvavector{q}{k-1}}.
\end{align*}
Also, note that
\begin{align*}
    \norm{\cvaavg{k} - \cvaavg{k-1}} &= \norm{\frac{1}{n-f}\sum_{i \in \H} \cvavector{i}{k} - \cvaavg{k-1}} \leq \frac{1}{n-f} \sum_{i \in \H} \norm{\cvavector{i}{k} - \cvaavg{k-1}}\\
    &\leq \frac{1}{n-f} \sum_{i \in \H} \frac{4f}{n-f}\max_{p,q \in \H}\norm{\cvavector{p}{k-1}-\cvavector{q}{k-1}} =  \frac{4f}{n-f}\max_{p,q \in \H}\norm{\cvavector{p}{k-1}-\cvavector{q}{k-1}}.
\end{align*}
Now applying Lemma \ref{lemma:max_contraction}, we obtain that 
\begin{align*}
    \norm{\cvaavg{k} - \cvaavg{k-1}} &\leq \left(\frac{3f}{n-2f}\right)^{k-1}\frac{4f}{n-f} \max_{i,j \in \H}\norm{\cvavector{i}{0}-\cvavector{j}{0}}
\end{align*}
Now as $n \geq (5+\cvabyzbound)f$, we obtain that
\begin{align*}
    \norm{\cvaavg{k} - \cvaavg{k-1}} &\leq \left(\frac{3}{3+\cvabyzbound}\right)^{k-1}\frac{4f}{n-f} \max_{i,j \in \H}\norm{\cvavector{i}{0}-\cvavector{j}{0}}
\end{align*}
Now note that by triangle inequality we have
\begin{align*}
    \norm{\cvaavg{\numbercvarounds} - \cvaavg{0}} &\leq \sum_{k \in [\numbercvarounds]} \norm{\cvaavg{k} - \cvaavg{k-1}} \leq \sum_{k \in [\numbercvarounds]}  \left(\frac{3}{3+\cvabyzbound}\right)^{k-1} \frac{4f}{n-f}\max_{i,j \in \H}\norm{\cvavector{i}{0}-\cvavector{j}{0}}\\
    &\leq \sum_{k = 1}^{\infty}  \left(\frac{3}{3+\cvabyzbound}\right)^{k-1} \frac{4f}{n-f}\max_{i,j \in \H}\norm{\cvavector{i}{0}-\cvavector{j}{0}} = \frac{1}{1-\frac{3}{3+\cvabyzbound}}  \frac{4f}{n-f}\max_{i,j \in \H}\norm{\cvavector{i}{0}-\cvavector{j}{0}}\\
    &= \frac{3+\cvabyzbound}{\cvabyzbound} \frac{4f}{n-f}\max_{i,j \in \H}\norm{\cvavector{i}{0}-\cvavector{j}{0}}.
\end{align*}
Squaring both sides, we then have
\begin{align}
    {\norm{\cvaavg{\numbercvarounds} - \cvaavg{0}}^2} &\leq \left( \frac{3+\cvabyzbound}{\cvabyzbound} \right)^2 \left(\frac{4f}{n-f}\right)^2 {\max_{i,j \in \H}\norm{\cvavector{i}{0}-\cvavector{j}{0}}^2}.
    \label{eq:dist_avg_max}
\end{align}
Now note that 
\begin{align}\nonumber
    {\max_{i,j \in \H}\norm{\cvavector{i}{0}-\cvavector{j}{0}}^2} &\leq 4{\max_{i \in \H}\norm{\cvavector{i}{0}-\cvaavg{0}}^2} \leq 4{\sum_{i \in \H}\norm{\cvavector{i}{0}-\cvaavg{0}}^2}\\
    &\leq 4(n-f) \frac{1}{n-f} {\sum_{i \in \H}\norm{\cvavector{i}{0}-\cvaavg{0}}^2} = 4(n-f) \diam{\cvavector{}{}}{0}.
\end{align}
Combining above with \eqref{eq:dist_avg_max}, we obtain that
\begin{align*}
    {\norm{\cvaavg{\numbercvarounds} - \cvaavg{0}}^2} &\leq \left( \frac{3+\cvabyzbound}{\cvabyzbound} \right)^2 \frac{(8f)^2}{n-f} \diam{\cvavector{}{}}{0}.
\end{align*}
\end{proof}

\section{D-SGD without Momentum under $(\alpha,\lambda)$-\reduction{}}
\label{app_sec:dsgd_anal}
In this section, we prove Proposition~\ref{prop:dsgd_main} which is convergence guarantee for D-SGD (i.e., setting $\beta = 0$ in Algorithm~\ref{algo}) under  $(\alpha,\lambda)$-\reduction. First let us recall Proposition~\ref{prop:dsgd_main} (with more details). 

\noindent \fcolorbox{black}{white}{
\parbox{0.97\textwidth}{\centering
\begin{repproposition}{prop:dsgd_main}
Consider Algorithm~\ref{algo} with $\beta = 0$. 
Suppose assumptions~\ref{asp:lip},~\ref{asp:bnd_var} and~\ref{asp:heter} hold true and that the coordination phase satisfies $(\alpha,\lambda)$-\reduction{} for $\alpha < 1$. Define
\begin{gather*}
c_0 := 6\left({\avgloss(\avgmodel{0})- Q^*}\right) \text{ and } c_1=  \frac{6\sqrt{1+\alpha}}{1-\alpha}.
\end{gather*}
Set $\gamma = \min\{\frac{1}{2L},\frac{1}{c_1L},\sqrt{\frac{nc_0}{9LT\sigma^2}} \} . $ For any correct node $i \in \H$, Algorithm~\ref{algo} then guarantees 
\begin{align*}
\expect{\norm{\nabla \avgloss \left( \returnedmodel{\honestnode} \right)}^2} &\leq  6\sqrt{\frac{c_0 L \sigma^2}{nT}} + \frac{4c_0 c_1^2 L n \alpha}{9T\sigma^2} \left( 3\sigma^2 +  \heter^2 \right)+\frac{(2+c_1)Lc_0}{T} + 20 c_1^2 \lambda\left( 3\sigma^2 + \heter^2 \right)\\
&\in \mathcal{O}\left( \sqrt{\frac{\sigma^2}{nT}}+ \frac{\lambda}{(1-\alpha)^2}\left( \sigma^2 + \heter^2 \right)\right)
\end{align*}
\end{repproposition}
}
}

Before proving the proposition, we first prove a few useful lemmas.
\
\begin{lemma}
\label{lem:growth_SGD}
Suppose that assumptions~\ref{asp:lip} and~\ref{asp:bnd_var} hold true. Consider Algorithm~\ref{algo} 
with $\beta = 0$ and $\gamma \leq 1/2L$. For each $t \in [T]$, we obtain that
\begin{align*}
     \expect{\avgloss(\avgmodel{t+1})- \avgloss(\avgmodel{t})} &\leq -\frac{\gamma}{4} \expect{ \norm{\nabla \avgloss \left( \avgmodel{t} \right)}^2}  + {L\gamma^2 \frac{\sigma^2}{n-f}} + \frac{\gamma L^2}{2}  \expect{\diam{\model{}{}}{t}}+
     \frac{5\lambda}{\gamma} \expect{\diam{\model{}{}}{t+\nicefrac{1}{2}}}.
\end{align*}
\end{lemma}

\begin{proof}
\newcommand{\bias}[1]{R_{#1}}
Consider an arbitrary $t \in [T]$. We define $\effectgrad{t}:=\frac{\avgmodel{t}-\avgmodel{t+1}}{\gamma}$, the step taken by the average of local models at iteration $t$. Also, define $\bias{t} := \avggrad{t}- \effectgrad{t}$.
By the smoothness of the loss function (Assumption~\ref{asp:lip}), we have 
\begin{align*}
    \avgloss(\avgmodel{t+1})- \avgloss(\avgmodel{t}) &\leq \iprod{\avgmodel{t+1} - \avgmodel{t}}{\nabla \avgloss\left( \avgmodel{t} \right)} + \frac{L}{2} \norm{\avgmodel{t+1} - \avgmodel{t}}^2 \nonumber \\
    & = - \gamma \iprod{\effectgrad{t}}{ \nabla \avgloss \left( \avgmodel{t} \right)} + \frac{L\gamma^2}{2} \norm{\effectgrad{t}}^2\\
    &=  - \gamma \iprod{\avggrad{t} + \bias{t}} { \nabla \avgloss \left( \avgmodel{t} \right)}  + \frac{L\gamma^2}{2} \norm{\avggrad{t} + \bias{t}}^2.
\end{align*}
Now denoting $\AvgGrad{t} = \frac{1}{n-f} \sum_{i \in \H} \localgrad{i}\left(\model{i}{t}\right)$, and taking the conditional expectation, we have
\begin{align*}
    \condexpect{t}{\avgloss(\avgmodel{t+1})}- \avgloss(\avgmodel{t}) &\leq - \gamma \iprod{\condexpect{t}{\avggrad{t} + \bias{t}}} { \nabla \avgloss \left( \avgmodel{t} \right)}  + \frac{L\gamma^2}{2} \condexpect{t}{\norm{\avggrad{t} + \bias{t}}^2} \\
    &\overset{(a)}{\leq} - \gamma \iprod{ \AvgGrad{t} + \condexpect{t}{\bias{t}}} { \nabla \avgloss \left( \avgmodel{t} \right)}  + {L\gamma^2} \condexpect{t}{\norm{\avggrad{t}}^2} + {L\gamma^2} \condexpect{t}{\norm{\bias{t}}^2} \\
    &\overset{(b)}{\leq} - \gamma \iprod{ \AvgGrad{t} + \condexpect{t}{\bias{t}}} { \nabla \avgloss \left( \avgmodel{t} \right)}  + {L\gamma^2} \condexpect{t}{\norm{ \AvgGrad{t}}^2} +  {L\gamma^2 \frac{\sigma^2}{n-f}} + {L\gamma^2} \condexpect{t}{\norm{\bias{t}}^2}
\end{align*}
where  (a) uses Young's inequality and (b) is based on the facts that by Assumption~\ref{asp:bnd_var}, we have  $\condexpect{t}{\avggrad{t}} = \AvgGrad{t}$, and $\condexpect{t}{\norm{\avggrad{t}- \AvgGrad{t}}^2}  \leq \frac{\sigma^2}{n-f}$.
We then obtain that
\begin{align}\nonumber
    \condexpect{t}{\avgloss(\avgmodel{t+1})}- \avgloss(\avgmodel{t}) &\leq - \gamma \iprod{ \AvgGrad{t}} { \nabla \avgloss \left( \avgmodel{t} \right)}  + {L\gamma^2} \condexpect{t}{\norm{ \AvgGrad{t}}^2}\\ &+  {L\gamma^2 \frac{\sigma^2}{n-f}} + {L\gamma^2} \condexpect{t}{\norm{\bias{t}}^2} - \gamma \iprod{ \condexpect{t}{\bias{t}}} { \nabla \avgloss \left( \avgmodel{t} \right)}.\label{eq:gorth_zeroth}
\end{align}
Now note that 
\begin{align}
    -  \iprod{ \condexpect{t}{\bias{t}}} { \nabla \avgloss \left( \avgmodel{t} \right)} \leq 4 \norm{ \condexpect{t}{\bias{t}}}^2 + \frac{1}{4} \norm{ \nabla \avgloss \left( \avgmodel{t} \right)}^2 \leq 4 \condexpect{t}{\norm{{\bias{t}}}^2} + \frac{1}{4} \norm{ \nabla \avgloss \left( \avgmodel{t} \right)}^2,\label{eq:groth_first}
\end{align}
where the second inequality uses Jensen's inequality.
Also, using the fact that $\gamma \leq \frac{1}{2L}$, we have
\begin{align}\nonumber
    - \gamma \iprod{ \AvgGrad{t}} { \nabla \avgloss \left( \avgmodel{t} \right)}  + {L\gamma^2} {\norm{ \AvgGrad{t}}^2}
    &\leq \frac{\gamma}{2} \left(-2  \iprod{ \AvgGrad{t}} { \nabla \avgloss \left( \avgmodel{t} \right)}  +{\norm{ \AvgGrad{t}}^2} \right)\\
    &= \frac{\gamma}{2} \left(  -\norm{\nabla \avgloss \left( \avgmodel{t} \right)}^2 + \norm{\nabla \avgloss \left( \avgmodel{t} \right) -  \AvgGrad{t}}^2 \right)
    \label{eq:groth_middle}
\end{align}
Combining \eqref{eq:gorth_zeroth}, \eqref{eq:groth_first}, and \eqref{eq:groth_middle}, we obtain that
\begin{align}\label{eq:desenet_bound}
     \condexpect{t}{\avgloss(\avgmodel{t+1})}- \avgloss(\avgmodel{t}) &\leq -\frac{\gamma}{4} \norm{\nabla \avgloss \left( \avgmodel{t} \right)}^2 + \frac{\gamma}{2} \norm{\nabla \avgloss \left( \avgmodel{t} \right) -  \AvgGrad{t}}^2 + {L\gamma^2 \frac{\sigma^2}{n-f}} + (4\gamma+L\gamma^2) \condexpect{t}{\norm{\bias{t}}^2}.
\end{align}
Now note that 
\begin{align*}
    \bias{t} = \avggrad{t}- \effectgrad{t} = \avggrad{t}-  \frac{\avgmodel{t}-\avgmodel{t+\nicefrac{1}{2}} + \avgmodel{t+\nicefrac{1}{2}}-\avgmodel{t+1}}{\gamma} = \frac{\avgmodel{t+1}-\avgmodel{t+\nicefrac{1}{2}}}{\gamma}.
\end{align*}
Thus, using the definition of $(\alpha,\lambda)$-\reduction, we obtain that
\begin{align*}
    \condexpect{t}{\norm{\bias{t}}^2} = \frac{1}{\gamma^2}\condexpect{t}{\norm{\avgmodel{t+1}-\avgmodel{t+\nicefrac{1}{2}}}^2} \leq \frac{1}{\gamma^2} \lambda \condexpect{t}{\diam{\model{}{}}{t+\nicefrac{1}{2}}}
\end{align*}
Combining this with~\eqref{eq:desenet_bound} taking total expectation, and using the fact that $\gamma \leq \frac{1}{L}$, we obtain that
\begin{align*}
     \expect{\avgloss(\avgmodel{t+1})- \avgloss(\avgmodel{t})} &\leq -\frac{\gamma}{4} \expect{ \norm{\nabla \avgloss \left( \avgmodel{t} \right)}^2} + \frac{\gamma}{2} \expect{\norm{\nabla \avgloss \left( \avgmodel{t} \right) -  \AvgGrad{t}}^2}\\
     &+ {L\gamma^2 \frac{\sigma^2}{n-f}}+ 
     \frac{5\lambda}{\gamma} \expect{\diam{\model{}{}}{t+\nicefrac{1}{2}}}.
\end{align*}
Now note that 
\begin{align*}
    \expect{\norm{\AvgGrad{t}-\nabla \avgloss \left( \avgmodel{t} \right)}^2} &= \expect{\norm{\frac{1}{n-f} \sum_{i \in \H} \localgrad{i}(\model{i}{t}) - \frac{1}{n-f} \sum_{i \in \H} \localgrad{i}\left(\avgmodel{t}\right)}^2}\\
    &= \expect{\norm{\frac{1}{n-f} \sum_{i \in \H} \left(\localgrad{i}(\model{i}{t}) - \localgrad{i}\left(\avgmodel{t}\right)\right)}^2} \\
    &\leq \frac{1}{n-f} \sum_{i \in \H} \expect{\norm{\localgrad{i}(\model{i}{t}) - \localgrad{i}\left(\avgmodel{t}\right)}^2}\\
    &\leq \frac{L^2}{n-f} \sum_{i \in \H} \expect{\norm{\model{i}{t} - \avgmodel{t}}^2} \leq  L^2\expect{\diam{\model{}{}}{t}}.
\end{align*}
Therefore,
\begin{align*}
     \expect{\avgloss(\avgmodel{t+1})- \avgloss(\avgmodel{t})} &\leq -\frac{\gamma}{4} \expect{ \norm{\nabla \avgloss \left( \avgmodel{t} \right)}^2}  + {L\gamma^2 \frac{\sigma^2}{n-f}} + \frac{\gamma L^2}{2}  \expect{\diam{\model{}{}}{t}}+
     \frac{5\lambda}{\gamma} \expect{\diam{\model{}{}}{t+\nicefrac{1}{2}}}.
\end{align*}
This is the lemma.
\end{proof}

\begin{lemma}
\label{lem:SGD:drift}
    Suppose that assumptions~\ref{asp:lip},~\ref{asp:bnd_var}, and \ref{asp:heter} hold true. Consider Algorithm~\ref{algo} with $\gamma \leq \frac{1}{6L} \frac{1-\alpha}{\sqrt{\alpha(1+\alpha)}}$, and $\beta=0$. Suppose that the coordination phase satisfies $(\alpha,\lambda)$-\reduction{} for $\alpha < 1$. For each $t \in [T]$, we obtain that
\begin{align*}
   \expect{\diam{\model{}{}}{t}} \leq \frac{6\alpha(1+\alpha)}{(1-\alpha)^2} \gamma^2 \left( \left(4+\frac{8}{n-f} \right) \sigma^2 + 4 \heter^2 \right),
\end{align*}
and
\begin{align*}
    \expect{\diam{\model{}{}}{t+\nicefrac{1}{2}}} \leq \frac{6(1+\alpha)}{(1-\alpha)^2} \gamma^2 \left( \left(4+\frac{8}{n-f} \right) \sigma^2 + 4 \heter^2 \right).
\end{align*}
\end{lemma}

\begin{proof}
  
 First, we analyze the growth of $\expect{\diam{\model{}{}}{t}}$. From Algorithm \ref{algo} (for $\beta = 0$) recall that for all $i \in \H$, we have $\model{i}{t+\nicefrac{1}{2}} = \model{i}{t} - \gamma \gradient{i}{t}$. As $(x + y)^2 \leq (1 + c) x^2 + (1 + \nicefrac{1}{c}) y^2$ for any $c > 0$, we obtain for all $i, \, j \in \H$ that
\begin{align*}
    \expect{\norm{\model{i}{t+\nicefrac{1}{2}} - \model{j}{t+\nicefrac{1}{2}} }^2} & \leq \expect{\norm{\model{i}{t} - \model{j}{t} - \gamma \left( \gradient{i}{t} - \gradient{j}{t} \right)}^2} \\
    & \leq (1 + c) \expect{\norm{\model{i}{t} - \model{j}{t}}^2} + \left(1 + \frac{1}{c} \right) \gamma^2 \expect{\norm{\gradient{i}{t} - \gradient{j}{t}}^2}.
\end{align*}
Thus, by definition of notation $\diam{*}{t}$ and by Lemma \ref{lem:diam_equal}, we have
\begin{align}
    \expect{\diam{\model{}{}}{t+\nicefrac{1}{2}}} \leq (1 + c) \expect{\diam{\model{}{}}{t}} + \left(1 + \frac{1}{c} \right) \gamma^2 \expect{\diam{\gradient{}{}}{t}}. \label{eqn:drift_before_alpha_sgd}
\end{align}
Recall that,  the coordination phase of Algorithm~\ref{algo} satisfies $(\alpha,\lambda)$-reduction. Thus, for all $t$, we have $\diam{\model{}{}}{t+1} \leq \alpha \diam{\model{}{}}{t+\nicefrac{1}{2}}$. Substituting from above we obtain that
\begin{align*}
   \expect{\diam{\model{}{}}{t+1}} \leq (1 + c) \alpha  \expect{\diam{\model{}{}}{t}} + \left(1 + \frac{1}{c} \right) \alpha \gamma^2 \expect{\diam{\gradient{}{}}{t}}. 
\end{align*}
  For $c = \frac{1-\alpha}{2\alpha}$, we obtain that
  \begin{align}\label{eq:drift_sgd_bound}
   \expect{\diam{\model{}{}}{t+1}} \leq \frac{1+\alpha}{2}\expect{\diam{\model{}{}}{t}} +\frac{\alpha(1+\alpha)}{1-\alpha} \gamma^2 \expect{\diam{\gradient{}{}}{t}}. 
\end{align}
Now note that for any $i \in \H$ we have 
\begin{align*}
    \gradient{i}{t}-\avggrad{t}  &= \gradient{i}{t} - \localgrad{i}\left( \model{i}{t}\right) + \localgrad{i}\left( \model{i}{t}\right) - \localgrad{i}\left( \avgmodel{t}\right) + \localgrad{i}\left( \avgmodel{t}\right) - 
    \nabla \avgloss \left( \avgmodel{t} \right)+
    \nabla \avgloss \left( \avgmodel{t} \right) - \avggrad{t}.
\end{align*}
Thus,
\begin{align*}
    \expect{\norm{
    \gradient{i}{t}-\avggrad{t}}^2} &\leq 4\expect{\norm{\gradient{i}{t} - \localgrad{i}\left( \model{i}{t}\right)}^2}+4\expect{\norm{\localgrad{i}\left( \model{i}{t}\right) - \localgrad{i}\left( \avgmodel{t}\right)}^2}\\
    &+4\expect{\norm{\localgrad{i}\left( \avgmodel{t}\right) - 
    \nabla \avgloss \left( \avgmodel{t} \right)}^2}+4\expect{\norm{\nabla \avgloss \left( \avgmodel{t} \right) - \avggrad{t}}^2}.
\end{align*}
Using assumptions \ref{asp:lip}, and \ref{asp:bnd_var}, we obtain that
\begin{align}\nonumber
    \expect{\norm{
    \gradient{i}{t}-\avggrad{t}}^2} &\leq 4 \sigma^2 + 4L^2\expect{\norm{ \model{i}{t}- \avgmodel{t}}^2}\\
    &+4\expect{\norm{\localgrad{i}\left( \avgmodel{t}\right) - 
    \nabla \avgloss \left( \avgmodel{t} \right)}^2}+4\expect{\norm{\nabla \avgloss \left( \avgmodel{t} \right) - \avggrad{t}}^2}.\label{eq:sgd_decompone}
\end{align}
Now note that 
\begin{align*}
    \expect{\norm{\nabla \avgloss \left( \avgmodel{t} \right) - \avggrad{t}}^2} &= \frac{1}{(n-f)^2} \expect{\norm{\sum_{i \in \H} \left(\localgrad{i}\left( \avgmodel{t}\right) - \gradient{i}{t} \right) } ^2} \\
    &\leq \frac{2}{(n-f)^2} \expect{\norm{\sum_{i \in \H} \left(\localgrad{i}\left( \avgmodel{t}\right) - \localgrad{i}\left( \model{i}{t}\right) \right) } ^2}\\&+\frac{2}{(n-f)^2} \expect{\norm{\sum_{i \in \H} \left(\localgrad{i}\left( \model{i}{t}\right) - \gradient{i}{t} \right) } ^2}\\
    &\leq \frac{2L^2}{n-f} \expect{\sum_{i \in \H} \norm{  \avgmodel{t} - \model{i}{t}}^2}+  \frac{2}{n-f} \sigma^2\\
    &=2L^2 \expect{\diam{\model{}{}}{t}} + \frac{2}{n-f} \sigma^2.
\end{align*}
Combining this with~\eqref{eq:sgd_decompone}, we obtain that
\begin{align*}
    \expect{\norm{
    \gradient{i}{t}-\avggrad{t}}^2} &\leq 4 \sigma^2 + 4L^2\expect{\norm{ \model{i}{t}- \avgmodel{t}}^2}
    +4\expect{\norm{\localgrad{i}\left( \avgmodel{t}\right) - 
    \nabla \avgloss \left( \avgmodel{t} \right)}^2}+ 8L^2 \expect{\diam{\model{}{}}{t}} + \frac{8}{n-f} \sigma^2
\end{align*}
As $i$ above is an arbitrary node in $\H$, the above holds true for all $i \in \H$. Averaging over all $i \in \H$ on both sides yields
\begin{align*}
    \frac{1}{\card{\H}} \sum_{i \in \H} \expect{\norm{
    \gradient{i}{t}-\avggrad{t}}^2} \leq & 4 \sigma^2 + 4L^2 \frac{1}{\card{\H}} \sum_{i \in \H} \expect{\norm{ \model{i}{t}- \avgmodel{t}}^2}
    +4 \frac{1}{\card{\H}} \sum_{i \in \H} \expect{\norm{\localgrad{i}\left( \avgmodel{t}\right) - \nabla \avgloss \left( \avgmodel{t} \right)}^2} \\
    & + 8L^2 \expect{\diam{\model{}{}}{t}} + \frac{8}{n-f} \sigma^2.
\end{align*}
Recall, from Section~\ref{app_sec:monna_proof}, the notation $\diam{*}{t}$, i.e., $\diam{*}{t}=\frac{1}{\card{C}}\sum_{i \in \H} \norm{*^{(i)}_{t} - \bar{*}_{t}}^2$. Using this above, we get 
\begin{align*}
   \expect{\diam{\gradient{}{}}{t}} \leq & 4 \sigma^2 + 4L^2 \expect{\diam{\model{}{}}{t}} 
    +4 \frac{1}{\card{\H}} \sum_{i \in \H} \expect{\norm{\localgrad{i}\left( \avgmodel{t}\right) - 
    \nabla \avgloss \left( \avgmodel{t} \right)}^2} + 8L^2 \expect{\diam{\model{}{}}{t}} + \frac{8}{n-f} \sigma^2.
\end{align*}
From Assumption~\ref{asp:heter}, we have $\frac{1}{\card{\H}} \sum_{i \in \H} \expect{\norm{\localgrad{i}\left( \avgmodel{t}\right) - 
    \nabla \avgloss \left( \avgmodel{t} \right)}^2} \leq \zeta^2$. Using this above, we obtain that
\begin{align}\label{eq:bound_gt}
    \expect{\diam{\gradient{}{}}{t}} \leq \left(4+\frac{8}{n-f} \right) \sigma^2 + 12L^2 \expect{\diam{\model{}{}}{t}} + 4 \heter^2.
\end{align}

Combining this with~\eqref{eq:drift_sgd_bound}, we obtain that
\begin{align*}
   \expect{\diam{\model{}{}}{t+1}} \leq \left( \frac{1+\alpha}{2} + 12 \frac{\alpha(1+\alpha)}{1-\alpha} \gamma^2 L^2 \right) \expect{\diam{\model{}{}}{t}} +\frac{\alpha(1+\alpha)}{1-\alpha} \gamma^2 \left( \left(4+\frac{8}{n-f} \right) \sigma^2 + 4 \heter^2 \right).
\end{align*}
For $\gamma \leq \frac{1}{6L} \frac{1-\alpha}{\sqrt{\alpha(1+\alpha)}}$, we have 

\begin{align*}
   \expect{\diam{\model{}{}}{t+1}} \leq \frac{5+\alpha}{6}  \expect{\diam{\model{}{}}{t}} +\frac{\alpha(1+\alpha)}{1-\alpha} \gamma^2 \left( \left(4+\frac{8}{n-f} \right) \sigma^2 + 4 \heter^2 \right).
\end{align*}
Unrolling the recursion, we obtain that
\begin{align}\nonumber
   \expect{\diam{\model{}{}}{t}} &\leq \frac{\alpha(1+\alpha)}{1-\alpha} \gamma^2 \left( \left(4+\frac{8}{n-f} \right) \sigma^2 + 4 \heter^2 \right) \sum_{s=0}^t \left(\frac{5+\alpha}{6} \right)^s \\ \nonumber
   &\leq \frac{\alpha(1+\alpha)}{1-\alpha} \gamma^2 \left( \left(4+\frac{8}{n-f} \right) \sigma^2 + 4 \heter^2 \right) \sum_{s=0}^\infty \left(\frac{5+\alpha}{6} \right)^s \\ &= \frac{6\alpha(1+\alpha)}{(1-\alpha)^2} \gamma^2 \left( \left(4+\frac{8}{n-f} \right) \sigma^2 + 4 \heter^2 \right).\label{eq:sgd_final_theta}
\end{align}
Combining~\eqref{eqn:drift_before_alpha_sgd}, \eqref{eq:bound_gt}, and \eqref{eq:sgd_final_theta}, we also obtain that
\begin{align*}
    \expect{\diam{\model{}{}}{t+\nicefrac{1}{2}}} \leq \frac{6(1+\alpha)}{(1-\alpha)^2} \gamma^2 \left( \left(4+\frac{8}{n-f} \right) \sigma^2 + 4 \heter^2 \right).
\end{align*}
This is the desired result.
\end{proof}

\begin{lemma}
\label{lem:dsgd}
Consider Algorithm~\ref{algo} with $\beta = 0$. Define \begin{gather*}
c_0 := 6\left({\avgloss(\avgmodel{0})- Q^*}\right) \text{ and } c_1=  \frac{6\sqrt{1+\alpha}}{1-\alpha}.
\end{gather*}
Suppose that assumptions~\ref{asp:lip},~\ref{asp:bnd_var} and~\ref{asp:heter} hold true, and that the coordination phase satisfies $(\alpha,\lambda)$-\reduction{} for $\alpha < 1$. Suppose also that $\gamma \leq \min\{\frac{1}{2L},\frac{1}{c_1L}\}$. 
Then, for any correct node $i$, and any $T\geq 1$, we have
\begin{align*}
     \expect{\norm{\nabla \avgloss \left( \returnedmodel{\honestnode} \right)}^2} \leq \frac{c_0}{\gamma T}  + {9L\gamma \frac{\sigma^2}{n}} + 4 c_1^2 n \alpha L^2 \gamma^2 \left( 3\sigma^2 +  \heter^2 \right) +20 c_1^2 \lambda\left( 3\sigma^2 + \heter^2 \right).
\end{align*}

\end{lemma}

\begin{proof}
    Combining Lemma~\ref{lem:growth_SGD} and Lemma~\ref{lem:SGD:drift}, we have
    \begin{align*}
        \expect{ \norm{\nabla \avgloss \left( \avgmodel{t} \right)}^2} &\leq -\frac{4}{\gamma}\expect{\avgloss(\avgmodel{t+1})- \avgloss(\avgmodel{t})}   + {4L\gamma \frac{\sigma^2}{n-f}} + 2 L^2 \expect{\diam{\model{}{}}{t}}+
     \frac{20 \lambda}{\gamma^2} \expect{\diam{\model{}{}}{t+\nicefrac{1}{2}}}\\
     &\leq -\frac{4}{\gamma}\expect{\avgloss(\avgmodel{t+1})- \avgloss(\avgmodel{t})}   + {4L\gamma \frac{\sigma^2}{n-f}} \\&+ (2 L^2 \alpha \gamma^2+20 \lambda) \frac{6(1+\alpha)}{(1-\alpha)^2}  \left( \left(4+\frac{8}{n-f} \right) \sigma^2 + 4 \heter^2 \right)\\
     &\leq-\frac{4}{\gamma}\expect{\avgloss(\avgmodel{t+1})- \avgloss(\avgmodel{t})}   + {4L\gamma \frac{\sigma^2}{n-f}} + (2L^2 \gamma^2 \alpha +20 \lambda) \frac{6(1+\alpha)}{(1-\alpha)^2}  \left( 12\sigma^2 + 4 \heter^2 \right),
    \end{align*}
    where in the last inequality, we used the fact that $n-f  \geq 1$. 
    Averaging over $t=0,\ldots,T-1$, we obtain that
    \begin{align}\nonumber
        \frac{1}{T}\sum_{t=0}^{T-1} \expect{ \norm{\nabla \avgloss \left( \avgmodel{t} \right)}^2}  &\leq \frac{4}{\gamma T}\expect{\avgloss(\avgmodel{0})- \avgloss(\avgmodel{T})}   + {4L\gamma \frac{\sigma^2}{n-f}} + (2L^2 \gamma^2 \alpha +20 \lambda) \frac{6(1+\alpha)}{(1-\alpha)^2}  \left( 12\sigma^2 + 4 \heter^2 \right) \\
        &\leq \frac{4}{\gamma T}\left({\avgloss(\avgmodel{0})- Q^*} \right) + {6L\gamma \frac{\sigma^2}{n}} + (2L^2 \gamma^2 \alpha +20 \lambda) \frac{6(1+\alpha)}{(1-\alpha)^2}  \left( 12\sigma^2 + 4 \heter^2 \right), \label{eq:SGD_rate_begin}
    \end{align}
    where we used the fact that $\avgloss(\avgmodel{T}) \geq Q^*$, and $n\geq 5f$.
    Now note that for any   correct node $i \in \H$, we have
\begin{align*}
    \expect{\norm{\nabla \avgloss \left( \model{i}{t} \right)}^2} &\leq \frac{3}{2}\expect{\norm{\nabla \avgloss \left( \avgmodel{t} \right)}^2} +  3\expect{\norm{\nabla \avgloss \left( \avgmodel{t} \right)-\nabla \avgloss \left( \model{i}{t} \right)}^2}\\
    &\leq \frac{3}{2}\expect{\norm{\nabla \avgloss \left( \avgmodel{t} \right)}^2} +  3L^2\expect{\norm{ \avgmodel{t} - \model{i}{t}}^2},
\end{align*}
where the second inequality follows from Assumption~\ref{asp:lip}.
Combining this with~\eqref{eq:SGD_rate_begin} and 
using the bound on $\expect{\diam{\model{}{t}}{}}$ from Lemma \ref{lem:SGD:drift}, we obtain that
\begin{align*}
    \frac{1}{T}\sum_{t=0}^{T-1} \expect{\norm{\nabla \avgloss \left( \model{i}{t} \right)}^2}
    &\leq   \frac{6}{\gamma T}\left({\avgloss(\avgmodel{0})- Q^*}\right)   + {9L\gamma \frac{\sigma^2}{n}} + (4L^2 \gamma^2 n \alpha  +20 \lambda) \frac{9(1+\alpha)}{(1-\alpha)^2}  \left( 12\sigma^2 + 4 \heter^2 \right),
\end{align*}
where we used the fact that $n \geq 1$.
Defining $c_0 := 6\left({\avgloss(\avgmodel{0})- Q^*}\right)$, $c_1^2=  \frac{36(1+\alpha)}{(1-\alpha)^2}$, we have
\begin{align*}
    \frac{1}{T}\sum_{t=0}^{T-1} \expect{\norm{\nabla \avgloss \left( \model{i}{t} \right)}^2}
    &\leq   \frac{c_0}{\gamma T}  + {9L\gamma \frac{\sigma^2}{n}} + 4 c_1^2 n \alpha L^2 \gamma^2 \left( 3\sigma^2 +  \heter^2 \right) +20 c_1^2 \lambda\left( 3\sigma^2 + \heter^2 \right).
\end{align*}
As $ \returnedmodel{\honestnode} \sim \mathcal{U}\{ \model{i}{0}, \dots, \model{i}{T-1} \}$, we have
\begin{align*}
    \expect{\norm{\nabla \avgloss \left( \returnedmodel{\honestnode} \right)}^2} = \frac{1}{T}\sum_{t=0}^{T-1}\expect{\norm{\nabla \avgloss \left( \model{i}{t} \right)}^2}.
\end{align*}
This proves the desired result.
\end{proof}

Proof of Proposition~\ref{prop:dsgd_main} then straightforwardly follows.
\begin{proof}[Proof of Proposition~\ref{prop:dsgd_main}]
Recall from Lemma~\ref{lem:dsgd} that
$$\expect{\norm{\nabla \avgloss \left( \returnedmodel{\honestnode} \right)}^2} \leq   \frac{c_0}{\gamma T}  + {9L\gamma \frac{\sigma^2}{n}} + 4 c_1^2 n \alpha L^2 \gamma^2 \left( 3\sigma^2 +  \heter^2 \right) +20 c_1^2 \lambda\left( 3\sigma^2 + \heter^2 \right).$$
As $\gamma = \min\{\frac{1}{2L},\frac{1}{c_1L},\sqrt{\frac{nc_0}{9LT\sigma^2}} \}$, we have
$$\frac{1}{\gamma} \leq \max \{2L, c_1L,\sqrt{\frac{9LT\sigma^2}{nc_0}} \} \leq  2L+c_1L + \sqrt{\frac{9LT\sigma^2}{nc_0}}   .$$
Therefore,
\begin{align*}
   \expect{\norm{\nabla \avgloss \left( \returnedmodel{\honestnode} \right)}^2}  \leq 6\sqrt{\frac{c_0 L \sigma^2}{nT}} + \frac{4c_0 c_1^2 L n \alpha}{9T\sigma^2} \left( 3\sigma^2 +  \heter^2 \right)+\frac{(2+c_1)Lc_0}{T} + 20 c_1^2 \lambda\left( 3\sigma^2 + \heter^2 \right).
\end{align*}
Now ignoring the non-dominant $\frac{1}{T}$ terms and noting $c_0 \in \mathcal{O}(1)$ and $c_1 \in \mathcal{O}\left( \frac{1}{1-\alpha}\right)$, we obtain that
\begin{align*}
  \expect{\norm{\nabla \avgloss \left( \returnedmodel{\honestnode} \right)}^2} \in \mathcal{O}\left( \sqrt{\frac{\sigma^2}{nT}}+ \frac{\lambda}{(1-\alpha)^2}\left( \sigma^2 + \heter^2 \right)\right)
\end{align*}
which is the desired result.
\end{proof}

\color{black}

\section{Signed Echo Broadcast (SEB)}
\label{app:cb}
SEB is composed of four rounds of communication. First, in round $\mathsf{SEND}$, each sender node $i$ sends its message $m_i$ to all the other nodes $j$, which contains an identifier and a signature by node $i$. In the context of \newalgorithm{}, the message is a vector $x^{(i)}_{k-1}$, with an identifier of the form $(i,t,k)$, where $t$ is the SGD iteration and $k$ is the coordination round. Second, in round $\mathsf{ECHO}$, upon receiving $m_i$, each recipient node $j$ verifies that $m_i$ is the first message with a valid signature from node $i$ and with identifier $(i,t,k)$. If so, node $j$ signs $m_i$ with its private key $pk_j$, thereby obtaining a {\em signature} $s_{ij}$ of message $m_i$ which node $j$ sends to node $i$. Otherwise, $m_i$ is ignored. Third, in round $\mathsf{FINAL}$, upon receiving at least $\frac{n+f}{2} -1$ valid signatures $s_{ij}$ , the sender node $i$ sends the set $S_i$ of received signatures $s_{ij}$ to all other nodes. Fourth and finally, in round $\mathsf{ACCEPT}$, upon receiving the set $S_i$, each recipient node $j$ verifies the signatures of the set, and if they are valid, node $j$ accepts $m_i$ and terminates the protocol. For $n > 3f$, SEB guarantees {\em validity}, even under asynchrony, i.e., any correct node's message is eventually delivered to any other correct nodes. It also guarantees {\em consistency}, i.e., two different correct nodes cannot deliver different messages from a \byzantine{} node ~\citep[Section 3.10.4]{cachin2011introduction}. The message complexity of this protocol is linear in the total number of nodes, i.e., $\mathcal{O}(n)$. Therefore, it does not affect the communication complexity of \newalgorithm{}.

\section{Additional Information on the Experimental Setup}\label{app:exp_setup}

\subsection{Dataset Pre-Processing}\label{app:pre_process}
MNIST receives an input image normalization of mean $0.1307$ and standard deviation $0.3081$. Furthermore, the images of CIFAR-10 are horizontally flipped, and per channel normalization is also applied with means 0.4914, 0.4822, 0.4465 and standard deviations 0.2023, 0.1994, 0.2010.

\subsection{Model Architecture and Detailed Experimental Setup}\label{app:model_arch}
In order to present the detailed architecture of the models used, we adopt the following compact notation introduced as done, e.g., in \cite{distributed-momentum}.\\

\parbox{0.98\textwidth}{
L(\#outputs) represents a \textbf{fully-connected linear layer}, R stands for \textbf{ReLU activation}, S stands for \textbf{log-softmax}, C(\#channels) represents a \textit{fully-connected 2D-convolutional layer} (kernel size 5, padding 0, stride 1), M stands for \textbf{2D-maxpool} (kernel size 2), B stands for \textbf{batch-normalization}, and D represents \textbf{dropout} (with fixed probability 0.25).
}.\\\\
The architecture of the models, as well as other details on the experimental setup, are presented in Table~\ref{table_exp}. Note that CNN stands for convolutional neural network, and NLL refers to the negative log likelihood loss.
\begin{table}[h!]
\centering
\def\arraystretch{1.5}
\begin{tabular}{|c||c|c|} 
 \hline
 \textit{Dataset} & MNIST & CIFAR-10 \\ [0.5ex] 
 \hline
 \textit{Data heterogeneity} & $\alpha \in \{0.5, 1, 5\} $ & $\alpha = 5$\\ 
 \hline
 \textit{Model type} & CNN & CNN\\ 
 \hline
 \textit{Model architecture} & C(20)-R-M-C(20)-R-M-L(500)-R-L(10)-S & (3,32×32)-C(64)-R-B-C(64)-R-B-M-D-C(128)-\\
 & & R-B-C(128)-R-B-M-D-L(128)-R-D-L(10)-S\\
 \hline
 \textit{Loss} & NLL & NLL\\
  \hline
  \textit{$\ell_2$-regularization} & $10^{-4}$ & $10^{-2}$\\
  \hline
 \textit{Learning rate} & $\gamma = 0.75$ & $\gamma = 0.5$\\
  \hline
  \textit{Batch size} & $b = 25$ & $b = 50$\\
  \hline
  \textit{Momentum} & $\beta = 0.99$ (except for SCC: $\beta = 0.9$) & $\beta = 0.99$ (except for SCC: $\beta = 0.9$)\\
  \hline
 \textit{Number of nodes} & $n=26$ & $n=16$\\
 \hline
 \textit{Number of faults} & $f= 5$ & $f = 3$\\
 \hline
 \textit{Number of Iterations} & $T= 600$ & $T = 2000$\\
 \hline
\end{tabular}
\caption{Detailed experimental setting of Section~\ref{sec:experiments}}
\label{table_exp}
\end{table}

\subsection{Data Heterogeneity}\label{app_exp_setup_distribution}
\begin{figure*}[ht!]
    \centering
    \includegraphics[width=0.33\textwidth]{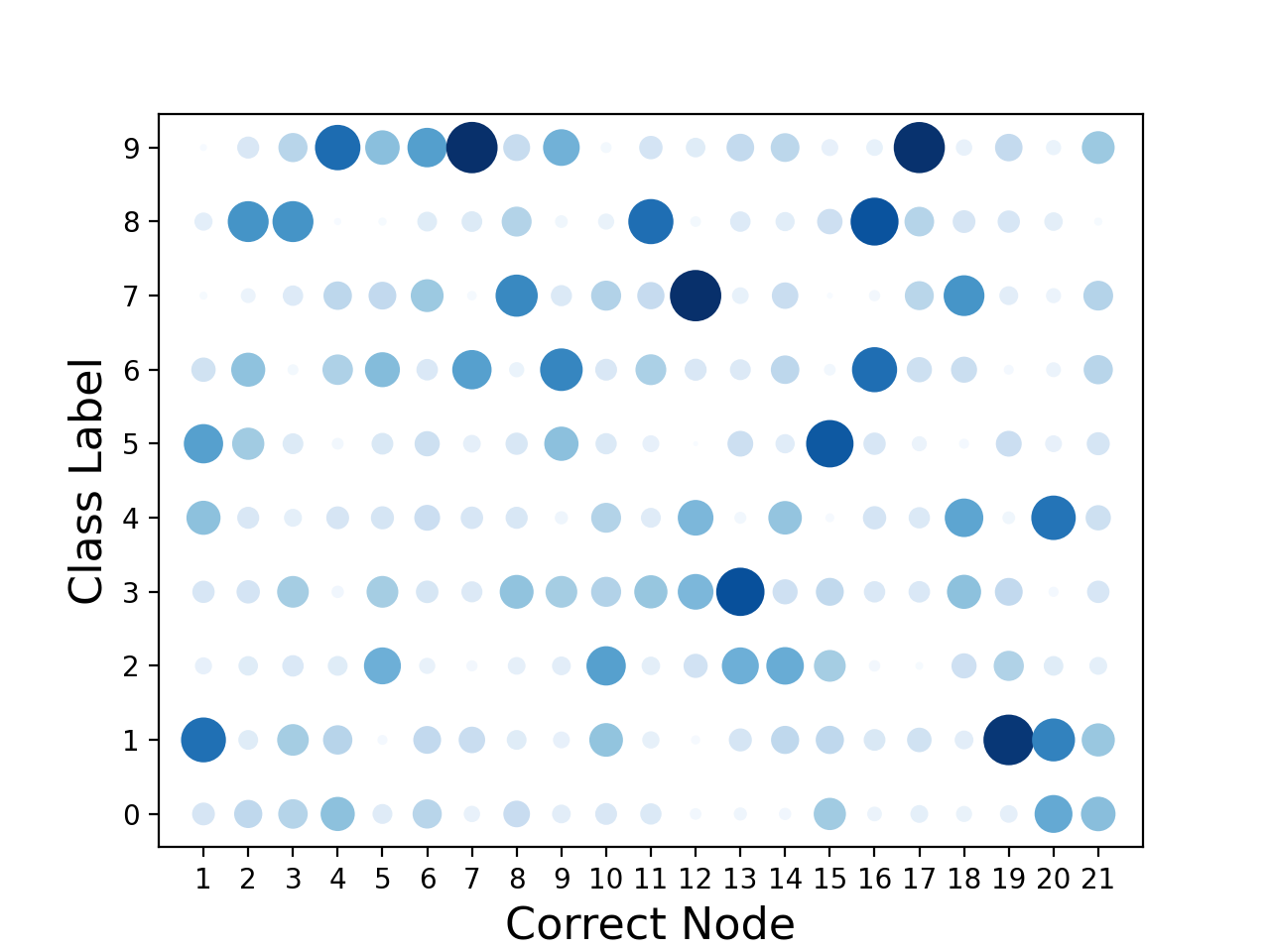}%
    \includegraphics[width=0.33\textwidth]{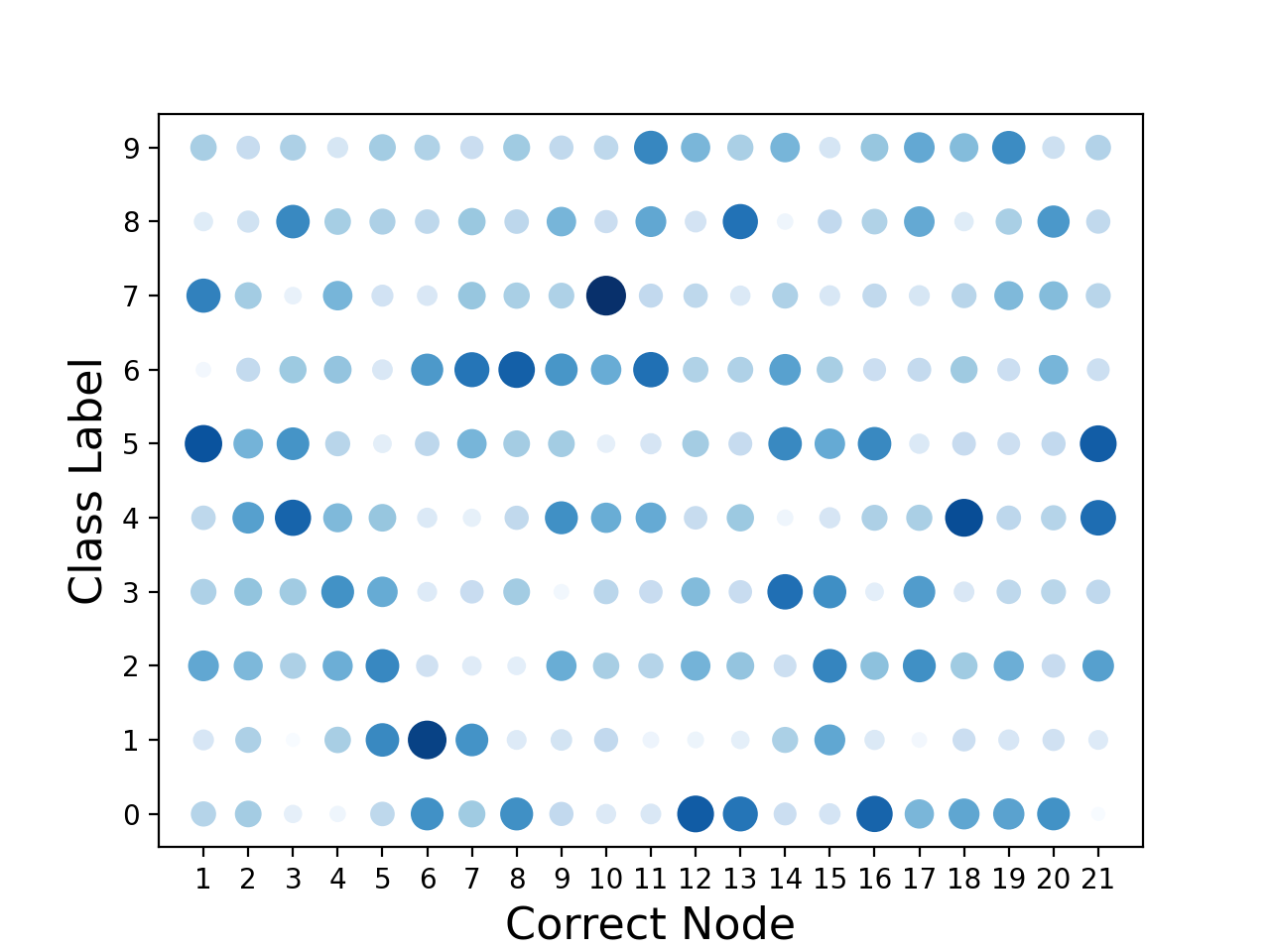}%
    \includegraphics[width=0.33\textwidth]{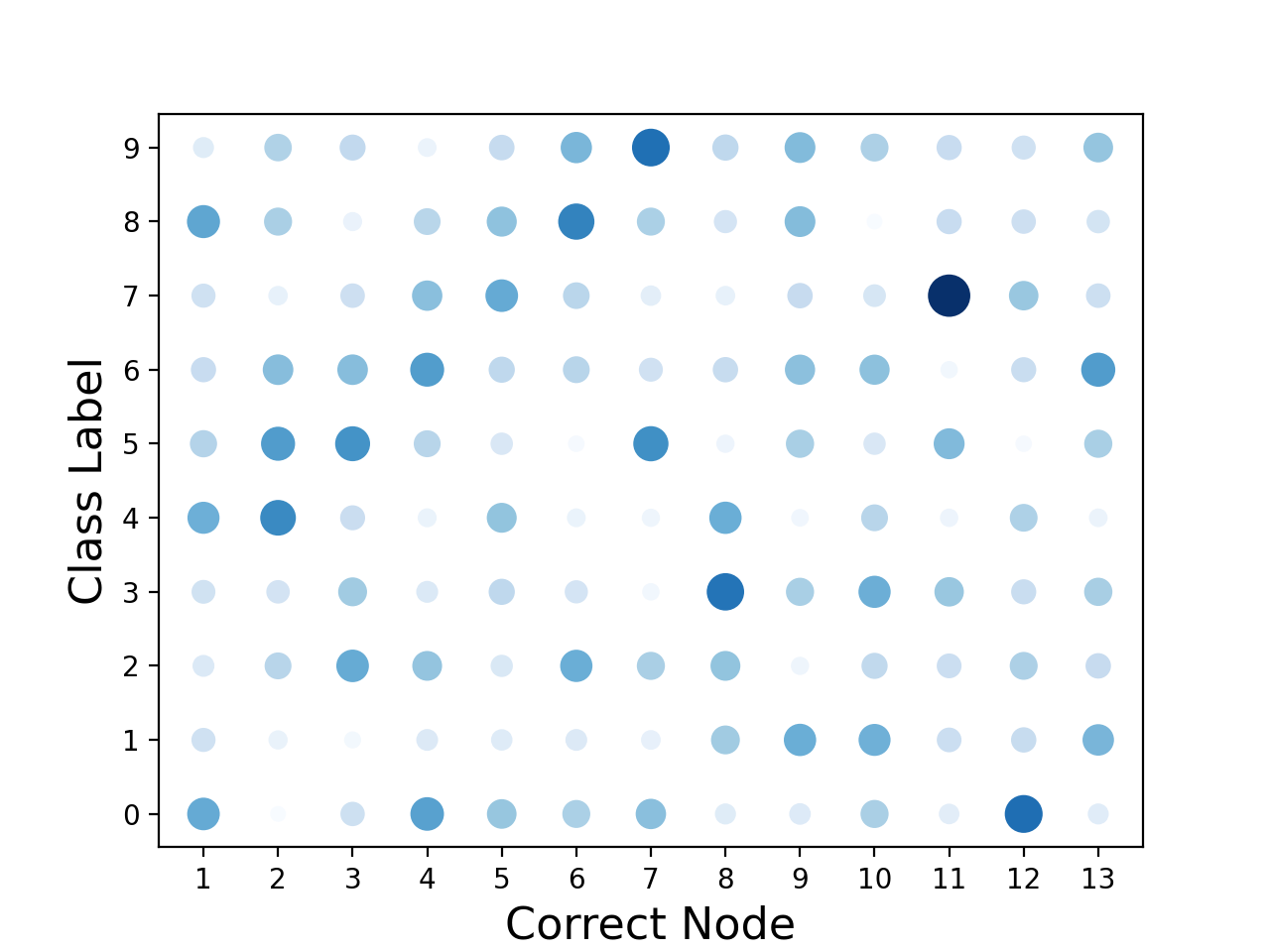}%
    \caption{Distribution of class labels across correct nodes when sampling from a Dirichlet distribution of parameter $\alpha$.\\\textit{Left}: MNIST with $\alpha = 1$, \textit{Middle}: MNIST with $\alpha = 5$, \textit{Right}: CIFAR-10 with $\alpha = 5$.}
\label{fig:distribution}
\end{figure*}

We simulate data heterogeneity in the correct nodes' datasets by making the nodes sample from MNIST and CIFAR-10 using a \textbf{Dirichlet} distribution of parameter $\alpha > 0$. A smaller $\alpha$ implies a more heterogeneous setting (i.e., the more probable it is for nodes to sample datapoints from only one class). For MNIST, we choose $\alpha \in \{1, 5\}$, while we set $\alpha = 5$ on the more difficult task CIFAR-10.
The corresponding distributions of class labels across correct nodes are shown in Figure~\ref{fig:distribution}.

\subsection{Attacks}\label{app:attacks}
We use four state-of-the-art \textit{gradient-based} attacks, namely \textit{fall of empires (FOE)}~\cite{empire}, \textit{a little is enough (ALIE)}~\cite{little}, \textit{sign-flipping (SF)}~\cite{allen2020byzantine}, and \textit{label-flipping (LF)}~\cite{allen2020byzantine}. Since these attacks are originally designed on gradients, we first modify them to be executed on parameter vectors. The first three adapted attacks (FoE, ALIE, and SF) rely on the following key notion. Let $a_t$ be the attack vector in iteration $t$ and let $\zeta_t$ be a fixed non-negative real number. In every iteration $t$, all \byzantine{} nodes broadcast the same vector $\overline{\theta_t} + \zeta_t a_t$ to all other nodes, where $\overline{\theta_t}$ is the average of the parameter vectors of the correct nodes in iteration $t$. Each attack among the first three follows the general scheme we just described, with the following particularities.
\begin{enumerate}[label=(\alph*), leftmargin=*]
    \item \textbf{ALIE.} In this attack, $a_t = - \sigma_t$, where $\sigma_t$ is the opposite vector of the coordinate-wise standard deviation of $\overline{\theta_t}$. In our experiments on ALIE, $\zeta_t$ is chosen through an extensive grid search. Essentially, in each iteration $t$, we choose the value that results in the worst \byzantine{} vector, i.e, the vector for which the  distance to $\overline{\theta_t}$ is the largest.
    \item \textbf{FOE.} In this attack, $a_t = - \overline{\theta_t}$. All \byzantine{} nodes thus send $(1 - \zeta_t) \overline{\theta_t}$ in iteration $t$. Similar to \textit{ALIE}, $\zeta_t$ for \textit{FoE} is also estimated through grid searching.
    \item \textbf{SF.} In this attack, $a_t = - \overline{\theta_t}$ and $\zeta_t = 2$. All \byzantine{} nodes thus send $-\overline{\theta_t}$ in iteration $t$.
    \item \textbf{LF.} Under the LF attack, all \byzantine{} nodes send the same vector $\hat{\theta_t}$ in iteration $t$, where $\hat{\theta_t}$ is the average of the correct parameter vectors but computed on \textit{flipped} labels. In order to do so, in each iteration $t$, the \byzantine{} nodes compute the gradients of the correct nodes on flipped labels. Since the labels for MNIST and CIFAR-10 are in $\{0, 1, ..., 9\}$, the labels are flipped such that $l' = 9 - l$ for every training datapoint, where $l$ is the original label and $l'$ is the flipped label. Each \byzantine{} node then averages all \textit{flipped} parameter vectors to get $\hat{\theta_t}$.
\end{enumerate}

\subsection{Computing Infrastructure}
\subsubsection{Software Dependencies:}
Python 3.8.10 has been used to run our scripts. Besides the standard libraries associated with Python 3.8.10, our scripts use the following libraries:
\small
\begin{center}
\begin{tabular}{||c | c||} 
 \hline
 Library & Version \\ [0.5ex] 
 \hline\hline
 numpy & 1.19.1  \\
 \hline
 torch & 1.6.0 \\
 \hline
 torchvision &  0.7.0\\
  \hline
 pandas      & 1.1.0 \\
 \hline
 matplotlib & 3.0.2\\
 \hline
 PIL & 7.2.0 \\
 \hline
 requests    & 2.21.0  \\
 \hline
 urllib3     & 1.24.1 \\
 \hline
 chardet & 3.0.4  \\
 \hline
 certifi & 2018.08.24\\
 \hline
 idna & 2.6\\
 \hline
 six & 1.15.0 \\
 \hline
 pytz  & 2020.1 \\
 \hline
 dateutil & 2.6.1\\
 \hline
 pyparsing & 2.2.0\\
 \hline
 cycler & 0.10.0\\
 \hline
 kiwisolver & 1.0.1 \\
 \hline
 cffi & 1.13.2\\
 \hline
\end{tabular}
\end{center}
\normalsize

Some dependencies are essential, while others are optional (e.g., only used to process the results and produce the plots).Furthermore, our code has been tested on the following OS: Ubuntu 20.04.4 LTS (GNU/Linux 5.4.0-121-generic x86\_64).

\subsubsection{Hardware Dependencies:} We list below the hardware components used:
\begin{itemize}
    \item 1 Intel(R) Core(TM) i7-8700K CPU @ 3.70GHz
    \item 2 Nvidia GeForce GTX 1080 Ti
    \item 64 GB of RAM
\end{itemize}
\section{Additional Experimental Results}\label{app:exp_results}

\subsection{Remaining Plots on MNIST}\label{app_results_mnist}
\begin{figure*}[ht]
    \centering
    \includegraphics[width=71mm]{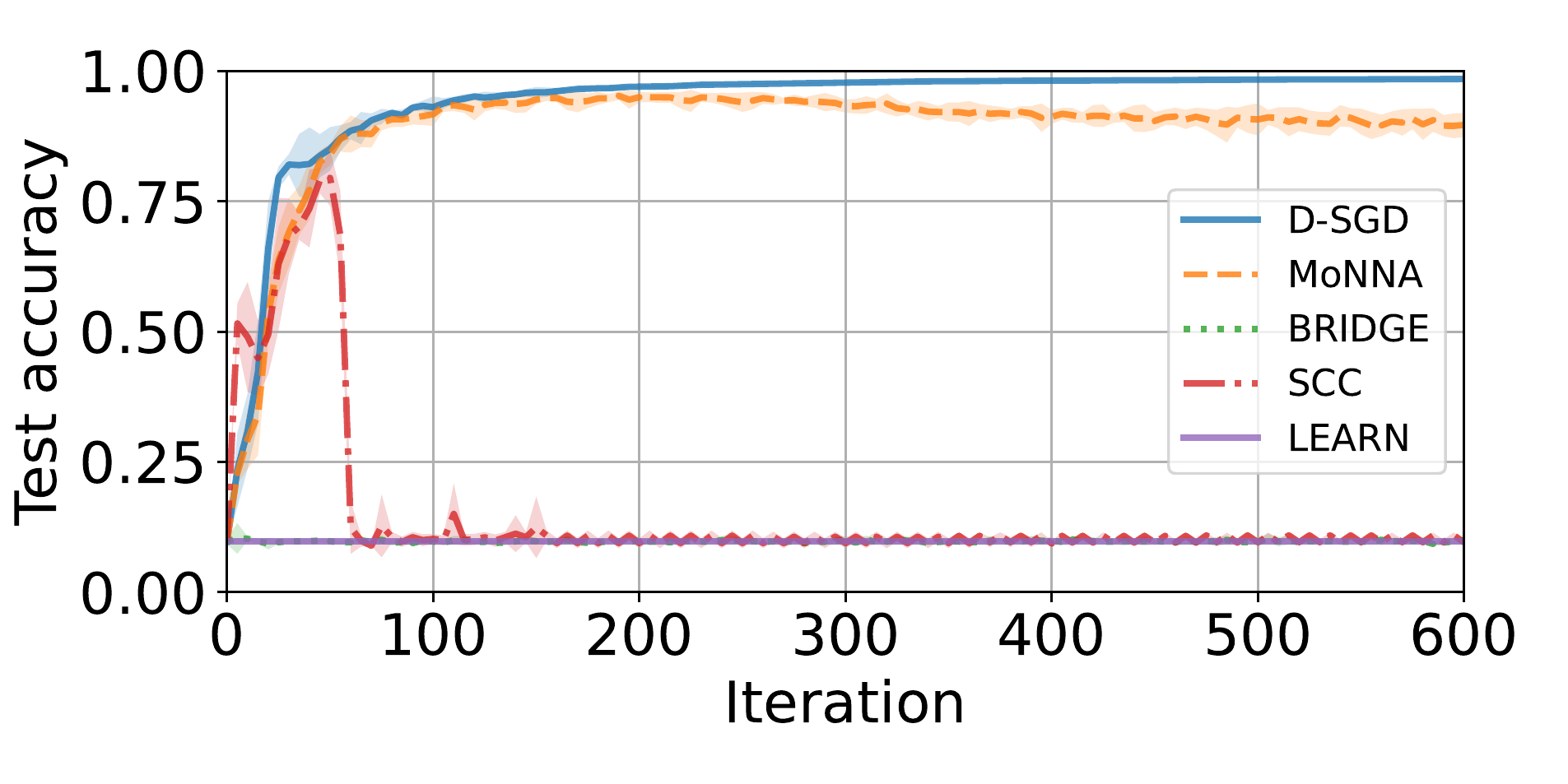}%
    \includegraphics[width=71mm]{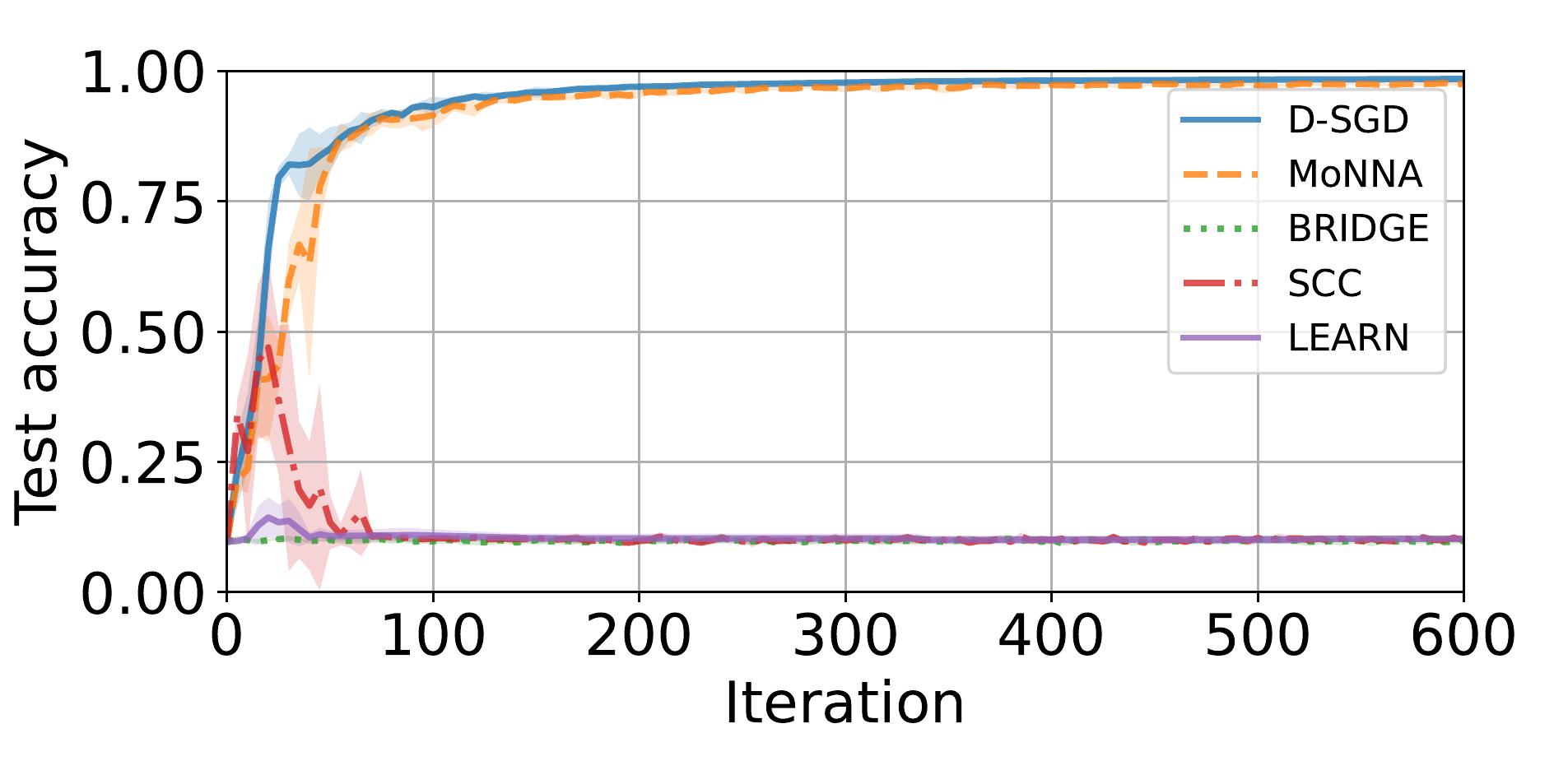}\\
    \includegraphics[width=71mm]{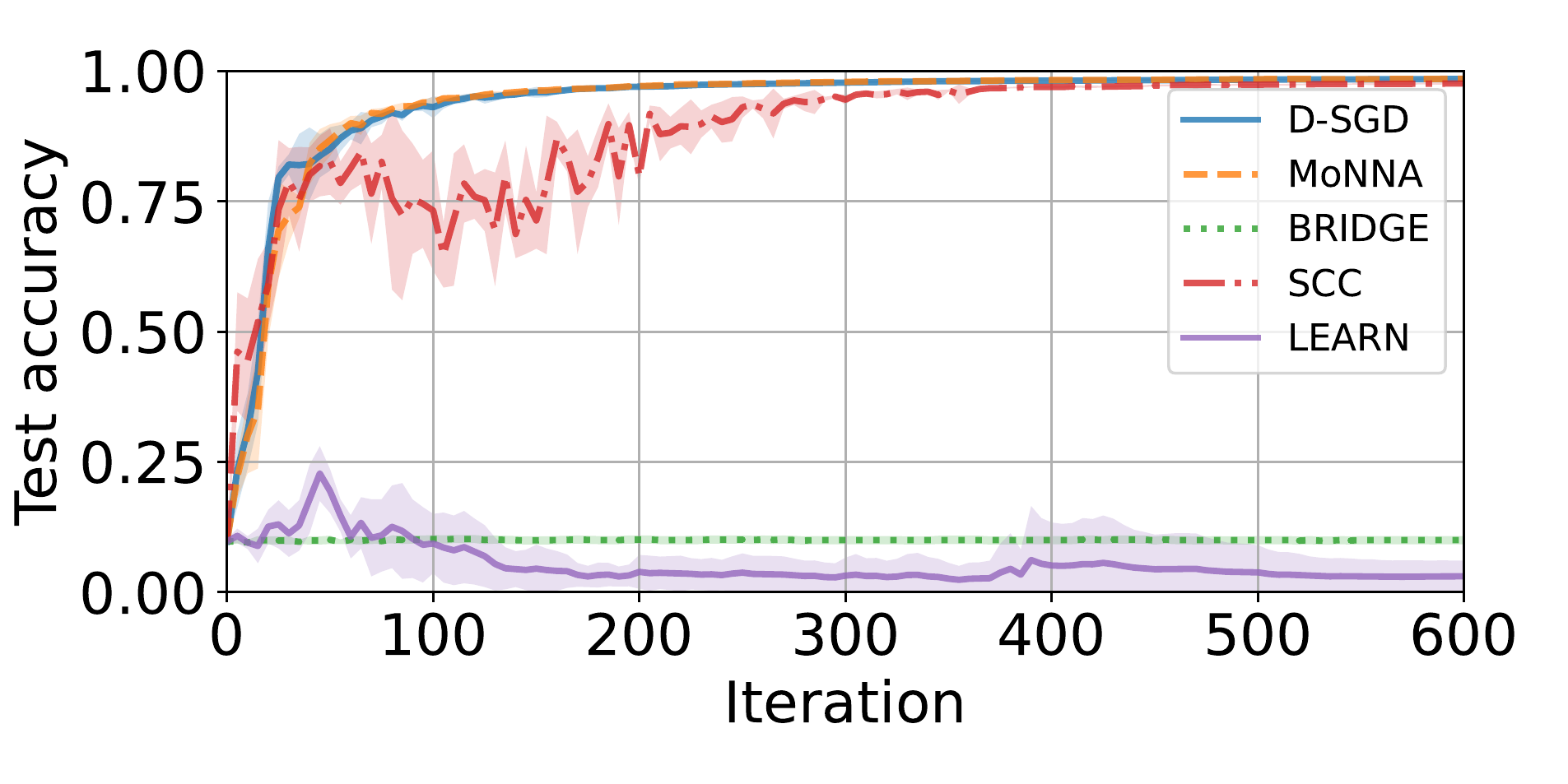}%
    \includegraphics[width=71mm]{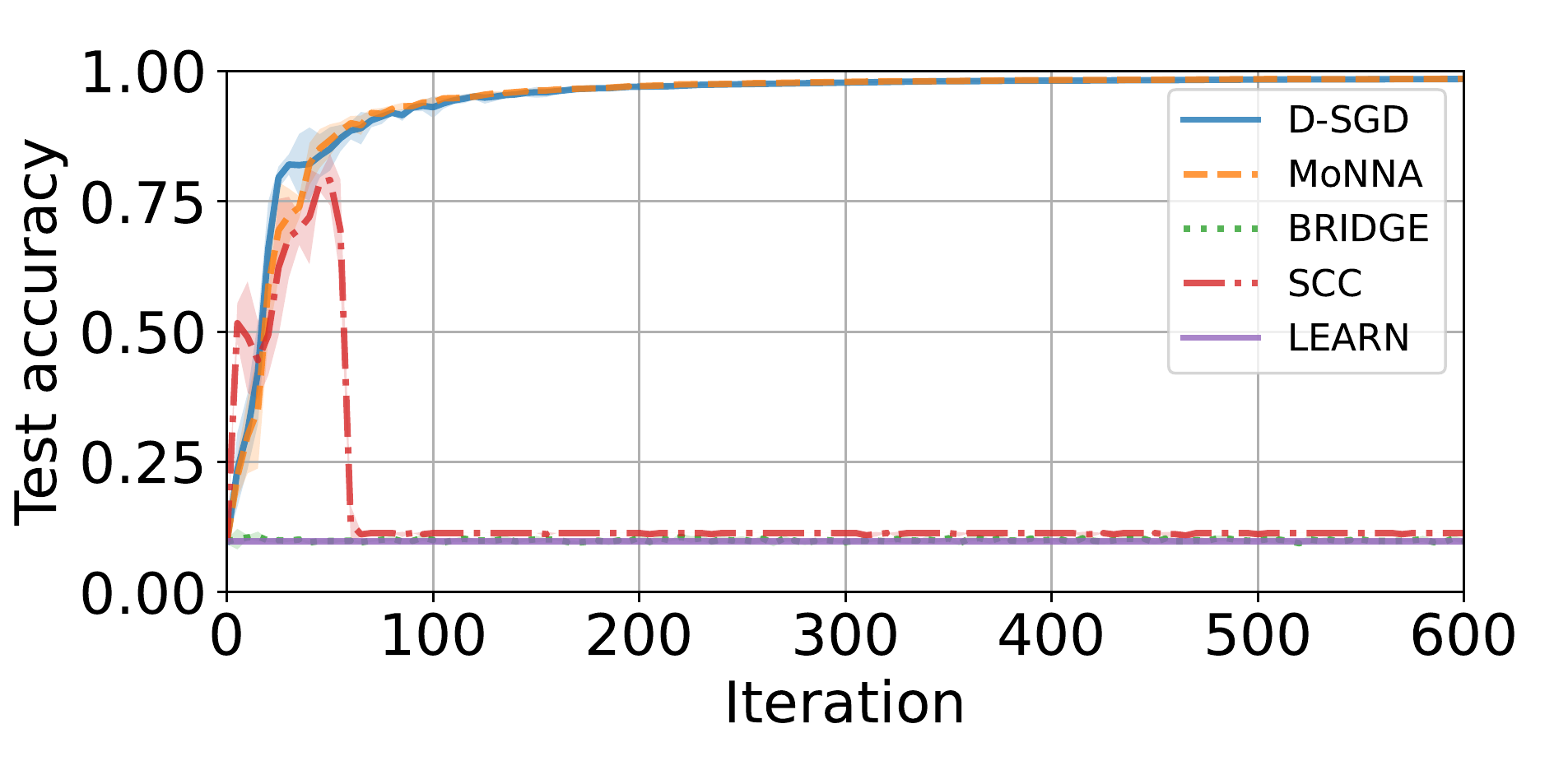}
    \caption{Learning accuracies achieved on MNIST with $\alpha = 1$ by D-SGD, \newalgorithm{}, BRIDGE, SCC, and LEARN. There are $n = 26$ nodes among which $f = 5$ are \byzantine{}. The \byzantine{} nodes execute the \textit{FOE} (row 1, left), \textit{ALIE} (row 1, right), \textit{LF} (row 2, left), and \textit{SF} (row 2, right) attacks. All algorithms except LEARN compute 15,000 gradients, while LEARN computes 180,300 gradients.}
    \label{fig:experiments-2}
\end{figure*}

\begin{figure*}[ht]
    \centering
    \includegraphics[width=71mm]{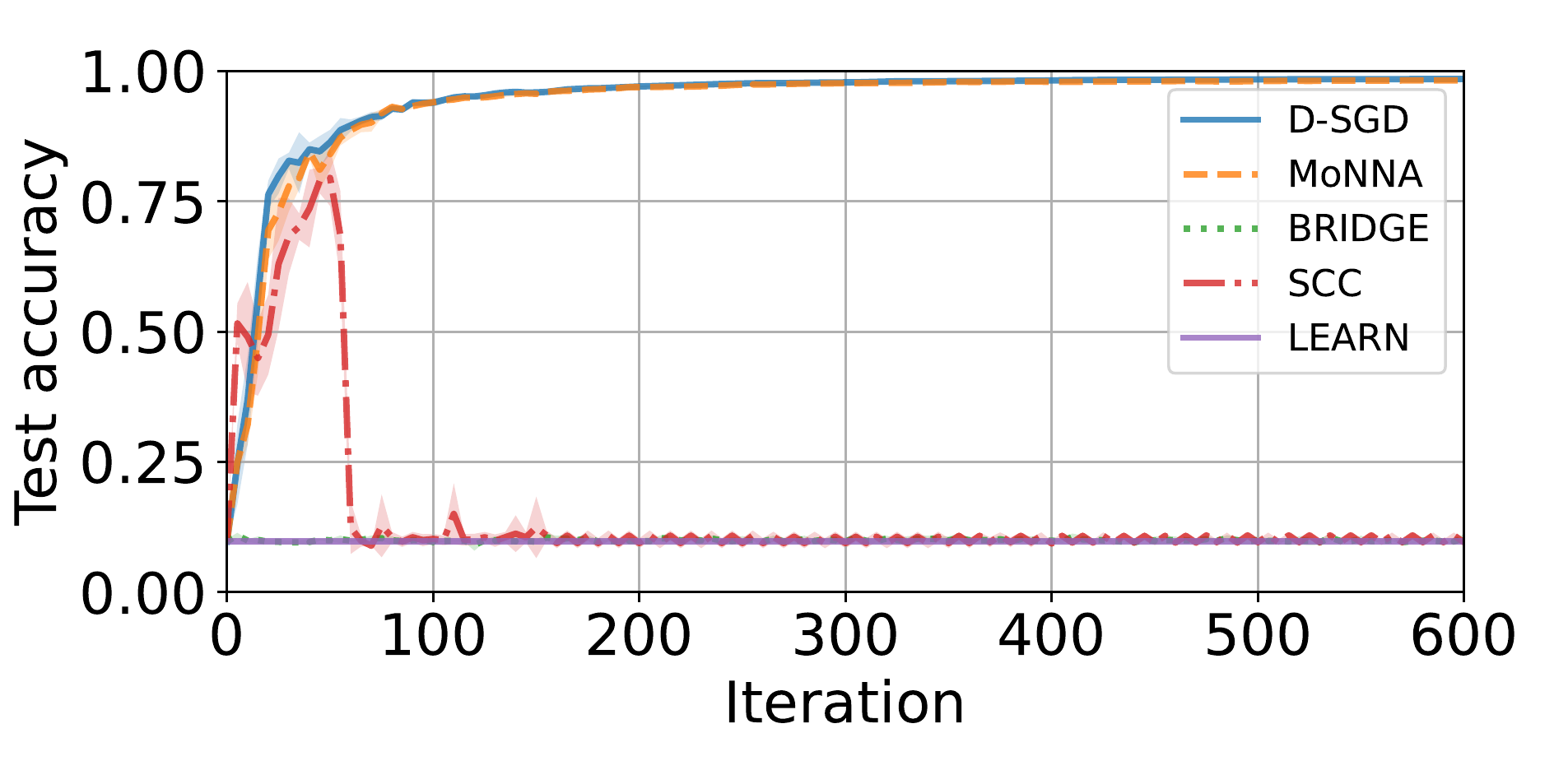}%
    \includegraphics[width=71mm]{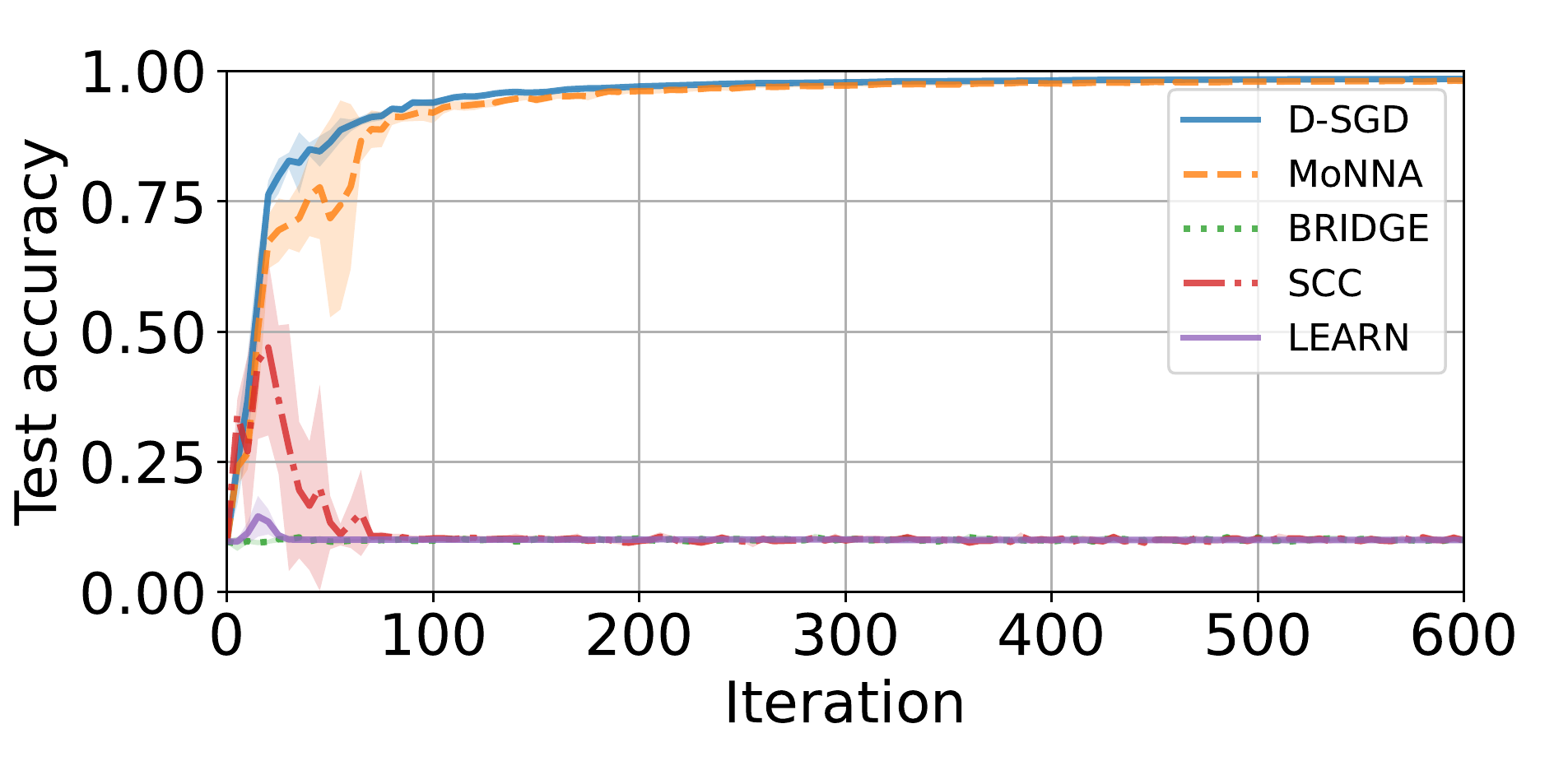}\\
    \includegraphics[width=71mm]{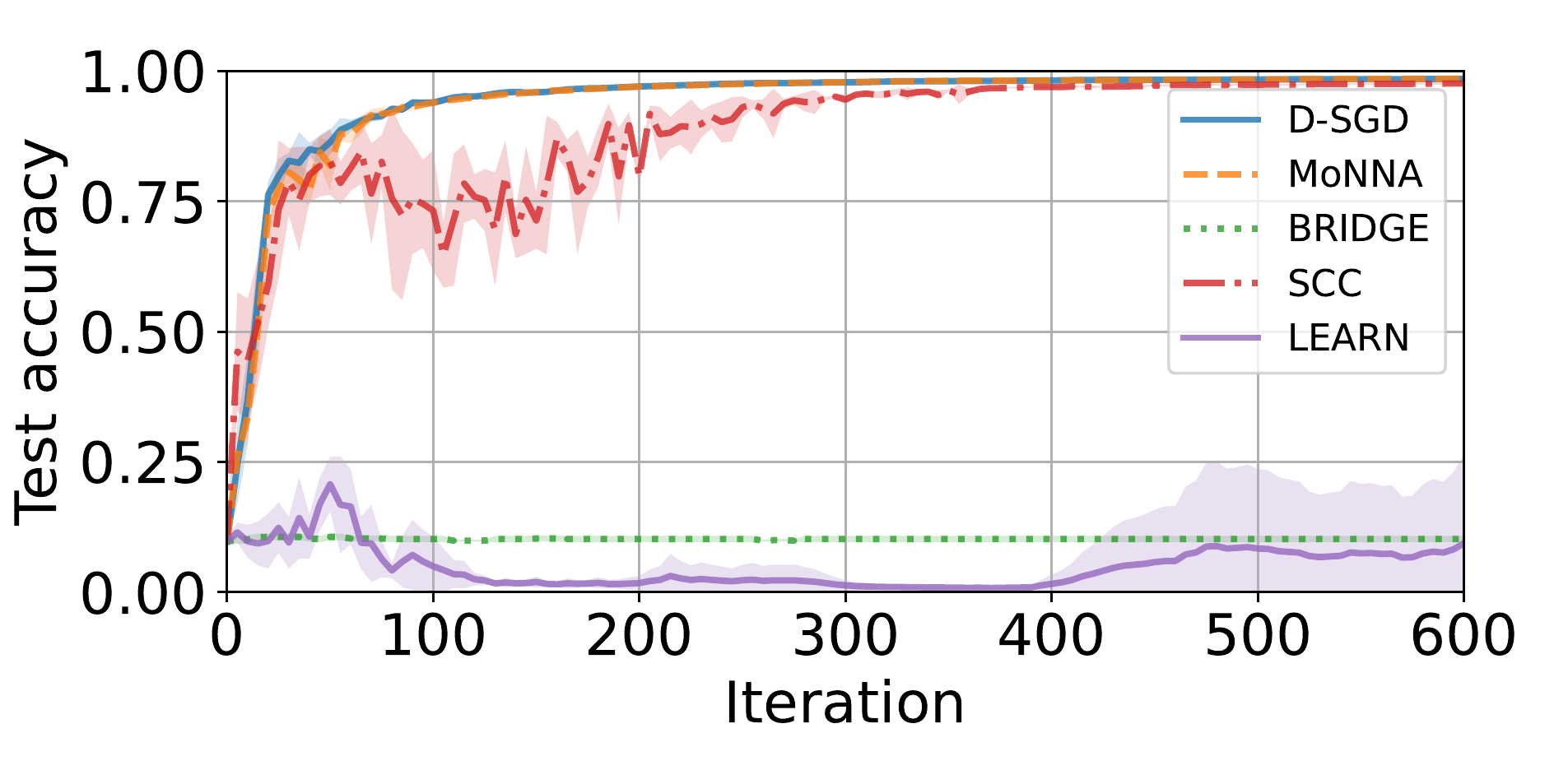}%
    \includegraphics[width=71mm]{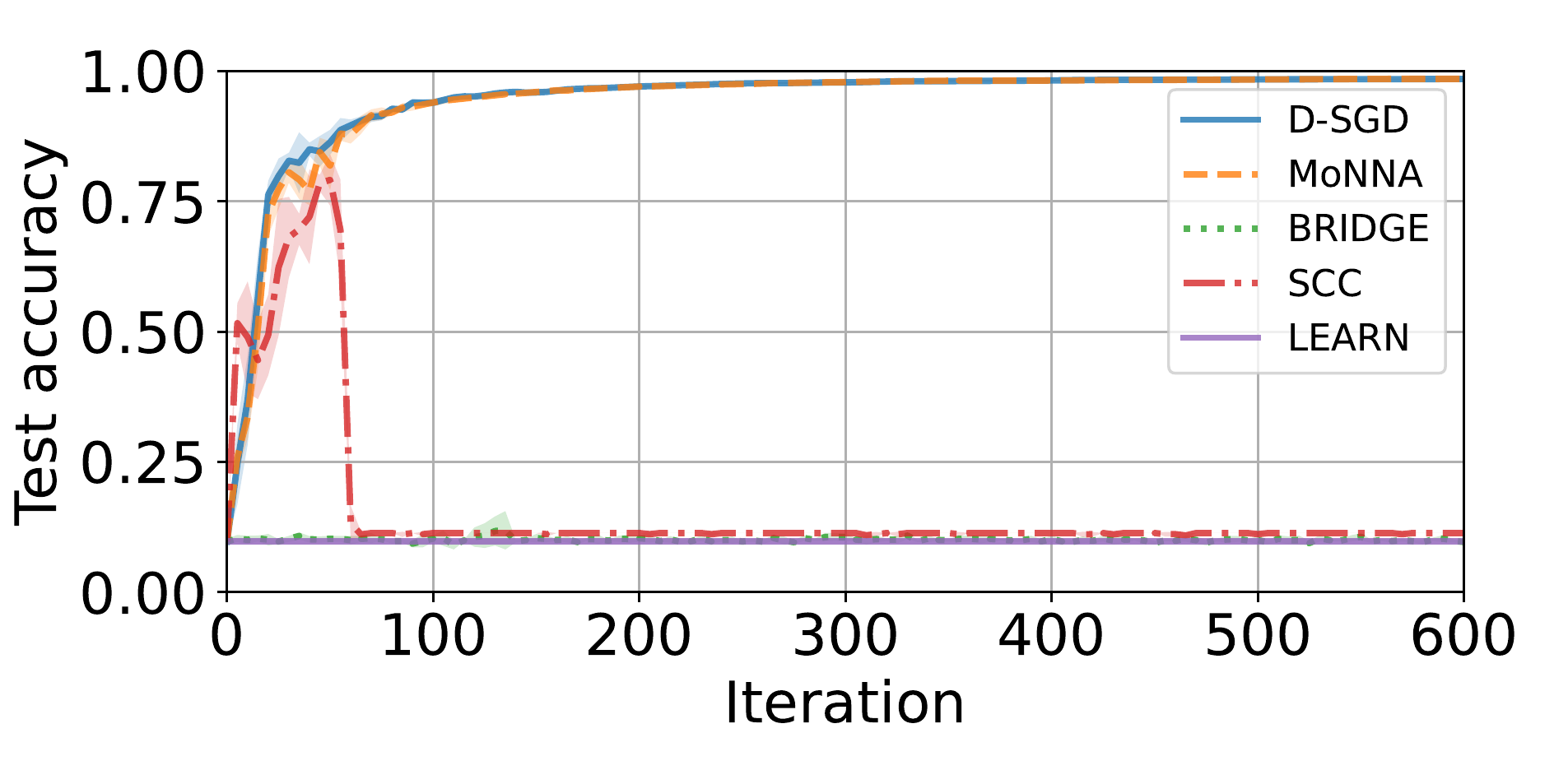}
    \caption{Learning accuracies achieved on MNIST with $\alpha = 5$ by D-SGD, \newalgorithm{}, BRIDGE, SCC, and LEARN. There are $n = 26$ nodes among which $f = 5$ are \byzantine{}. The \byzantine{} nodes execute the \textit{FOE} (row 1, left), \textit{ALIE} (row 1, right), \textit{LF} (row 2, left), and \textit{SF} (row 2, right) attacks. All algorithms except LEARN compute 15,000 gradients, while LEARN computes 180,300 gradients.}
    \label{fig:experiments-3}
\end{figure*}

We complete the missing results from Figure~\ref{fig:experiments-1} in the main paper by presenting in Figures~\ref{fig:experiments-2} and~\ref{fig:experiments-3} the performance of \newalgorithm{} on MNIST with $\alpha = 1$ and 5, respectively.

As observed in Section~\ref{sec:experiment-results}, \newalgorithm{} is the only considered algorithm that provides consistently good performances when tested on MNIST in two heterogeneity regimes and in the presence of \byzantine{} nodes. Indeed, under all attacks, \newalgorithm{} almost matches the performance of D-SGD in terms of learning accuracy (as well as computational workload per node). While SCC showcases satisfactory results under LF, the FOE, ALIE, and SF attacks prevent the model from learning. Similar observations hold for BRIDGE and LEARN which are completely unable to learn, with their final accuracies stagnating at around 10\%.

\subsection{Necessity of Momentum and NNA}\label{app_results_ablation_study}
Our theoretical results indicate that the two key ingredients of \newalgorithm{}, namely Polyak's momentum and NNA, are sufficient to guarantee convergence in adversarial settings. In this section, we empirically measure their necessity by comparing our algorithm to momentum-less solutions as well as other algorithms that do not use \cva{}. In particular, in addition to D-SGD, \newalgorithm{}, and SCC, we run our experiments on two prominent aggregation algorithms from the literature, 
namely geometric median (GM)~\cite{chen2017distributed} and coordinate-wise trimmed mean (CWTM)~\cite{yin2018byzantine}. We also execute both GM and CWTM with momentum $\beta = 0.99$ (referred to as MoGM and MoCWTM, respectively). Additionally, we also compare \newalgorithm{} to its momentum-less variant, namely NNA. We report on these results on the CIFAR-10 and MNIST datasets in Figures~\ref{fig:experiments-4} and~\ref{fig:experiments-5}, respectively.

\begin{figure*}[ht]
    \centering
    \includegraphics[width=71mm]{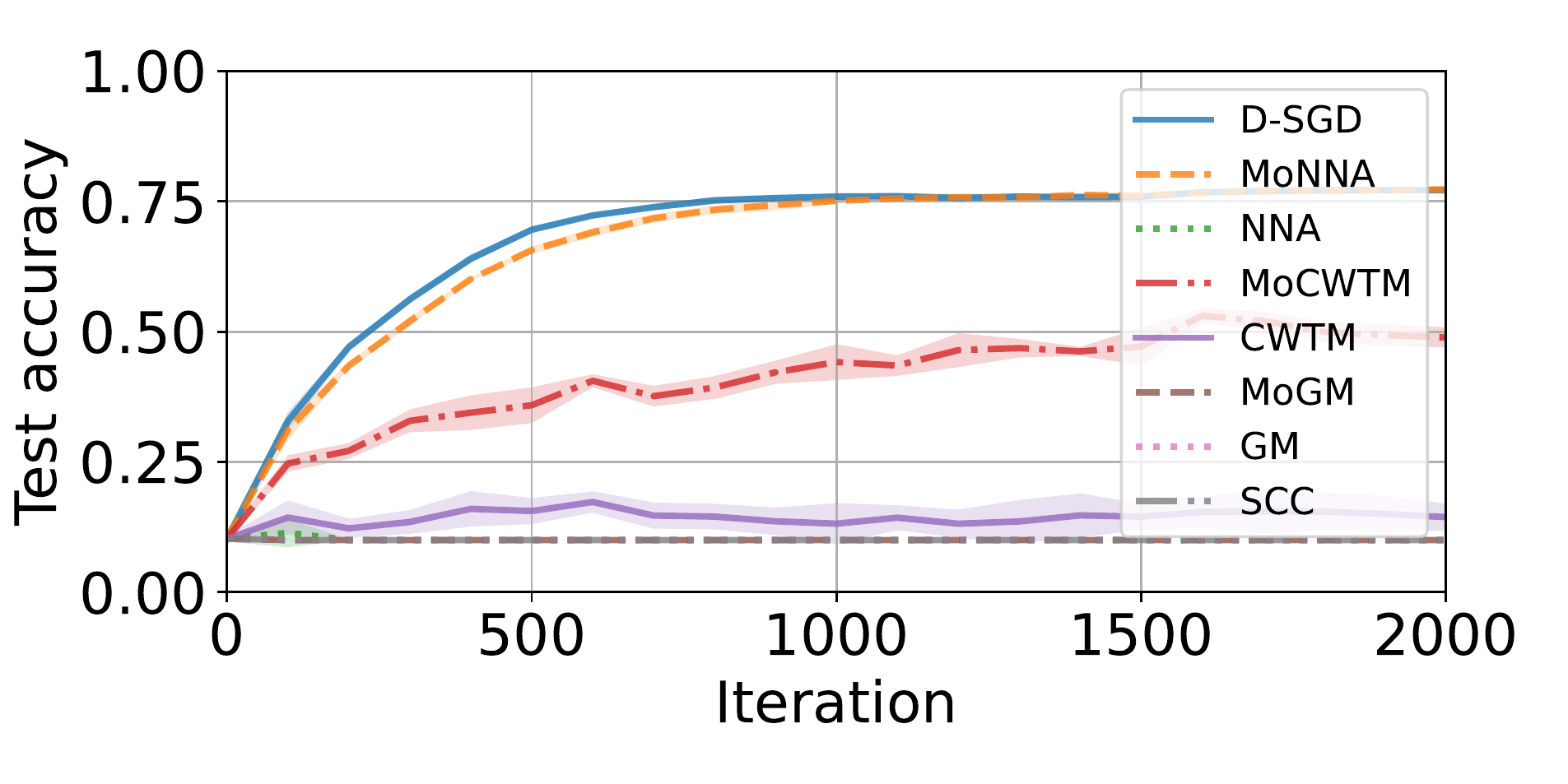}%
    \includegraphics[width=71mm]{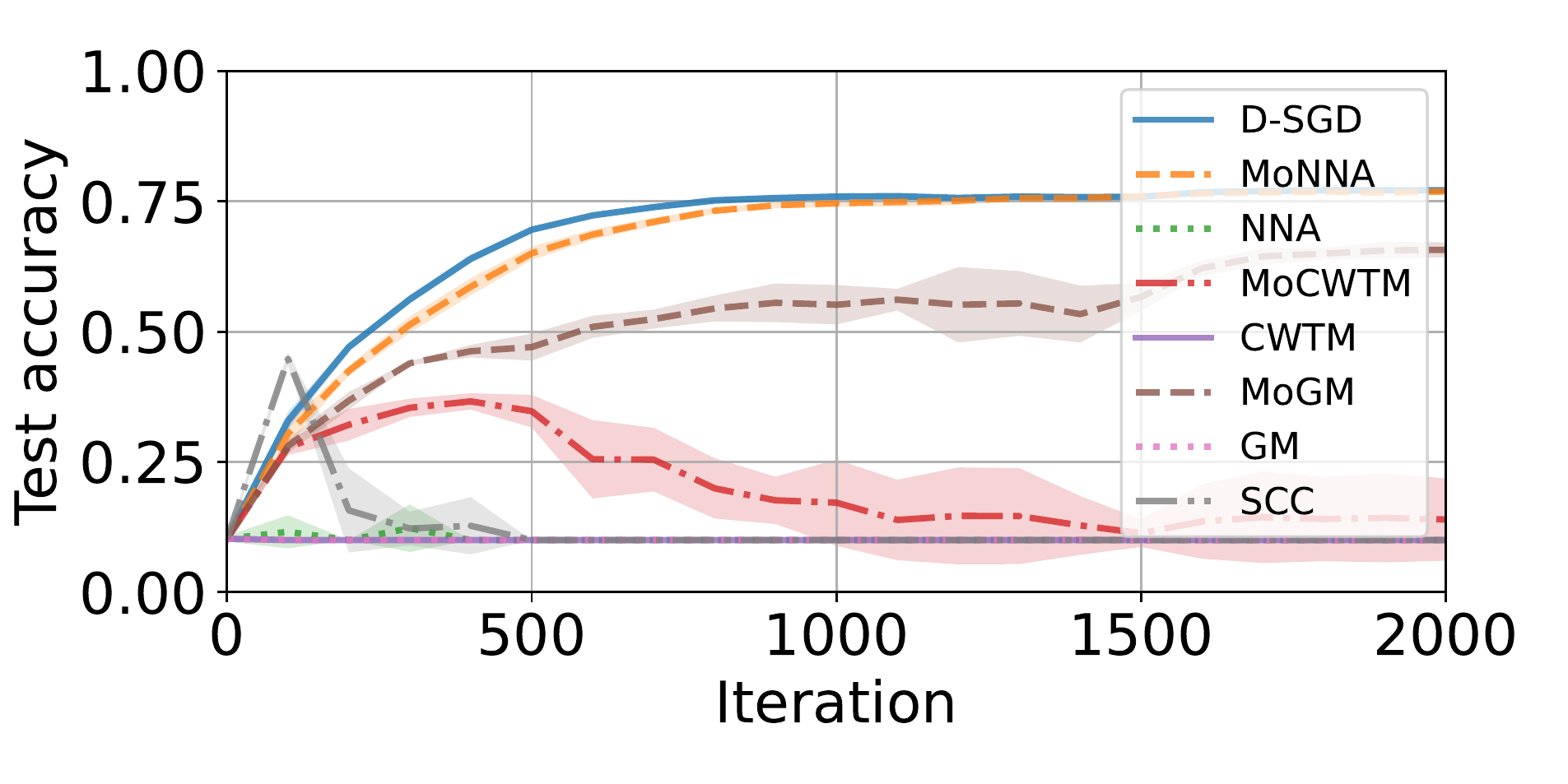}\\
    \includegraphics[width=71mm]{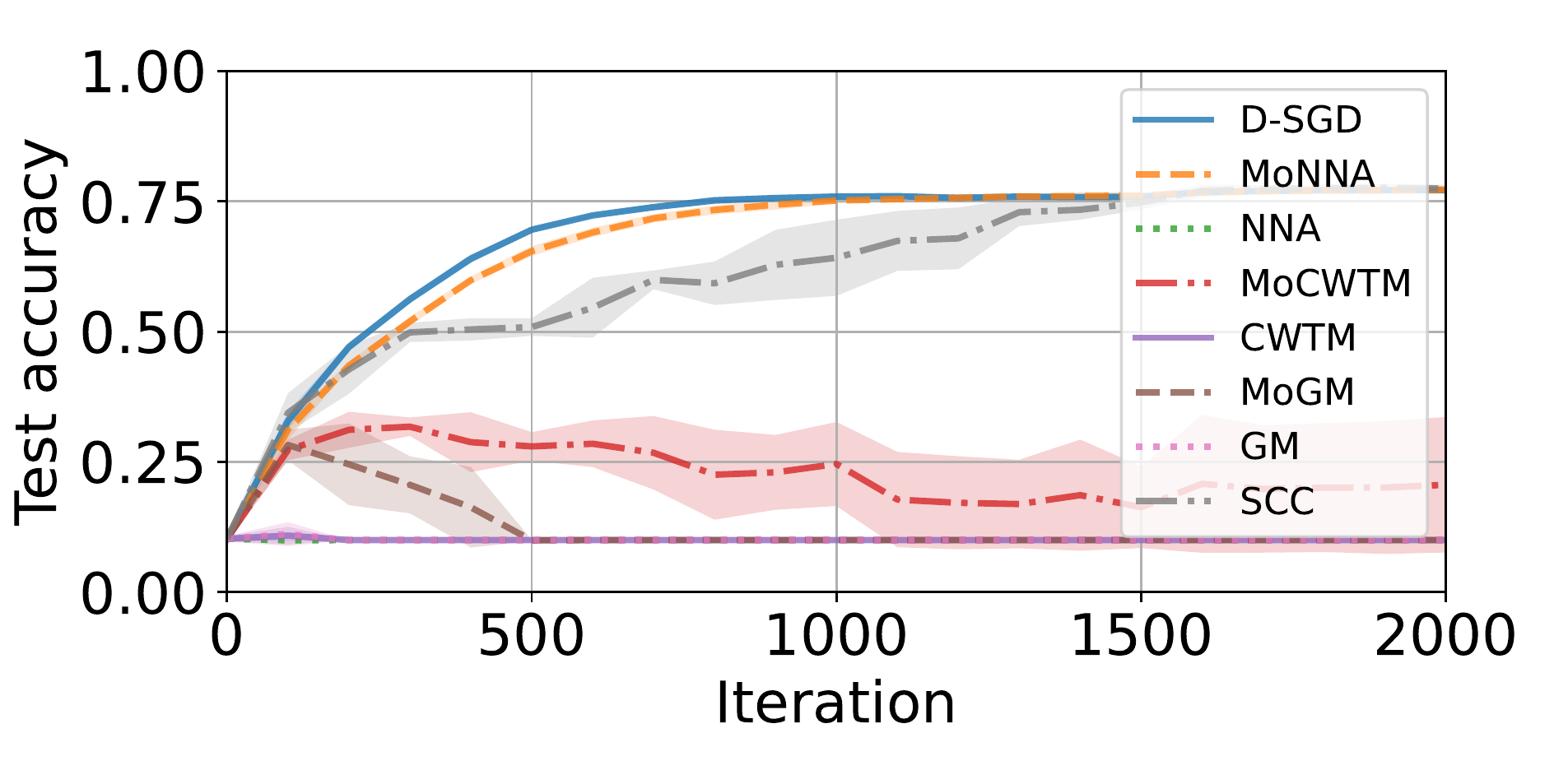}%
    \includegraphics[width=71mm]{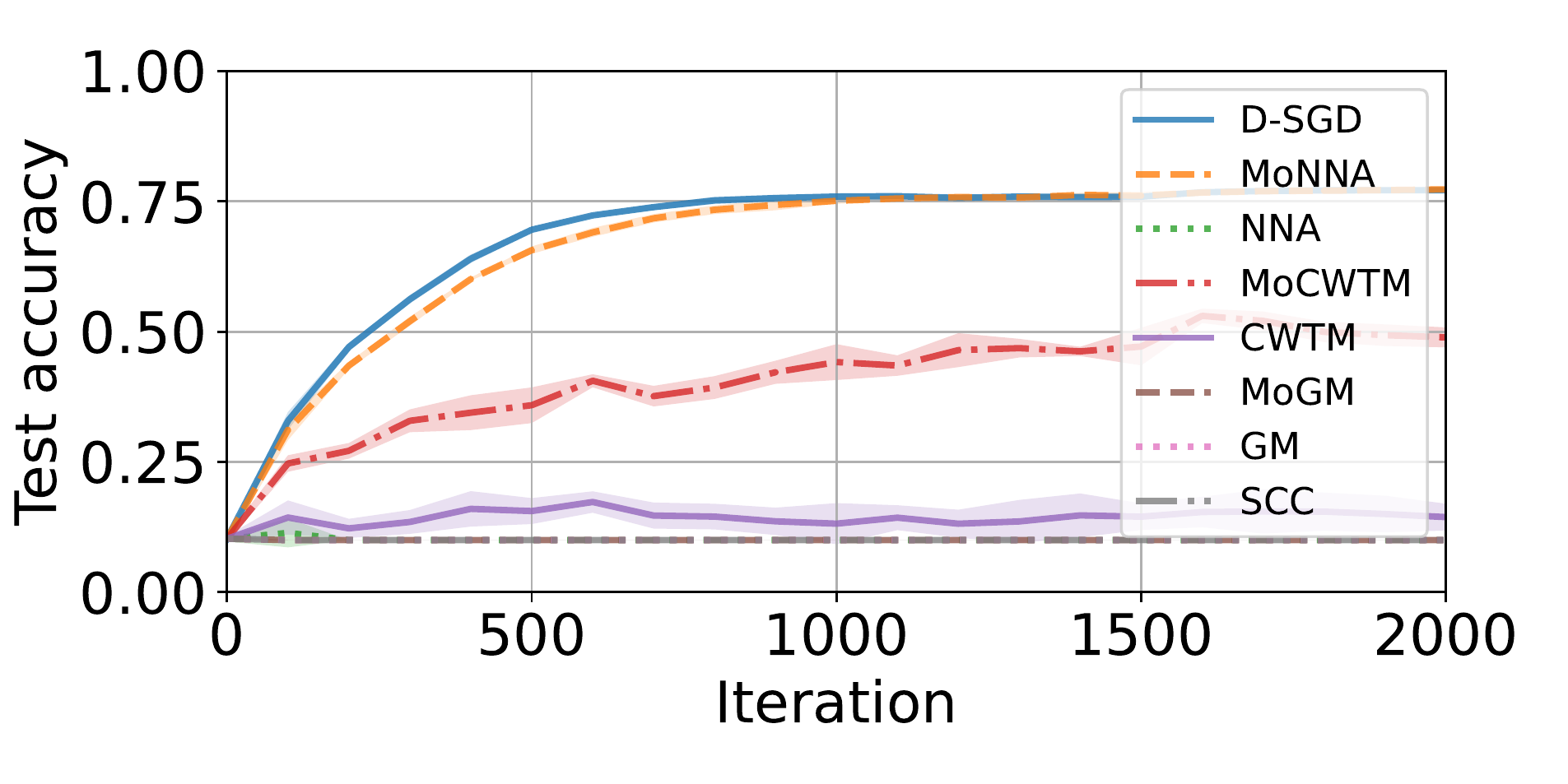}
    \caption{Comparison of the learning accuracies achieved by various aggregation algorithms on CIFAR-10 with $\alpha = 5$, notably including D-SGD, \newalgorithm{}, GM, MoGM (i.e., GM with $\beta = 0.99$), CWTM, MoCWTM (i.e., CWTM with $\beta = 0.99$), and SCC. There are $n = 16$ nodes among which $f = 3$ are \byzantine{}. The \byzantine{} nodes execute the \textit{FOE} (row 1, left), \textit{ALIE} (row 1, right), \textit{LF} (row 2, left), and \textit{SF} (row 2, right) attacks.}
    \label{fig:experiments-4}
\end{figure*}

\begin{figure*}[ht]
    \centering
    \includegraphics[width=71mm]{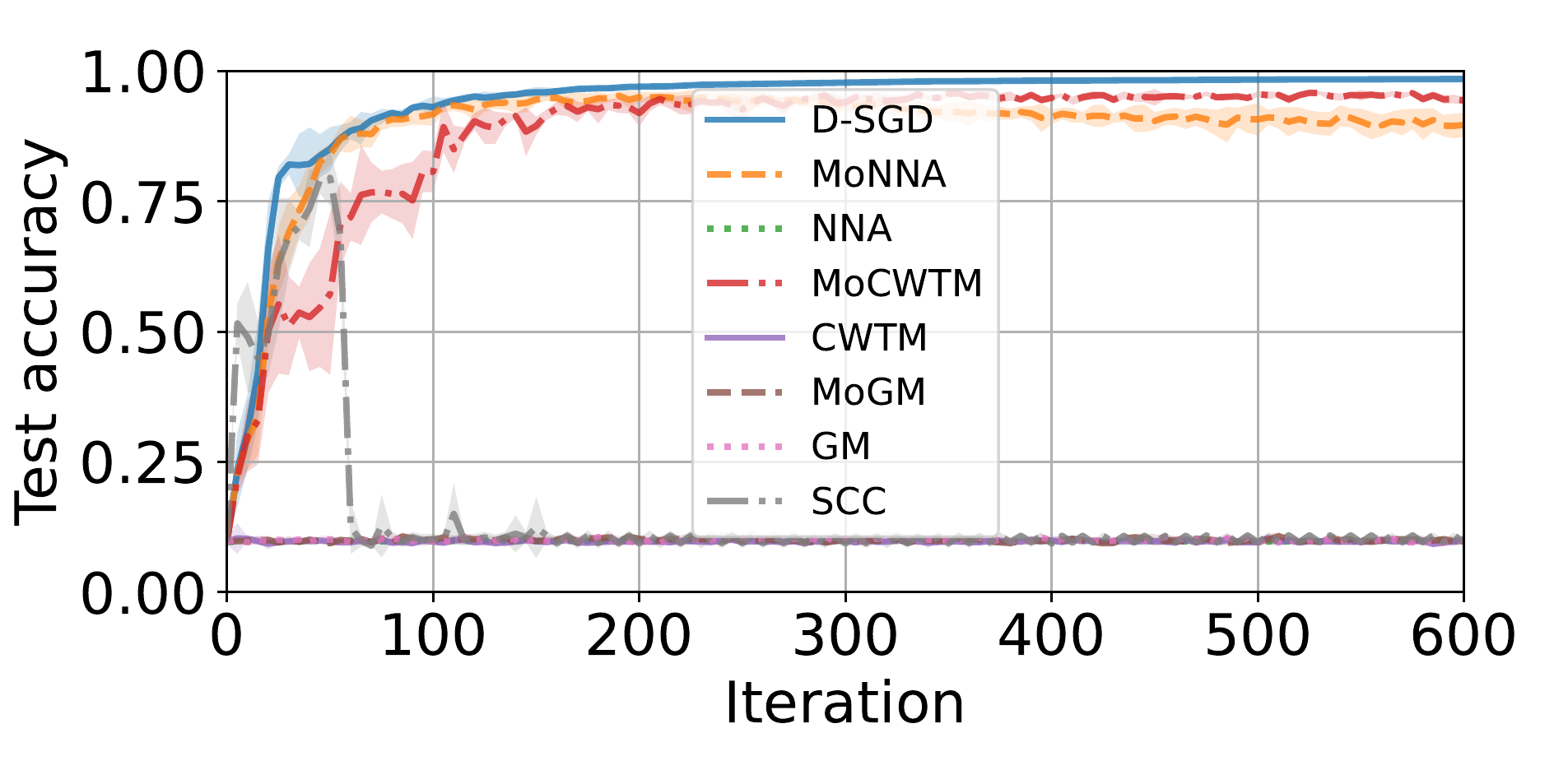}%
    \includegraphics[width=71mm]{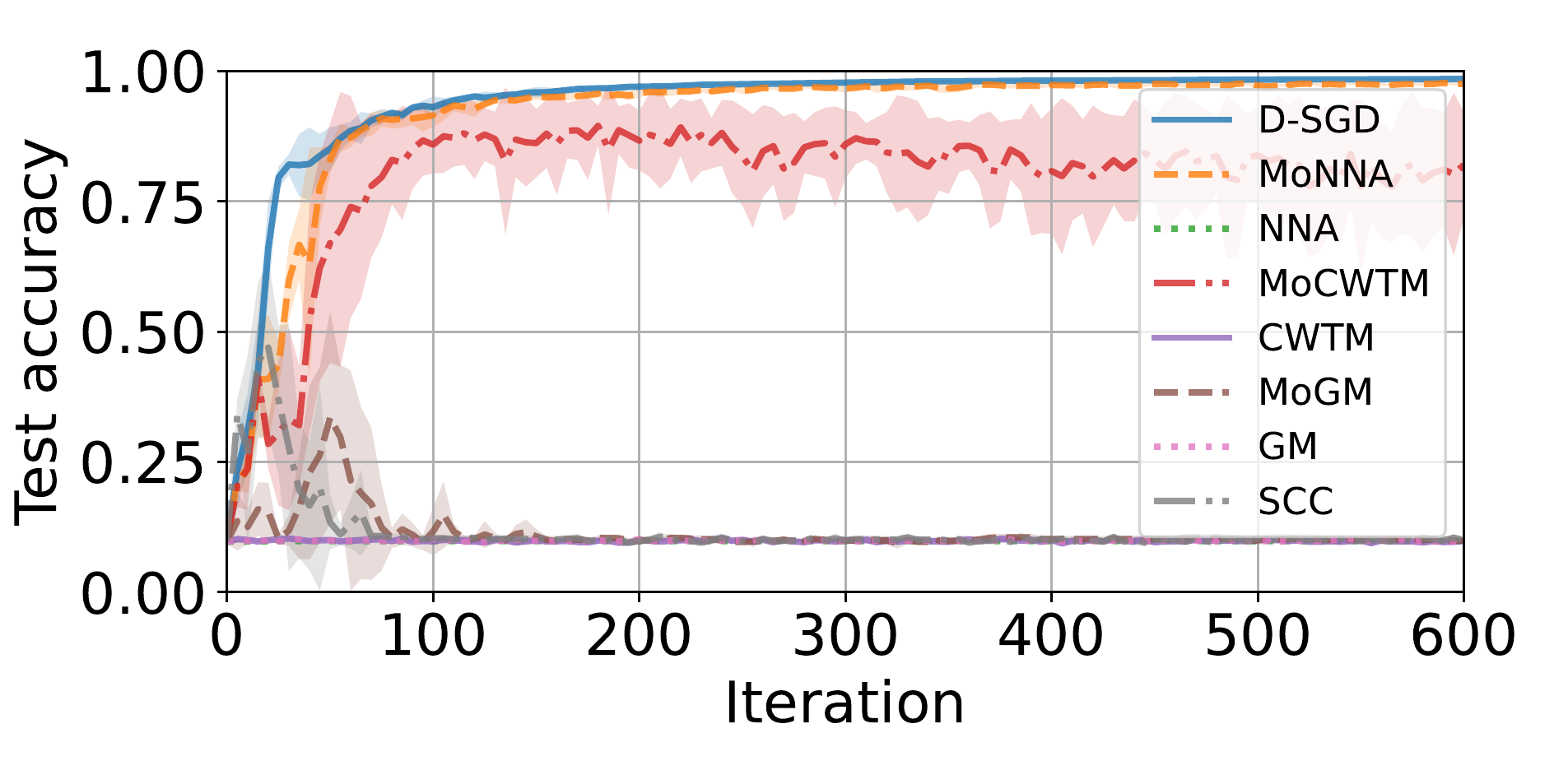}\\
    \includegraphics[width=71mm]{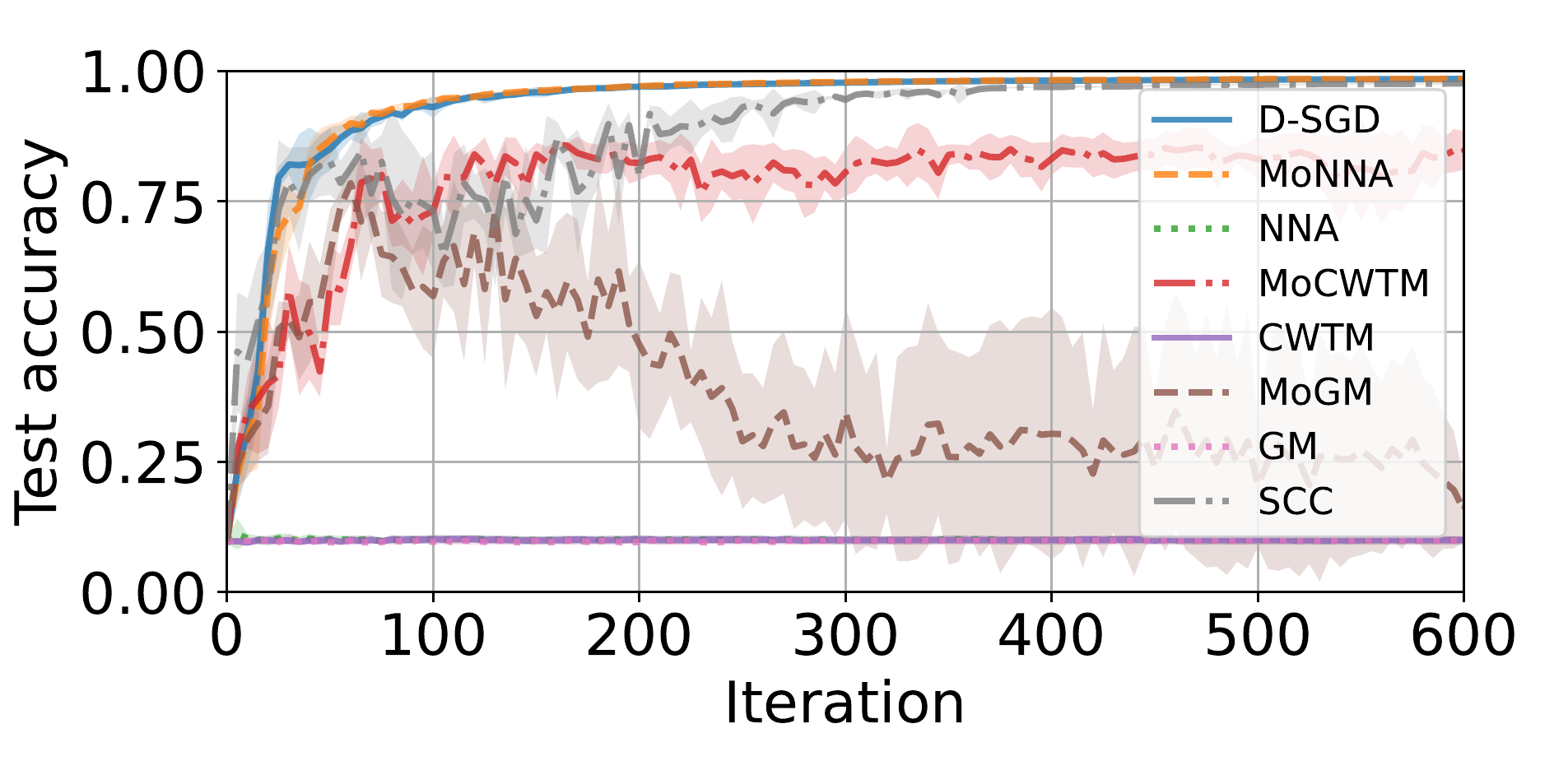}%
    \includegraphics[width=71mm]{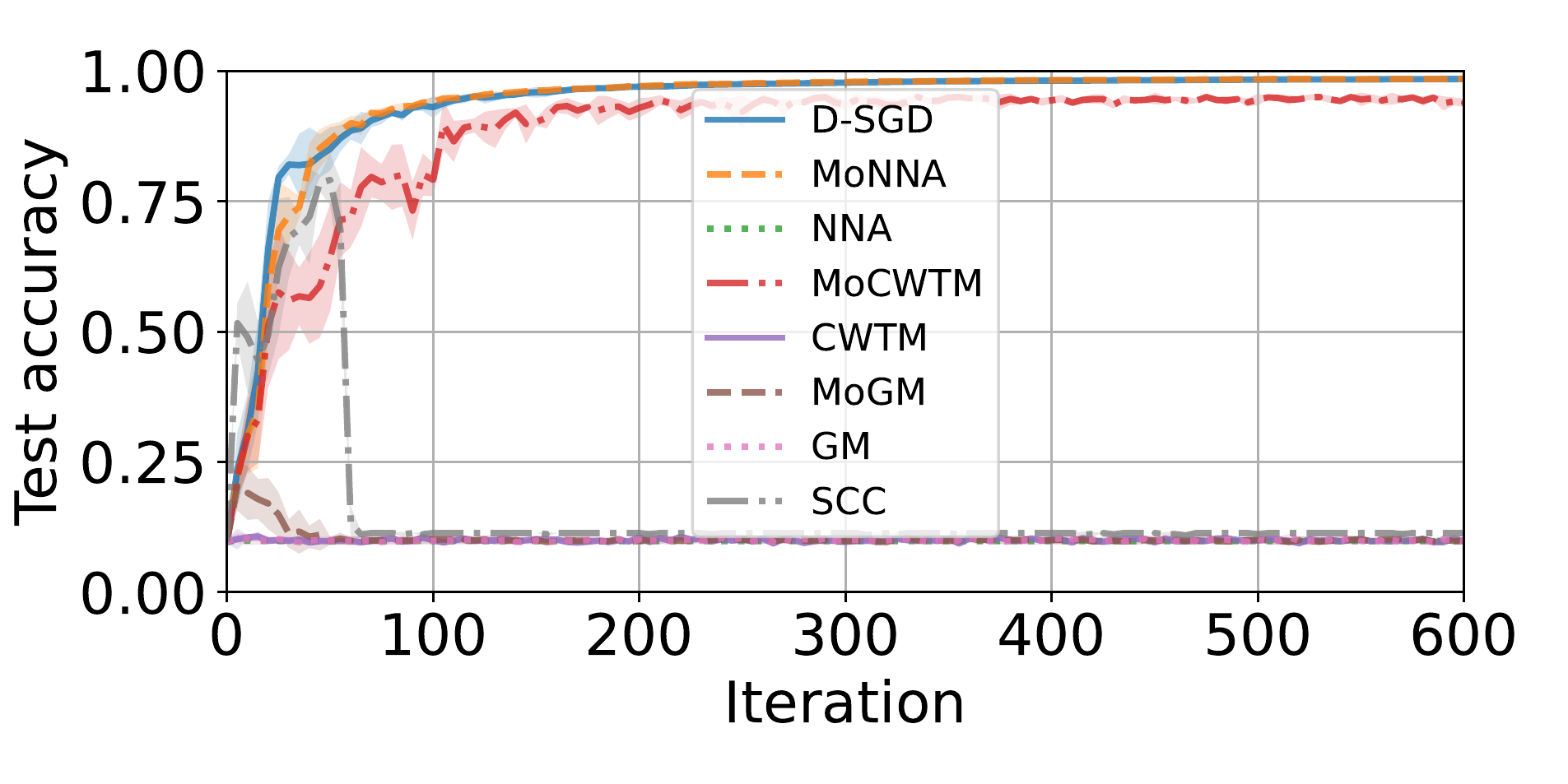}
    \caption{Comparison of the learning accuracies achieved by various algorithms on MNIST with $\alpha = 1$, notably including D-SGD, \newalgorithm{}, GM, MoGM (i.e., GM with $\beta = 0.99$), CWTM, MoCWTM (i.e., CWTM with $\beta = 0.99$), and SCC. There are $n = 26$ nodes, among which $f = 5$ are \byzantine{}. The \byzantine{} nodes execute \textit{FOE} (row 1, left), \textit{ALIE} (row 1, right), \textit{LF} (row 2, left), and \textit{SF} (row 2, right).}
    \label{fig:experiments-5}
\end{figure*}

Our observations are twofold. First, it is clear from Figures~\ref{fig:experiments-4} and~\ref{fig:experiments-5} that momentum plays a crucial role in ensuring the robustness of \newalgorithm{}. Indeed, momentum-less \newalgorithm{} (i.e., simply \cva{}) is completely unable to learn under all attacks, showcasing a very low accuracy constant at 10\% throughout the entire learning. However, as previously mentioned, \newalgorithm{} drastically mitigates these attacks. Indeed, the model steadily increases in accuracy to finally reach 95\% on MNIST and 75\% on CIFAR-10 under all four attacks. Moreover, the importance of momentum is also further corroborated by the equally poor performances of CWTM and GM.

Second, we show the critical importance of the NNA scheme when defending against \byzantine{} nodes. Although much more resilient than its momentum-less counterpart (especially on MNIST), MoCWTM remains largely vulnerable to attacks which are able to completely hinder its learning on CIFAR-10. Indeed, even though the accuracy increases under FOE and SF, it plateaus at 50\%, which is 25\% less than the accuracy obtained with \newalgorithm{} on CIFAR-10. Additionally, ALIE completely annihilates the performance of MoCWTM, with a final accuracy close to 10\%. The same observation holds for LF. Furthermore, while SCC and MoGM showcase good results under LF and ALIE respectively, the other attacks completely degrade their performances on both CIFAR-10 and MNIST. The \textbf{worst case performances} of \newalgorithm{}'s rivals are thus very poor (unlike \newalgorithm{} which performs well in all cases). We argue that one should carefully examine this fundamental metric when evaluating the robustness of aggregation techniques, as the same algorithm can simultaneously greatly defend against some attacks but perform poorly against others.

This entire analysis demonstrates the superiority of our solution and suggests that momentum and \cva{} might be two necessary components in practice to ensure the robustness of distributed asynchronous systems.

\end{document}